\documentclass{article}

\usepackage{times}

\PassOptionsToPackage{sort}{natbib}

\usepackage{iclr2023_conference}
\iclrfinaltrue

\usepackage[hidelinks]{hyperref}

\usepackage{amsmath}
\usepackage{amssymb}
\usepackage{amsthm}
\usepackage{comment}
\usepackage{dsfont}
\usepackage{graphicx}
\usepackage{mathtools}
\usepackage{mleftright}
\usepackage{nicefrac}
\usepackage{paralist}
\usepackage{xspace}
\usepackage{xcolor}
\usepackage{enumitem}
\usepackage{booktabs}
\usepackage{longtable}
\usepackage[capitalize]{cleveref}
\usepackage{thmtools}
\usepackage{thm-restate}

\usepackage[latin1]{inputenc}
\usepackage[T1]{fontenc}

\usepackage{url}
\usepackage{subcaption}

\graphicspath{{figures/}}

\newcommand{\hypergraph}{H}
\newcommand{\graph}{G}
\newcommand{\edges}     {E}
\newcommand{\reals}     {\mathds{R}}
\newcommand{\naturals}     {\mathds{N}}
\newcommand{\vertices}  {V}
\newcommand{\measure}  {\mu}
\newcommand{\curvature}{\kappa}
\DeclareMathOperator{\degree}{deg}
\DeclareMathOperator{\rbf}{RBF}
\DeclareMathOperator{\neighborhood}{\mathcal{N}}
\newcommand{\uniformity}{r}
\newcommand{\regularity}{k}
\newcommand{\cooccurrence}{c}
\newcommand{\intersecting}{s}
\newcommand{\smoothing}{\alpha}
\newcommand{\permutation}{\sigma}
\newcommand{\adjacent}{\sim}

\newcommand{\aggregation}{\texttt{\textsc{Agg}}\xspace}
\newcommand{\Wlog}{w.l.o.g.\xspace}
\newcommand{\suchthat}{s.t.\xspace}
\newcommand{\defeq}{\coloneqq} 

\DeclareMathOperator{\argmax}     {argmax}
\DeclareMathOperator{\diam}       {diam}
\DeclareMathOperator{\dist}       {d}
\DeclareMathOperator{\jump}       {J}
\DeclareMathOperator{\powerset}   {\mathcal{P}}
\DeclareMathOperator{\wasserstein}{W\!}

\newcommand{\mean}{\texttt{\textsc{A}}}
\newcommand{\bary}{\texttt{\textsc{B}}}
\newcommand{\maxi}{\texttt{\textsc{M}}}
\newcommand{\enrw}{\texttt{\textsc{EN}}}
\newcommand{\eerw}{\texttt{\textsc{EE}}}
\newcommand{\werw}{\texttt{\textsc{WE}}}

\newtheorem{lemma}{Lemma}

\newcommand{\ourmethod}{\textsc{Orchid}\xspace}
\newcommand{\orc}{ORC\xspace}
\newcommand{\frc}{FRC\xspace}
\newcommand{\ollivier}{Ollivier-Ricci curvature\xspace}

\title{Ollivier-Ricci Curvature for Hypergraphs: A~Unified Framework}

\newcommand{\oururl}{\hyperref{https://doi.org/10.5281/zenodo.7624573}{}{}{https://doi.org/10.5281/zenodo.7624573}}

\author{Corinna Coupette$^1$, Sebastian Dalleiger$^{1,2}$, Bastian Rieck$^{3,4}$\\[+0.5cm]
  $^1$Max Planck Institute for Informatics\\
  $^2$CISPA Helmholtz Center for Information Security\\
  $^3$AIDOS Lab, Institute of AI for Health, Helmholtz Munich\\
  $^4$Technical University of Munich~(TUM)
}

\begin{document}
	\maketitle

\begin{abstract}
  \noindent Bridging geometry and topology, 
curvature is a powerful and expressive invariant. 
While the utility of curvature has been theoretically and empirically confirmed in the context of manifolds and graphs, 
its generalization to the emerging domain of \emph{hypergraphs} has remained largely unexplored.
On graphs, 
the \emph{\ollivier} measures differences between random walks via Wasserstein distances, 
thus grounding a geometric concept in ideas from probability theory and optimal transport.
We develop \ourmethod, a flexible framework generalizing \ollivier to hypergraphs, and prove that the resulting curvatures have favorable theoretical properties.
Through extensive experiments on synthetic and real-world hypergraphs from different domains, 
we demonstrate that \ourmethod curvatures are both scalable and useful to perform a variety of hypergraph tasks in practice.
\end{abstract}

\section{Introduction}

Hypergraphs generalize graphs by allowing any number of nodes to participate in an edge. 
They enable us to faithfully represent complex
relations, 
such as co-authorship of scientific papers, 
multilateral interactions between chemicals, or group conversations,
which cannot be adequately captured by graphs. 
While hypergraphs are more expressive than graphs and other relational objects like simplicial complexes,
they are harder to analyze both theoretically and empirically, 
and many concepts that have proven useful for understanding graphs have yet to be transferred to the hypergraph setting.

\emph{Curvature} has established itself as a powerful
characteristic of Riemannian manifolds, 
as it permits the
description of \emph{global properties} through \emph{local measurements} by
harmonizing ideas from geometry and topology.
For graphs, \emph{graph curvature} measures to what extent the
neighborhood of an edge deviates from certain idealized model spaces,
such as cliques, grids, or trees.
It has proven helpful, for example, 
in assessing differences between real-world networks~\citep{samal2018comparative},
identifying bottlenecks in real-world networks~\citep{gosztolai2021unfolding}, 
and alleviating oversquashing in graph neural networks~\citep{topping2022oversquashing}.
One prominent notion of graph curvature is
\emph{\ollivier} (\orc). 
\orc compares random walks based at specific nodes, 
revealing differences in the information diffusion behavior in the graph. 
As the sizes of edges and edge intersections can vary in hypergraphs, 
there are many ways to generalize \orc to hypergraphs. 
While some notions of hypergraph \orc have been previously studied in isolation \citep[e.g.,][]{asoodeh2018curvature,leal2020ricci,eidi2020ollivier}, 
a unified framework for their definition and computation is still lacking.

\paragraph{Contributions.}
We introduce \ourmethod, 
a unified framework for \ollivier on hypergraphs. 
\ourmethod integrates and generalizes existing approaches to hypergraph \orc.
Our work is the first to identify the individual building blocks shared by all notions of hypergraph \orc, 
and to perform a rigorous theoretical and empirical analysis of
the resulting curvature formulations. 
We develop hypergraph \orc notions
that are aligned with our geometric intuition
while still efficient to compute,
and we demonstrate the utility of these notions in practice through extensive experiments.

\paragraph{Structure.} 
After providing the necessary background on graphs and hypergraphs
and recalling the definition of \ollivier for graphs 
in \cref{prelim}, 
we introduce \ourmethod, our framework for hypergraph \orc,
and analyze the theoretical properties of \ourmethod curvatures in \cref{theory}.
We assess the empirical properties and practical utility of \ourmethod curvatures through extensive experiments in \cref{experiments},
and discuss limitations and potential extensions of \ourmethod 
as well as directions for future work in \cref{sec:Discussion and Conclusion}. 
Further materials are provided in \cref{apx-proofs,apx-implementation,apx-datasets,apx-related,apx-results}.

  \section{Preliminaries}
\label{prelim}

\paragraph{Graphs and Hypergraphs}
\label{prelim:hg}

A \emph{simple graph} $\graph = (\vertices,
\edges)$ is a tuple containing $n$ nodes (vertices)~$\vertices = \{v_1,\dots,v_n\}$ and $m$ edges~$\edges = \{e_1,\dots,e_m\}$, 
with $e_i \in \binom{\vertices}{2}$ for all $i\in[m]$. 
Here, for a set $S$ and a positive integer $k\leq |S|$, 
$\binom{S}{k}$ denotes the set of all $k$-element subsets of $S$, 
and for $x\in\naturals$ with $0\notin \naturals$, $[x] = \{i\in\naturals\mid i\leq x\}$. 
In \emph{multi-graphs}, 
edges can occur multiple times, 
and hence, $\edges = (e_1,\dots,e_m)$ is an indexed family of sets, 
with $e_i \in \binom{\vertices}{2}$ for all $i\in[m]$. 
Generalizing simple graphs, a \emph{simple hypergraph} $\hypergraph = (\vertices,
\edges)$ is a tuple containing $n$ nodes $V$ and $m$ hyperedges $E\subseteq \powerset(V)\setminus\emptyset$, 
i.e., in contrast to edges, hyperedges can have any cardinality $\uniformity\in [n]$. 
In a \emph{multi-hypergraph}, $\edges = (e_1,\dots,e_m)$ is an indexed family of sets, 
with $e_i\subseteq V$ for all $i\in[m]$. 
We assume that all our hypergraphs are multi-hypergraphs, 
and we drop the prefix \emph{hyper} from \emph{hypergraph} and \emph{hyperedge} where it is clear from context.

We denote the degree of node $i$, i.e., the number of edges containing $i$, by $\degree(i) = |\{e \in E\mid i \in e\}|$, 
write $i \adjacent j$ if $i$ is adjacent to $j$ (i.e., there exists $e \in E$ such that $\{i,j\}\subseteq e$), 
and use $\neighborhood(i)$ ($\neighborhood(e)$) for the neighborhood of $i$ ($e$), i.e., the set of nodes adjacent to $i$ (edges intersecting edge $e$). 
While $\degree(i) = |\neighborhood(i)|$ in simple graphs and $\degree(i) \geq |\neighborhood(i)|$ in multigraphs, 
these relations do not generally hold for hypergraphs.
Two nodes $i \neq j$ are \emph{connected} in~$\hypergraph$ if there is a sequence of nodes $i = v_1, v_2, \dots, v_{k - 1}, v_k = j$ such that $v_l \adjacent v_{l+1}$ for all $l \in [k]$.
Every such sequence is a \emph{path} in~$\hypergraph$, 
whose \emph{length} is the cardinality of the set of edges used in the adjacency relation.
We refer to the length of a shortest path connecting nodes~$i, j$ as the \emph{distance} between them, denoted as~$\dist(i, j)$.
We assume that all (hyper)graphs are \emph{connected}, i.e., there exists a path between all pairs of nodes.
This turns~$\hypergraph$ into a metric space~$(\hypergraph, \dist)$ with \emph{diameter}
$\diam(\hypergraph) \defeq \max\mleft\{ \dist\mleft(i, j\mright) \mid i, j \in \vertices \mright\}$.

(Hyper)graphs in which all nodes have the same degree~$\regularity$~($\degree(i) = \regularity$ for all $i \in V$) are called \emph{$k$-regular}.
Three properties of hypergraphs that distinguish them from graphs give rise to additional (ir)regularities. 
First, \emph{hyperedges} can vary in cardinality, 
and a hypergraph in which all hyperedges have the same cardinality $\uniformity$ ($|e| = \uniformity$ for all $e\in E$) is called \emph{$\uniformity$-uniform}.  
Second, \emph{hyperedge intersections} can have cardinality greater than $1$, 
and we call a hypergraph \emph{$\intersecting$-intersecting}
if all nonempty edge intersections have the same cardinality $\intersecting$
($e\cap f \neq \emptyset \Leftrightarrow |e \cap f| = \intersecting$ for all $e,f\in\edges$).
Third, nodes can \emph{cooccur in any number of hyperedges};
we call a hypergraph \emph{$\cooccurrence$-cooccurrent}
if each node cooccurs $\cooccurrence$ times with any of its neighbors
($i\adjacent j \Leftrightarrow |\{e\in\edges\mid \{i,j\}\subseteq e\}| = \cooccurrence$ for all  $i,j\in\vertices$).
Using this terminology, simple graphs are $2$-uniform, $1$-intersecting, $1$-cooccurrent hypergraphs.

Given a hypergraph $\hypergraph = (\vertices, \edges)$, 
the \emph{unweighted clique expansion} of $\hypergraph$ is $\graph^\circ = (\vertices, \edges^\circ)$ with $\edges^\circ = \{\{i,j\}\mid \{i,j\}\subseteq e~\text{for some}~e \in E\}$, 
where two nodes are adjacent in $\graph^\circ$ if and only if they are adjacent in $\hypergraph$.
The \emph{weighted clique expansion} of $\hypergraph$ is $\graph^\circ$ endowed with a weighting function $w\colon\edges^\circ\rightarrow\naturals$, 
where $w(e) = |\{e\in \edges\mid \{i,j\}\subseteq e\}|$ for each $e\in \edges^\circ$, 
i.e., an edge $\{i,j\}$ is weighted by how often $i$ and $j$ cooccur in edges from $\hypergraph$.
Both of these transformations are lossy, 
i.e., we cannot uniquely reconstruct $\hypergraph$ from $\graph^\circ$.
The \emph{unweighted star expansion} of $\hypergraph$ is the bipartite graph $\graph' = (\vertices', \edges')$ 
with $\vertices' = \vertices \dot{\cup} \edges$ and $\edges'= \{\{i,e\}\mid i\in \vertices, e \in \edges, i\in e\}$,
and we can uniquely reconstruct $\hypergraph$ from $\graph'$ if we know which of its parts corresponds to the original node set of $\hypergraph$.

\paragraph{Ollivier-Ricci Curvature for Graphs}
\label{prelim:orc}

\ollivier~(\orc) extends the notion of Ricci curvature, defined for Riemannian manifolds, to metric spaces equipped with a probability measure or, equivalently, a random walk \citep{ollivier2007ricci,ollivier2009ricci}.
On graphs, which are metric spaces with the shortest-path distance $\dist(\cdot,\cdot)$,
the \orc $\curvature$ of a pair of nodes  $\{i,j\}$ is defined as
\begin{equation}
	\curvature(i, j) 
	\defeq 1 - \frac{1}{\dist(i,j)}\wasserstein_1\mleft(\measure_i, \measure_j\mright)\;,~\text{and hence},~
	\curvature(i, j)  = 1 - \wasserstein_1\mleft(\measure_i, \measure_j\mright)~\text{if $i\adjacent j$}\;,\label{eq:orc:edges}
\end{equation}
where $\measure_i$ is a probability measure associated with node $i$ that depends measurably on $i$ and has finite first moment, and $\wasserstein_1$ is the \emph{Wasserstein distance} of order~$1$, which captures the amount of work needed to transport the probability mass from $\measure_i$ to $\measure_j$ in an optimal coupling.
The use of the shortest-path distance is necessary to ensure that \orc is also well-defined for pairs of non-adjacent nodes.
This definition on edges or pairs of nodes alludes to the fact that Ricci curvature is associated to tangent vectors of a manifold.
A common strategy to measure curvature at a node $i$ is to average over the curvatures of all edges incident with $i$ \citep{jost2014ollivier,banerjee2021spectrum}, i.e., 
\begin{equation}
 	\curvature(i) = \frac{1}{\degree(i)}\sum_{\{i,j\}\in E}\kappa(i,j)\;.\label{eq:orc:nodes}
\end{equation}
A popular probability measure that easily generalizes to weighted graphs and multigraphs is
\begin{equation}
	\measure_i^{\alpha}(j) \defeq \begin{cases}
		\alpha & j = i\\
		(1 - \alpha)\frac{1}{\degree(i)} & i \adjacent j\\
		0                      & \text{otherwise}\;,
	\end{cases}\label{eq:orc:measure}
\end{equation}
where $\alpha$ serves as a smoothing parameter \citep{lin2011ricci}.
With this definition, stacking the probability measures yields the transition matrix of an $\alpha$-lazy random walk.

\section{Theory}
\label{theory}

Having introduced the concept of hypergraphs and the definition of \ollivier~(\orc) for graphs, 
we now develop our framework for \orc on hypergraphs, 
called \ourmethod~(Ollivier-Ricci Curvature for Hypergraphs In Data). 
We focus our exposition on undirected, unweighted multi-hypergraphs, but \ourmethod straightforwardly generalizes to other hypergraph variants.

\subsection{Ollivier-Ricci Curvatures for Hypergraphs (\ourmethod Curvatures)}

As mentioned in \cref{prelim:hg}, 
hypergraphs differ from graphs in that edges can have any cardinality, 
and consequently, 
edges can intersect in more than one node, 
and nodes can co-occur in more than one edge.
When generalizing \orc as defined in \cref{prelim:orc} to hypergraphs, 
these peculiarities 
become relevant in two places:
(1)~in the generalization of the measure $\measure$ for nodes, 
and (2)~in the generalization of the distance metric $\wasserstein_1$. 
Construing the distance metric as a function \emph{aggregating} measures~(\aggregation), with $\aggregation\colon V^+ \to \reals$, we can rewrite \cref{eq:orc:edges} for pairs of nodes~$\{i, j\}$ as 
\begin{equation}
  \curvature(i,j) \defeq 1 - \frac{\aggregation(\mu_i, \mu_j)}{\dist(i, j)}\;,\label{eq:orc:rewritten}
\end{equation}
which facilitates its generalization; we will also use $\curvature(e)$ for (hyper)edges as a shorthand notation for \cref{eq:orc:rewritten}.
When defining probability measures and \aggregation functions on hypergraphs, 
we would like to retain as much flexibility as possible while also ensuring the following conditions:
\begin{enumerate}[label=\Roman*., ref=\Roman*, leftmargin=\widthof{III.}+\labelsep]
	\item \emph{Mathematical generalization.}
	For graphs, 
	\aggregation simplifies to the original \orc on graphs.\label{cond:generalization}

  \item \emph{Permutation invariance.} $\aggregation(e) = \aggregation(\sigma(e))$ for edges~$e$ and all node index permutations~$\permutation$.\label{cond:permutation-invariance}

  \item \emph{Scalability.} The probability measures and \aggregation functions should be efficiently computable.\label{cond:scalability}
\end{enumerate} 
Beyond these properties, 
we would also like to have the following \emph{interpretability} features to ascertain that a hypergraph curvature measure is a \emph{conceptual generalization} of \orc:
\begin{enumerate}[label=\Alph*., ref=\Alph*, leftmargin=\widthof{III.}+\labelsep]
	\item \emph{Probabilistic intuition.} 
		\label{item:probabilistic-intuition}
		The probability measures assigned to nodes should correspond to a semantically sensible random walk on the hypergraph.
		\item \emph{Optimal transport intuition.}
		\label{item:transport-intuition}
		The generalization of the distance metric (\aggregation) should have a semantically sensible interpretation in terms of optimal transport.
		\item \emph{Geometric intuition.} \label{item:geometric-intuition}
		Edges in hypercliques should have positive curvature, 
		edges in hypergrids should have curvature zero, and
		edges in hypertrees should have negative curvature.
\end{enumerate}
We now specify probability measures and \aggregation functions for which the conditions above hold.

\paragraph{Probability Measures ($\mu$).}

In graphs, the most natural probability measures are induced by the $\smoothing$-lazy random walk given in \cref{eq:orc:measure}: 
With probability~$\smoothing$, we stay at the current node~$i$, and with probability $\nicefrac{(1-\smoothing)}{\degree(i)}$, we move to one of its neighbors. 
There are at least three direct extensions of this formulation to hypergraphs 
that all retain this probabilistic intuition, 
thus fulfilling the requirement of  Feature~\ref{item:probabilistic-intuition}.
These extensions, illustrated in \cref{fig:walks}, 
differ only in how they distribute the $(1-\smoothing)$ probability mass in \cref{eq:orc:measure} from node $i$ to the nodes in $i$'s neighborhood.
Given a hypergraph $\hypergraph$, 
for $i$ and $j$ with $i\adjacent j$, first, we could define
\begin{equation}
  \mu_i^\enrw(j) \defeq (1 - \smoothing)\frac{1}{|\neighborhood(i)|}\;,
  \label{eq:orchard:measure:rwnodes}
\end{equation}
by which we pick a neighbor~$j$ of node~$i$ uniformly at random.
We call this the \emph{equal-nodes random walk}~(\enrw), 
which is a random walk on the \emph{unweighted clique expansion} of $\hypergraph$.
Second, we could set
\begin{equation}
  \mu_i^\eerw(j)
	\defeq (1 - \smoothing)\frac{1}{\degree(i)-|\{e\ni i\mid |e| = 1\}|}\underset{e\supseteq \{i,j\}}{\sum}\frac{1}{|e|-1}\;,\label{eq:orchard:measure:rwedges}
\end{equation}
which first picks an edge $e\ni i$ with $|e|\geq 2$, 
then picks a node $j\in e\setminus \{i\}$, both uniformly at random.
We call this the \emph{equal-edges random walk}~(\eerw), 
which is a two-step random walk on the \emph{unweighted star expansion} of $\hypergraph$, 
starting at a node $i\in \vertices$, 
and non-backtracking in the second step.
It underlies the curvatures studied by \citet{asoodeh2018curvature} and \citet{banerjee2021spectrum}.
Third, we could define 
\begin{equation}
	\mu_i^\werw(j)
	\defeq 
		(1 - \smoothing)\underset{e\supseteq \{i,j\}}{\sum}\frac{|e|-1}{\underset{f\ni i}{\sum}\big(|f|-1\big)}\frac{1}{|e|-1}
		= (1 - \smoothing) \frac{|\{e\in E\mid \{i,j\}\subseteq e\}|}{\underset{f\ni i}{\sum}\big(|f|-1\big)}
		\;,\label{eq:orchard:measure:weightedrwedges}
\end{equation}
first picking an edge $e$ incident with $i$ with probability proportional to its cardinality, 
then picking a node $j\in e\setminus \{i\}$ uniformly at random. 
We call this the \emph{weighted-edges random walk}~(\werw): 
a two-step random walk from a node $i\in \vertices$ on a specific \emph{directed weighted star expansion} of $\hypergraph$ 
whose second step is non-backtracking---or equivalently, a random walk on a \emph{weighted clique expansion} of $\hypergraph$. 

\begin{figure}[t]
	\centering
	\begin{subfigure}[b]{0.495\linewidth}
		\centering
		\includegraphics[scale=0.7]{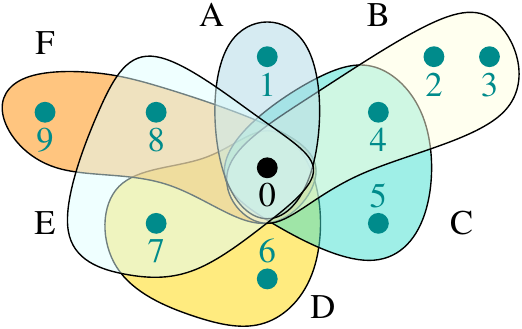}
		\subcaption{Node $0$ and its neighborhood}
	\end{subfigure}~
	\begin{subfigure}[b]{0.495\linewidth}
		\centering
		\includegraphics[scale=0.7]{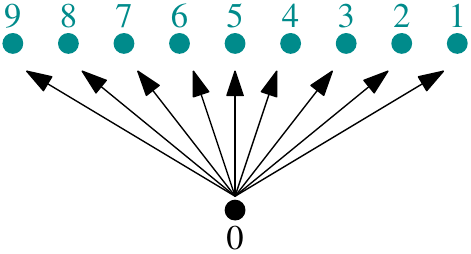}
		\subcaption{Equal-Nodes Random Walk (\enrw)}\label{fig:walks:enrw}
	\end{subfigure}
	\begin{subfigure}[b]{0.495\linewidth}
		\centering
		\includegraphics[scale=0.7]{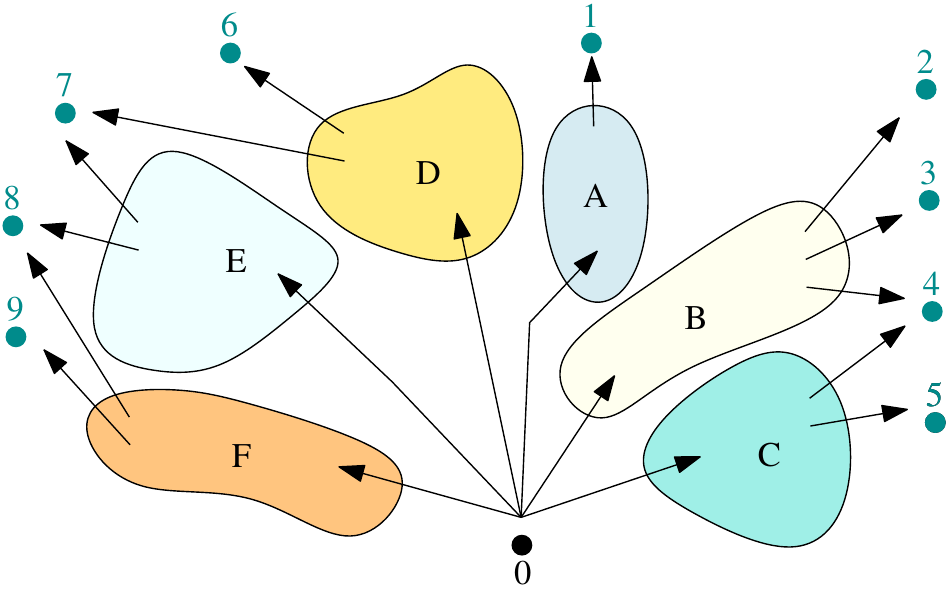}
		\subcaption{Equal-Edges Random Walk  (\eerw)}
	\end{subfigure}~
	\begin{subfigure}[b]{0.495\linewidth}
		\centering
		\includegraphics[scale=0.7]{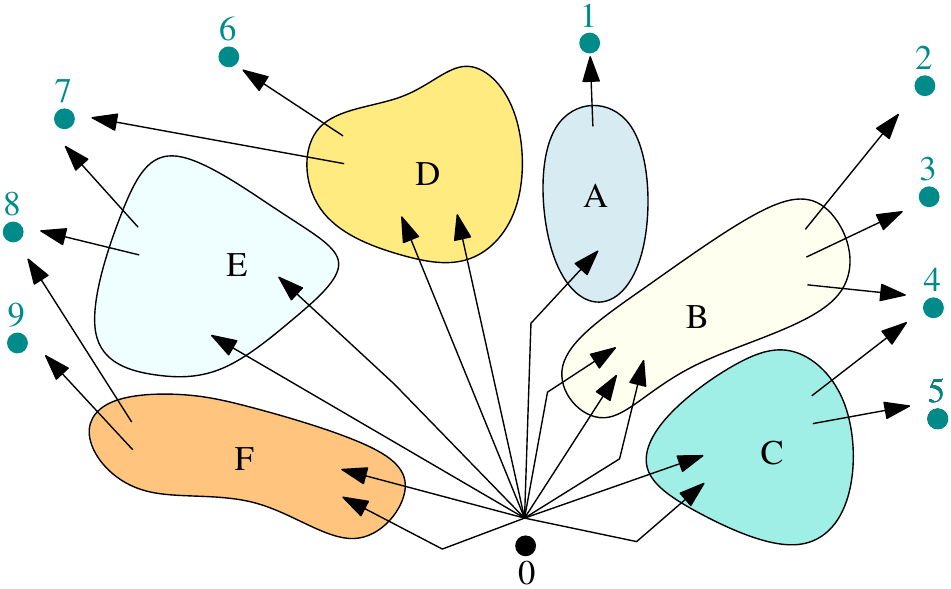}
		\subcaption{Weighted-Edges Random Walk (\werw)}
		\label{fig:walks:werw}
	\end{subfigure}
	\caption{\ourmethod's probability measures are based on random walks, depicted for the neighborhood of a node~0.
    Arrows outgoing from the same node or edge are traversed with uniform probability.
  }
	\label{fig:walks}
\end{figure}

\paragraph{Similarity Measures (\aggregation).}

In the original formulation of \orc, i.e., \cref{eq:orc:edges},
when determining the curvature of an edge $\{i,j\}$,
the Wasserstein distance $\wasserstein_1$ is used to aggregate the probability measures of $i$ and $j$.
There are at least three different extensions of this aggregation scheme to hypergraphs that retain an optimal transport intuition, as required by Feature~\ref{item:transport-intuition}.
Leveraging that an edge $e \subseteq V$ is simply a set of nodes, 
the easiest extension is to leave the aggregation function unchanged.
We continue determining the curvature for pairs of nodes, 
and account for the edges in $\hypergraph$ only in the definition of our probability measure. 
In this case, we could derive a curvature for an edge $e$ as the average over all curvatures of node pairs contained in $e$, i.e., we could define \aggregation as
\begin{equation}
 	\aggregation_\mean(e) \defeq \frac{2}{|e|(|e|-1)}\sum_{\{i,j\}\subseteq e}\wasserstein_1\!\big(\mu_i, \mu_j\big)\;.
 \end{equation}
This is equivalent to computing the curvature of $e$ based on the average over all $\wasserstein_1$ distances of probability measures associated with nodes contained in $e$:
\begin{align}
	\kappa_\mean(e) \defeq 
	1 - \aggregation_\mean(e) = 
	1 - \frac{2}{|e|(|e|-1)}\sum_{\{i,j\}\subseteq e}\wasserstein_1\mleft(\mu_i, \mu_j\mright) = 
	 \frac{2}{|e|(|e|-1)}\sum_{\{i,j\}\subseteq e}\kappa(i,j)\;.
\end{align}
Intuitively, this definition assesses the mean amount of work needed to transport the probability mass from one node in $e$ to another node in $e$.
Alternatively, and still keeping with the intuition from optimal transport,  we can define \aggregation as
\begin{equation}  
  \aggregation_\bary(e) \defeq \frac{1}{|e| - 1}\sum_{i\in e}\wasserstein_1\mleft(\mu_i, \bar{\mu}\mright)\;,~~\text{and consequently},~~\curvature_\bary(e) \defeq 1 - \aggregation_\bary(e)\;,
\end{equation}
where $\bar{\mu}$ denotes the Wasserstein barycenter of the probability measures of nodes contained in $e$, 
and the denominator generalizes the original $\dist(i,j)$.
\cite{asoodeh2018curvature} use this aggregation function.
Intuitively, $\aggregation_\bary$ is proportional to the minimum amount of work needed to transport all probability mass from the probability measures of the nodes to one place, 
with the caveat that this place need not correspond to a node in the underlying hypergraph. 
Finally, we can capture the maximum amount of work needed to transport all probability mass from one node in $e$ to another node in $e$ as
\begin{equation}
  \aggregation_\maxi(e) \defeq \max\mleft\{\wasserstein_1(\mu_i, \mu_j) \mid \{i, j\} \subseteq e\mright\}\;,~~\text{and consequently},~~\curvature_\maxi(e) \defeq 1 - \aggregation_\maxi(e)\;.
\end{equation}

Independent of the choice of \aggregation, 
the curvature at a node $i$ can be defined as the mean of all curvatures of meaningful directions containing $i$, i.e.,
\begin{equation}
	\curvature^{\neighborhood}(i) \defeq \frac{1}{|\neighborhood(i)|}\sum_{j\in\neighborhood(i)}\curvature(i,j)\;,
\end{equation}
or it can be derived as the mean of all curvatures of edges containing $i$, i.e., 
\begin{equation}
	\curvature^{\edges}(i) \defeq \frac{1}{\degree(i)}\sum_{e\ni i}\curvature(e)\;.
\end{equation}
Finally, since $\hypergraph$ is connected, we can define the
curvature of an arbitrary subset of nodes~$s \subseteq\vertices$ as
\begin{equation}
  \curvature(s) \defeq 1 - \frac{\aggregation(s)}{\dist(s)}\;,
  \label{eq:orchard:curvature:general}
\end{equation}
where \aggregation can be any of our aggregation functions, and 
$\dist(s) \defeq \max\mleft\{\dist\mleft(i, j\mright) \mid \mleft\{i, j\mright\} \subseteq s \mright\}$
refers to the \emph{extent} of the subset~$s$.
Note that for $s \in \edges$, $\dist(s) = 1$,
and thus, \cref{eq:orchard:curvature:general} is consistent with our previous definitions of hyperedge curvatures.

\subsection{Properties of \ourmethod Curvatures}

Having introduced our probability measures ($\mu$) and aggregation functions (\aggregation), 
we now analyze their properties and the properties of the resulting curvatures. 
All proofs are deferred to \cref{apx-proofs}.
First, we note that $\mu^\enrw$, $\mu^\eerw$, and $\mu^\werw$ are equivalent for certain hypergraph classes, and all aggregation functions coincide for graphs.
\begin{restatable}{lem}{measureequality}
\label{lem:dispersions-general}
	For graphs and $\uniformity$-uniform, $\regularity$-regular, $\cooccurrence$-cooccurrent hypergraphs, 
	$\mu^\enrw = \mu^\eerw =  \mu^\werw$.
\end{restatable}
\begin{restatable}{lem}{aggregationequality}
\label{lem:aggregation-general}
	For graphs, i.e., $2$-uniform hypergraphs, we have $\aggregation_\mean(e) = \aggregation_\bary(e) = \aggregation_\maxi(e)$ for all edges~$e \in \edges$.
\end{restatable}
Taken together, \cref{lem:dispersions-general} and
\cref{lem:aggregation-general} imply that for graphs, \ourmethod simplifies to \orc, 
regardless of the choice of probability measure and aggregation
function.
This fulfills Condition~\ref{cond:generalization}.
Moreover, \emph{all} our aggregation functions are permutation-invariant by construction, thus satisfying Condition~\ref{cond:permutation-invariance}.
Concerning Condition~\ref{cond:scalability}, $\curvature_\mean$
and $\curvature_\maxi$ exhibit better scalability than $\curvature_\bary$, 
as Wasserstein barycenters are harder to compute than individual distances~\citep{cuturi14-fast}.
Another reason to prefer $\curvature_\mean$ and
$\curvature_\maxi$ over $\curvature_\bary$ is the existence of upper and
lower bounds that are easy to calculate.
To this end, let $\dist_{\min}(\hypergraph) \defeq \min\mleft\{ \dist\mleft(u,
v\mright) \mid u \neq v \in \vertices \mright\}$ be the smallest nonzero
distance in~$\hypergraph$, 
and let $\mleft\|\cdot\mright\|_1$ refer
to the $L_1$ norm of a vector. 
We then obtain the following bounds for $\curvature_\mean$ and $\curvature_\maxi$.
\begin{restatable}{thm}{curvatureboundmean}
  For any probability measure $\mu$ and $C(e) \defeq \nicefrac{2}{\mleft|e\mright|\mleft(\mleft|e\mright|-1\mright)}$, 
  the curvature~$\curvature_\mean(e)$ of an edge $e\in\edges$ is bounded by
	\begin{equation}
    1 - \diam(\hypergraph)  C(e)\!\!\!\sum_{\{i, j\}\subseteq e} \!\mleft\|\mu_i - \mu_j\mright\|_1 \leq \curvature_\mean(e) \leq 1 - \dist_{\min}(\hypergraph) C(e)\!\!\!\sum_{\{i, j\} \subseteq e} \!\mleft\|\mu_i - \mu_j\mright\|_1\;.
	\end{equation}
\end{restatable}

\begin{restatable}{thm}{curvatureboundmax}
  For any probability measure $\mu$, the curvature $\curvature_\maxi(e)$ of an edge $e\in\edges$ is bounded by
  \begin{equation}
    1 - \diam(\hypergraph) \max_{\{i, j\}\subseteq e} \mleft\|\mu_i - \mu_j\mright\|_1 \leq \curvature_\maxi(e) \leq 1 - \dist_{\min}(\hypergraph) \max_{\{i, j\} \subseteq e} \mleft\|\mu_i - \mu_j\mright\|_1\;.
  \end{equation}
\end{restatable}

Directly from our definitions, we further obtain the following relationships between $\curvature_\mean$, $\curvature_\bary$, and~$\curvature_\maxi$, 
and between \ourmethod curvatures on hypergraphs and \orc on their unweighted clique expansions.
\begin{restatable}{cor}{curvaturerelations}
	Given a hypergraph $\hypergraph = (\vertices, \edges)$, 
	$\curvature_\maxi(e) \leq \curvature_\mean(e)$ 
	and 
	$\curvature_\maxi(e) \leq \curvature_\bary(e)$
	for all $e\in \edges$.
\end{restatable}
\begin{restatable}{cor}{expansionrelation}\label{cor:expansionrelation}
	Given a  hypergraph $\hypergraph = (\vertices, \edges)$ and its unweighted clique expansion $\graph^\circ = (\vertices,\edges^\circ)$, 
	for $\{i,j\}\in\edges^\circ$, 
	the \orc $\curvature(i,j)$ in $\graph^\circ$ equals its \ourmethod curvature  $\curvature(i,j)$ of direction $\{i,j\}\subseteq\vertices$ in $\hypergraph$ with\thinspace $\mu^{\enrw}$, 
	and\thinspace the\thinspace \orc $\curvature(i)$ of $i\in \vertices$ in\thinspace $\graph^\circ$ equals its \ourmethod curvature\thinspace $\curvature^{\neighborhood}(i)$ in $\hypergraph$ with\thinspace $\mu^{\enrw}$.
\end{restatable}
\cref{cor:expansionrelation} clarifies that the equal-nodes random walk establishes the connection between \ourmethod and \orc on graphs.
Moreover, \ourmethod curvatures capture relations between \emph{global} properties and \emph{local} measurements, similar to\thinspace the\thinspace Bonnet--Myers\thinspace theorem
in Riemannian geometry~\citep{Myers41}.
\begin{restatable}{thm}{bonnet}
	Given a subset of nodes $s \subseteq\vertices$ and 
	an arbitrary probability measure $\measure$,
  	let $\delta_i$ denote a Dirac
  	measure at node~$i$, 
  	and let $\jump(\mu_i) \defeq \wasserstein_1(\delta_i, \mu_i)$
  	denote the \emph{jump probability} of $\mu_i$.
  If 
  \begin{inparaenum}[(i)]
  	\item~all curvatures based on $\mu$ are strictly positive,
  	i.e., $\curvature(s) \geq \curvature > 0$ for all $s \subseteq \vertices$, and
    \item\label{item:aggregation-bound}~$\wasserstein_1(\mu_i, \mu_j) \leq \aggregation(s)$
  	for $\{i, j\} = \argmax\mleft(\dist\mleft(s\mright)\mright)$, then
  \end{inparaenum}
  \begin{equation}\label{eq:jump}
    \dist(s) \leq \frac{\jump(i) + \jump(j) }{\curvature(s)}\;.
  \end{equation}
  \label{lem:Bonnet}
\end{restatable}
Note that condition~(\ref{item:aggregation-bound}) of \cref{lem:Bonnet} is always satisfied by $\aggregation_\maxi$.
Finally, 
in \cref{apx-proofs},
we generalize the concepts of cliques, grids, and trees (prototypical positively curved, flat, and negatively curved graphs)
to hypergraphs,
and we prove the following lemmas to ensure that \ourmethod curvatures respect our geometric intuition, as required by Feature~\ref{item:geometric-intuition}. 

\begin{restatable}[Hyperclique curvature]{thm}{hyperclique}
	For an edge $e$ in a hyperclique  $\hypergraph = (\vertices, \edges)$ on $n$ nodes 
	with edges $\edges = \binom{\vertices}{\uniformity}$ for some $r\leq n$, with $\smoothing = 0$,
	\begin{align*}
		\curvature(e) 
		= 1 - \frac{1}{n-1},\text{~i.e.},
		~
		\lim_{n\rightarrow\infty}\curvature(e) = 1,\text{~	independent of $\uniformity$.}
	\end{align*}
\end{restatable}
\begin{restatable}[Hypergrid curvature]{thm}{hypergrid}
	For an edge $e$ in a $\uniformity$-uniform, 
	$\regularity$-regular hypergrid, with $\smoothing = 0$,
	$\curvature(e) 
		= 0$,
	independent of  $\uniformity$ and $\regularity$.
\end{restatable}
\begin{restatable}[Hypertree curvature]{thm}{hypertree}
	For an edge $e$ in a $\uniformity$-uniform, 
	$\regularity$-regular, $1$-intersecting hypertree, 
	\begin{align*}
		\text{with $\smoothing = 0$},~\curvature(e) 
		= 1 - \bigg(\frac{3(\regularity-1)}{\regularity}+\frac{1}{(\uniformity-1)\regularity}\bigg),
		~\text{i.e.},~
		\lim_{\regularity\rightarrow\infty}\curvature(e) = -2,\text{~independent of  $\uniformity$.}
	\end{align*}
\end{restatable}

\newcommand{\apsa}{aps-a\xspace}
\newcommand{\apsva}{aps-av\xspace}
\newcommand{\apsvcout}{aps-cv\xspace}
\newcommand{\iclr}{iclr\xspace}
\newcommand{\dblp}{dblp\xspace}
\newcommand{\dblpv}{dblp-v\xspace}
\newcommand{\ndcai}{ndc-ai\xspace}
\newcommand{\ndcpc}{ndc-pc\xspace}
\newcommand{\stex}{stex\xspace}
\newcommand{\sha}{sha\xspace}
\newcommand{\mus}{mus\xspace}
\newcommand{\syn}{syn\xspace}
\newcommand{\sync}{syn-c\xspace}
\newcommand{\synr}{syn-r\xspace}
\newcommand{\syns}{syn-s\xspace}

\section{Experiments}
\label{experiments}

Having established in~\cref{theory} that \ourmethod curvatures have our desired theoretical properties, 
and finding that they strictly generalize both \orc on graphs and existing definitions of hypergraph \orc, 
we now seek to ascertain that they are also meaningful in practice.
We ask the following questions:
\begin{enumerate}[label=\textbf{Q\arabic*},  leftmargin=\widthof{III.}+\labelsep]
	\item \textbf{Parametrization.} 
	How do our choices of $\smoothing$, $\mu$, and $\aggregation$ impact \ourmethod curvatures?
	\item \textbf{Hypergraph exploration.} 
	How can \ourmethod curvatures help us in exploring hypergraphs?
	\item \textbf{Hypergraph learning.}
	How can \ourmethod curvatures help us in hypergraph learning tasks?
\end{enumerate}

To address these questions, we experiment with data from different domains, spanning several orders of magnitude.
We investigate four \emph{individual real-world hypergraphs} in which edges represent co-authorship (\apsa, \dblp) and FDA-registered drugs (\ndcai, \ndcpc), 
six \emph{collections of real-world hypergraphs}
in which edges represent questions on Stack Exchange Sites (\stex), 
co-authorship by venues (\apsva, \dblpv),
co-citation by venues (\apsvcout),
chords in music pieces (\mus), 
and character cooccurrence on stage in Shakespeare's plays (\sha), 
as well as three \emph{collections of synthetic hypergraphs} based on different generative models (\sync, \synr, \syns),
for a total of 4\,321 hypergraphs.
We summarize their basic properties in \cref{tab:data}, 
and give more details on their statistics, semantics, and provenance in \cref{apx-datasets}. 
We implement \ourmethod in Julia and Python.
Our experiments are run on AMD EPYC 7702 CPUs with up to 256 cores. 
We discuss our implementation and results in more detail in \cref{apx-implementation,apx-results}, 
and make all our code, data, and results publicly available.\!\footnote{\oururl
} 

\begin{table}[t]
	\centering
	\caption{Hypergraphs used in \ourmethod experiments cover several domains and orders of magnitude. 
	$n$ and $m$ are node and edge counts, $\nicefrac{n}{m}$ is the aspect ratio, $c$ is the number of filled cells in the node-to-edge incidence matrix, $\nicefrac{c}{nm}$ is the density,  
	and $N$ is the number of hypergraphs in a collection.
}\label{tab:data}
	\begin{subfigure}{\linewidth}
		\centering
		\setlength{\tabcolsep}{4.25pt}
		\subcaption{Individual Hypergraphs}
		\begin{tabular}{rllrrrrr}
\toprule
&Nodes & Edges & $n$ & $m$ & $\nicefrac{n}{m}$ & $c$ & $\nicefrac{c}{nm}$ \\
\midrule
aps-a & Authors & APS Papers & 505\,827 & 688\,707 & 0.7345 & 2\,480\,373 & 0.000007 \\
dblp & Authors & DBLP Papers & 3\,108\,658 & 6\,011\,388 & 0.5171 & 19\,411\,479 & 0.000001 \\
ndc-ai & Active Ingr. & NDC Drugs & 7\,090 & 131\,450 & 0.0539 & 224\,084 & 0.000240 \\
ndc-pc & Pharm. Classes & NDC Drugs & 1\,263 & 70\,101 & 0.0180 & 273\,088 & 0.003084 \\
\bottomrule
\end{tabular}
 \vspace*{6pt}
	\end{subfigure}
	\begin{subfigure}{\linewidth}
		\setlength{\tabcolsep}{5.45pt}
		\centering
		\subcaption{Hypergraph Collections}\label{tab:hgcolldata}
		\begin{tabular}{rlllrrr}
\toprule
&Nodes & Edges & Graphs & $N$ & $(\nicefrac{n}{m})_{\max}$ & $(\nicefrac{c}{nm})_{\max}$ \\
\midrule
\apsva & Authors & APS Papers & Journals & 19 & 4.698182 & 0.005216 \\
\apsvcout & APS Cited P. & APS Citing P. & Journals & 19 & 1.396552 & 0.028430 \\
dblp-v & Authors & DBLP Papers & (Groups of) Venues & 1\,193 & 5.599424 & 0.002443 \\
mus & Frequencies & Chords & Music Pieces & 1\,944 & 1.454545 & 0.375000 \\
stex & Tags & Questions & StackExchange Sites & 355 & 1.233449 & 0.121528 \\
sha & Characters & Stage Groups & Shakespeare's Plays & 37 & 0.554054 & 0.304688 \\\midrule
syn-c&\multicolumn{3}{l}{Hypergraph Configuration Models}&250&0.5&0.005\\
syn-r&\multicolumn{3}{l}{Erd\H{o}s-R\'enyi Random Hypergraph Models}&250&0.5&0.005\\
syn-s&\multicolumn{3}{l}{Hypergaph Stochastic Block Models}&250&0.5&0.005\\
\bottomrule
\end{tabular}
 	\end{subfigure}
\end{table}

\paragraph{Q1 Parametrization.}
To understand how our choices of $\smoothing$, $\mu$, and $\aggregation$ impact \ourmethod curvatures, 
we first compute the pairwise mutual information between \ourmethod edge curvatures with 36~different parametrizations. 
As illustrated in \cref{fig:nmiheatmaps}, 
while changing $\smoothing$ for the same combination of $\measure$ and $\aggregation$ has similar effects across hypergraphs, 
there is no uniform pattern in the relationships between different combinations of $\measure$ and $\aggregation$. 
This underscores the fact that the various notions of \ourmethod curvature are not redundant but rather emphasize distinct aspects of hypergraph structure. 
For a fine-grained view of the differences between parametrizations, 
we inspect the distributions of our four curvature types, 
\begin{inparaenum}[(i)]
	\item edge curvature $\curvature(e)$, 
	\item edge-averaged node curvature $\curvature^{\edges}(i)$,
	\item directional curvature $\curvature(i,j)$ for all $\{i,j\}\subseteq e \in\edges$, and
	\item direction-averaged node curvature $\curvature^{\neighborhood}(i)$,
\end{inparaenum}
for each of our 36~parametrizations. 
By construction, 
directional curvature and direction-averaged node curvature do not vary with the choice of \aggregation, 
and $\curvature_\maxi$ lower-bounds $\curvature_\mean$ for edge curvatures and edge-averaged node curvatures.
However, the differences between $\curvature_\maxi$ and $\curvature_\mean$ vary across graphs,
while consistently, 
the larger $\smoothing$, the more concentrated our curvature distributions~(\cref{apx-results}).

\begin{figure}[t]
	\centering
	\vspace*{-1em}\hspace*{-0.5em}\begin{subfigure}[t]{0.25\linewidth}
		\centering
		\hspace*{-0.75em}\includegraphics[height=3.8cm]{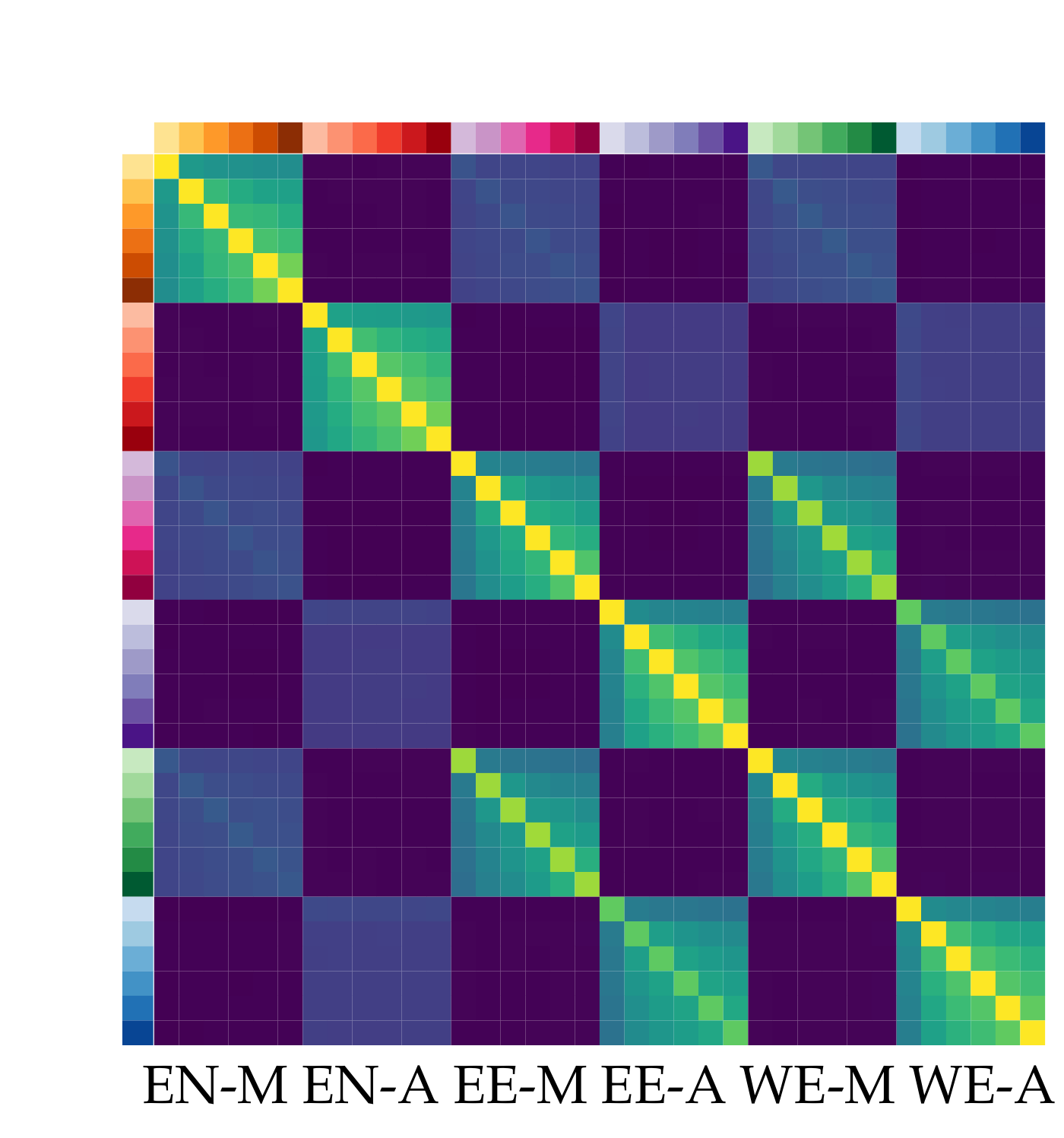}
		\vspace*{-0.25em}\subcaption{(5,10)-regular \sync}
	\end{subfigure}~
	\hspace*{-0.75em}\begin{subfigure}[t]{0.25\linewidth}
		\centering
		\hspace*{-1em}\includegraphics[height=3.8cm]{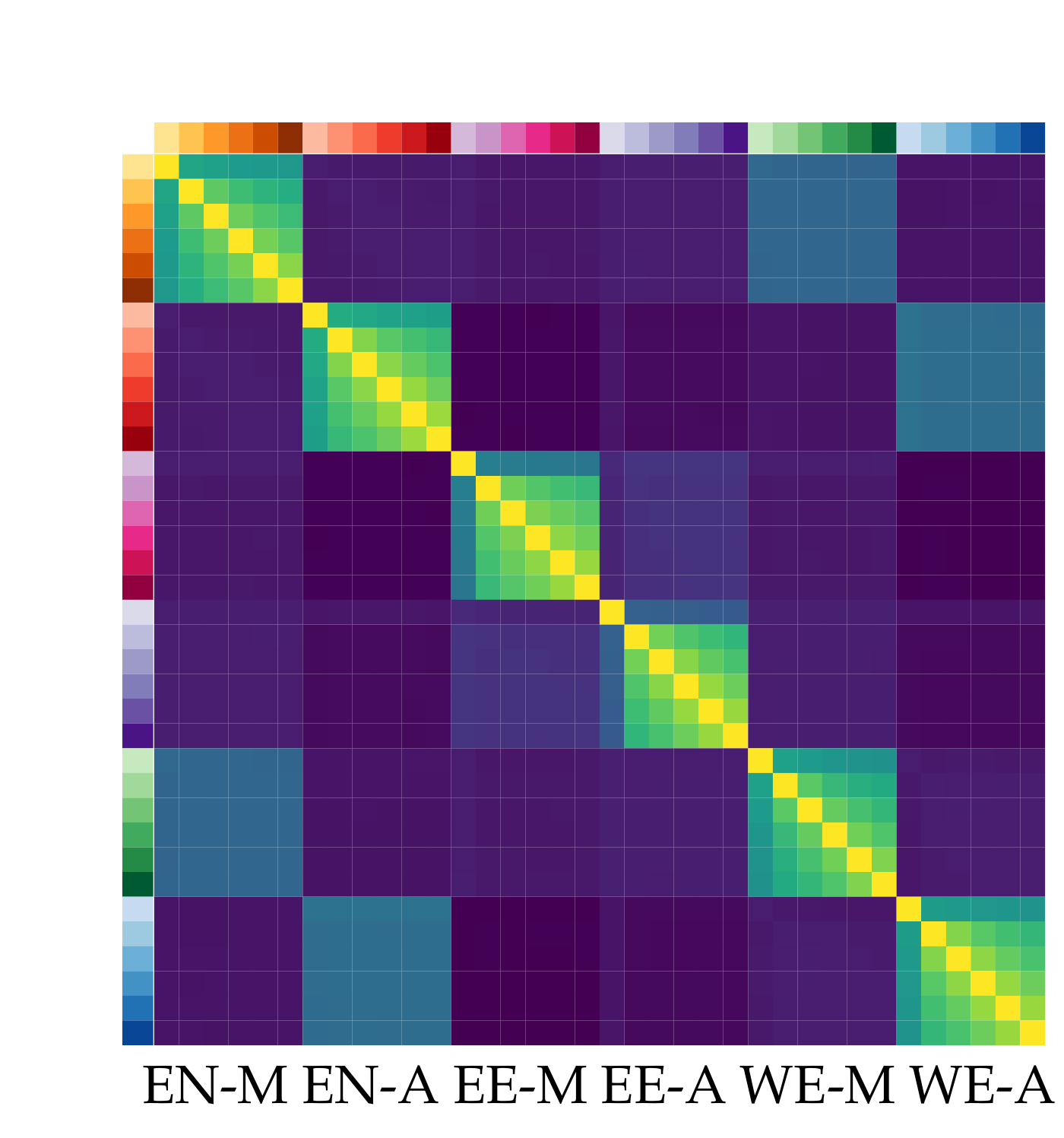}
		\vspace*{-0.25em}\subcaption{2-community \syns}
	\end{subfigure}~
	\hspace*{-0.75em}\begin{subfigure}[t]{0.25\linewidth}
		\centering
		\hspace*{-1em}\includegraphics[height=3.8cm]{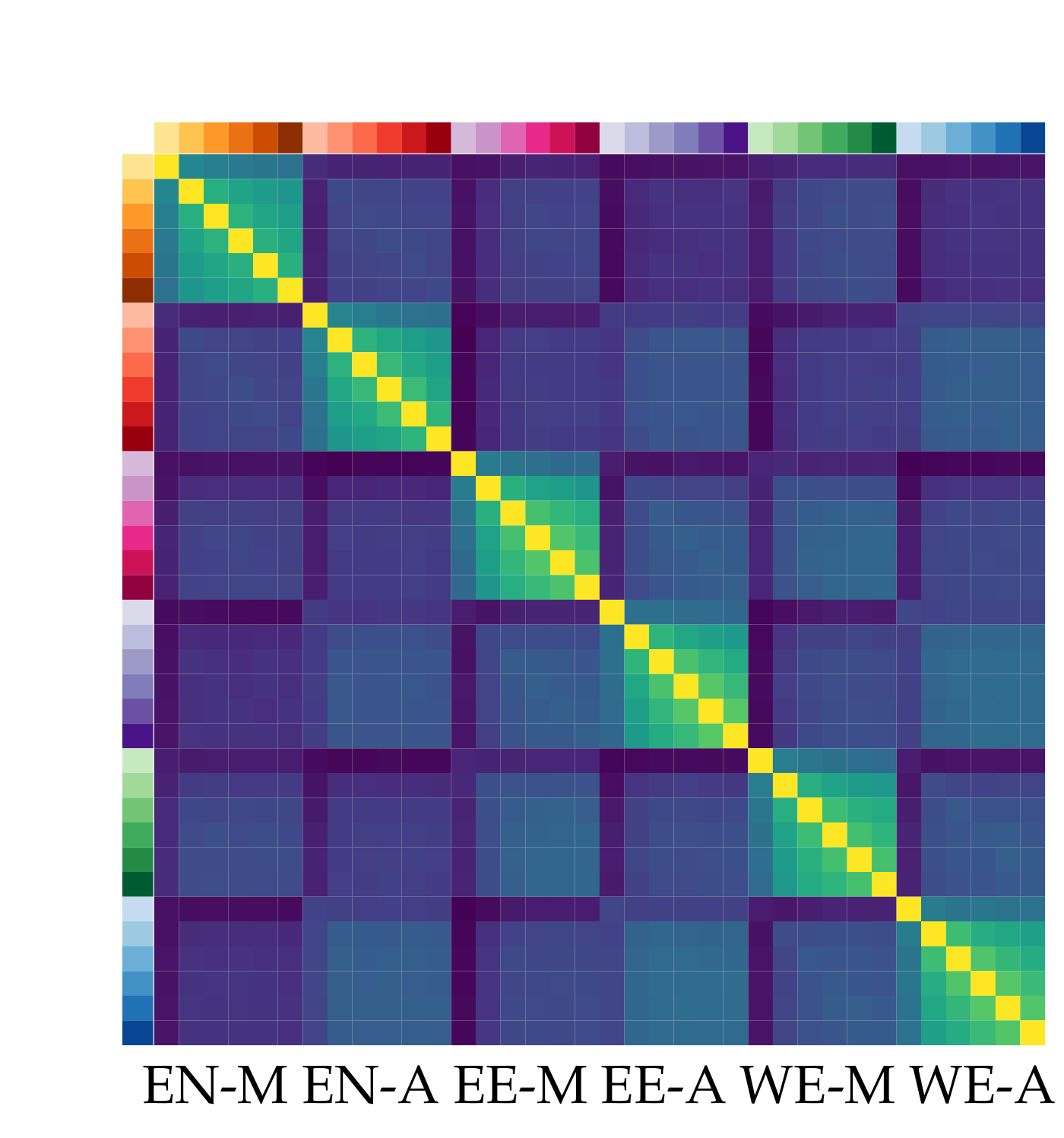}
		\vspace*{-0.25em}\subcaption{\ndcai}
	\end{subfigure}~
	\hspace*{-0.75em}\begin{subfigure}[t]{0.25\linewidth}
		\centering
		\hspace*{-1em}\includegraphics[height=3.8cm]{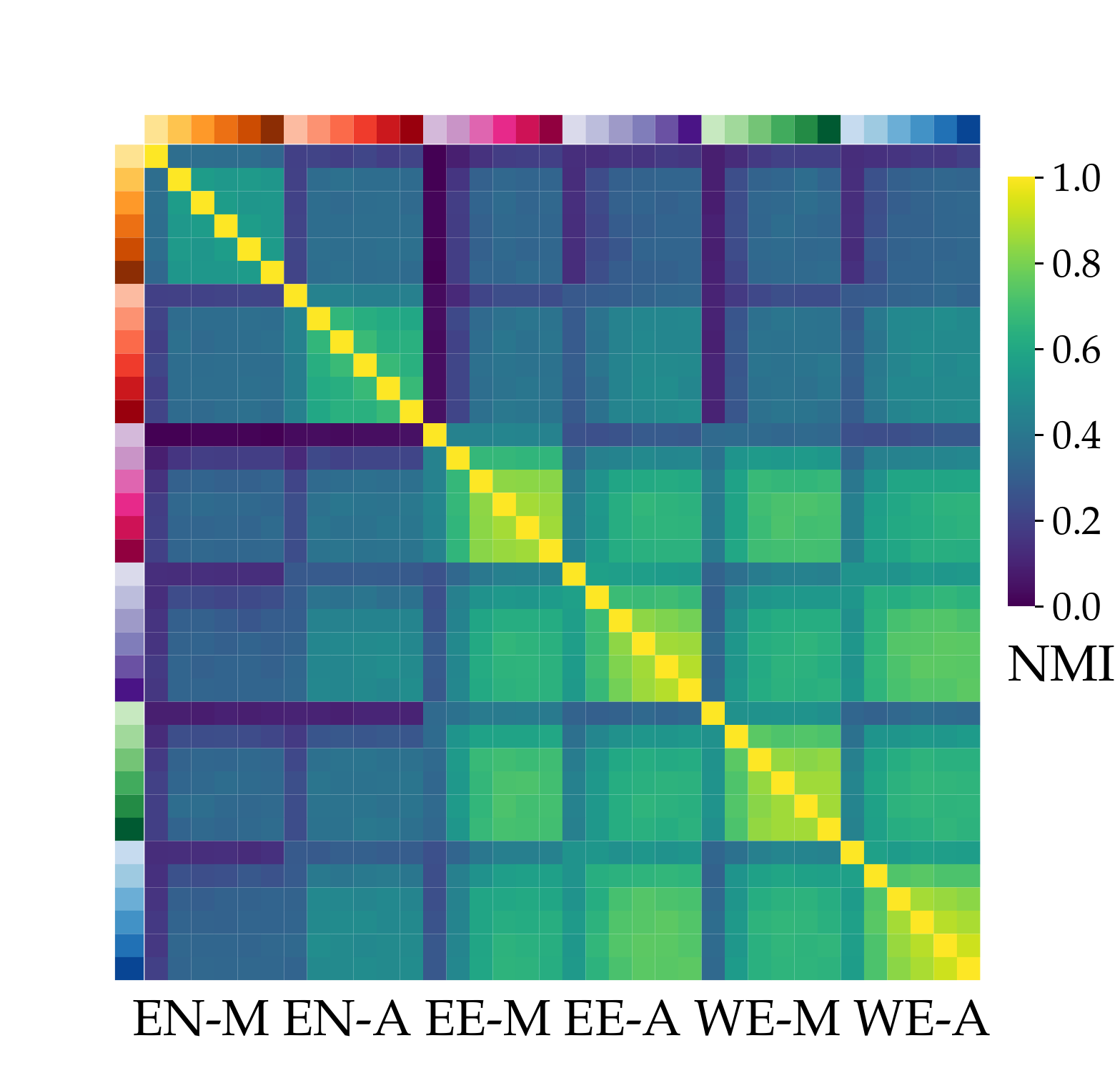}
		\vspace*{-1.35em}\subcaption{\ndcpc}
	\end{subfigure}~
	\caption{\ourmethod curvature notions are non-redundant. 
		We show the Min-Max-Normalized Mutual Information (NMI) between \ourmethod edge curvatures with  36~different parametrizations, 
		using probability measures $\mu^\enrw$ (EN), $\mu^\eerw$ (EE), or $\mu^\werw$ (WE), aggregations $\aggregation_\maxi$ (M) or $\aggregation_\mean$ (A), 
		and $\smoothing \in \{0.0,0.1,0.2,0.3,0.4,0.5\}$ (ordered $\rightarrow$, $\downarrow$), 
		for two synthetic and two real-world hypergraphs. 
	}\label{fig:nmiheatmaps}
\end{figure}

\paragraph{Q2 Hypergraph Exploration.}
To explore \emph{individual graphs}, 
we perform case studies on graphs from the \apsvcout collection, 
leveraging that most nodes in these graphs also occur as edges. 
We scrutinize the relationships between node and edge curvatures, other local node and edge statistics, and article metadata. 
We observe that curvature values span a considerable range even for articles with otherwise comparable statistics, 
but the curvature distributions of influential papers appear to differ systematically from those of less influential papers~(\cref{apx-results}).
Exploring \emph{graph collections}, 
we run kernel PCA (kPCA) \citep{schoelkopf1997kpca} with a radial basis function kernel (RBF kernel) and curvatures or other local features known to be powerful baselines  \citep{cai18effective},
e.g., node degrees and neighborhood sizes, 
as inputs to jointly embed graphs from a collection.
We statistically bootstrap the maximum mean discrepancy (MMD) \citep{gretton2006kernel} to test the null hypothesis that the feature distributions of two graphs are equal.
As shown in \cref{fig:kpca}, 
\ourmethod curvatures result in more interpretable embeddings and more discriminative tests than other local features.

\begin{figure}[t]
	\centering
	\vspace*{-1em}\begin{subfigure}[t]{0.66\linewidth} 
		\centering
		\includegraphics[height=3.8cm]{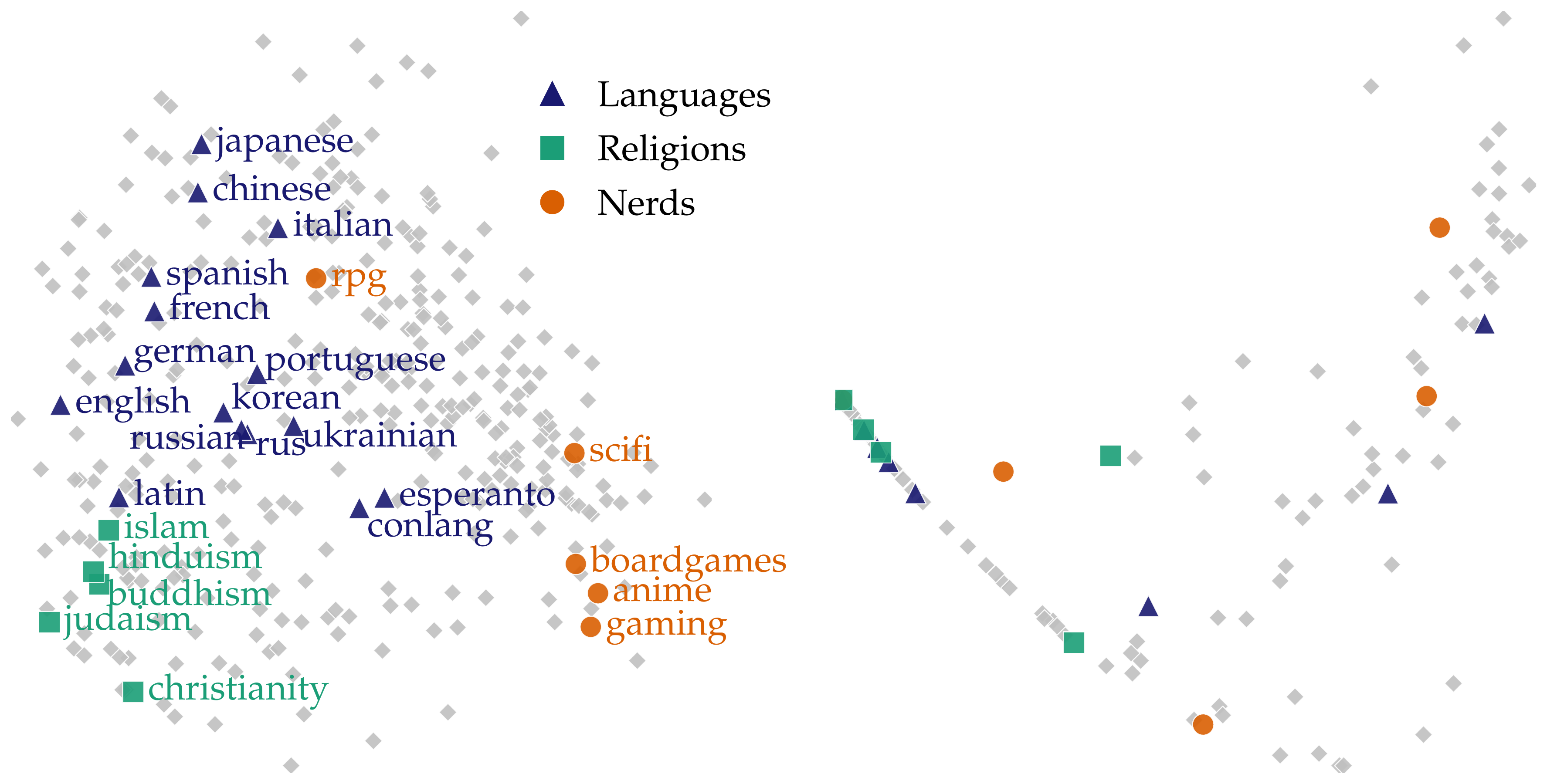}
		\vspace*{-2em}
	\end{subfigure}~
	\begin{subfigure}[t]{0.33\linewidth}
		\centering
		\includegraphics[height=3.8cm]{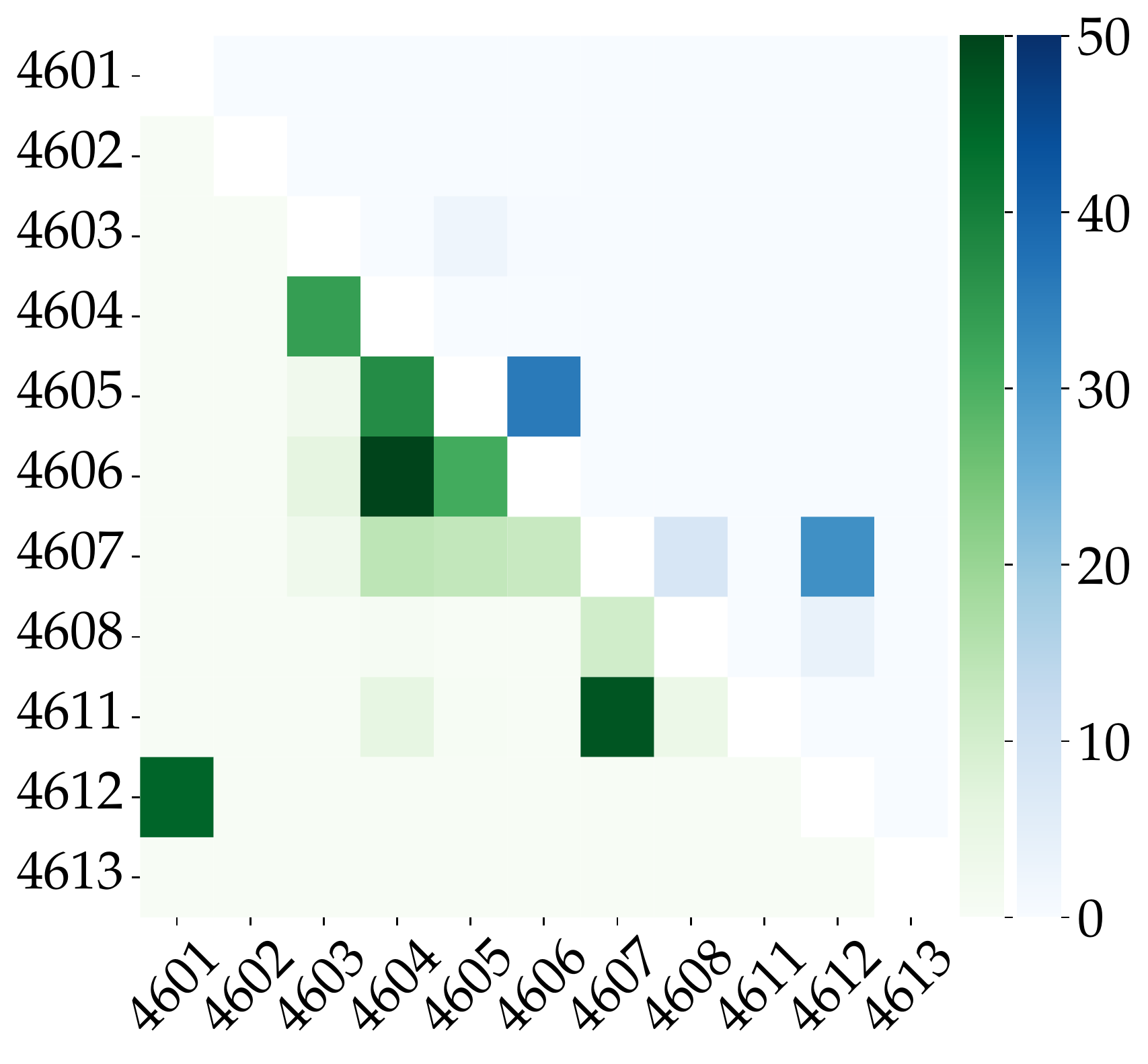}
	\end{subfigure}
	\begin{subfigure}{0.32\linewidth}
		\centering
		\subcaption{kPCA (directional curvature)}\label{fig:kpca-curvy}
	\end{subfigure}~
	\begin{subfigure}{0.32\linewidth}
		\centering
		\subcaption{kPCA\thinspace (edge\thinspace neighborhood\thinspace size)}\label{fig:kpca-noncurvy}
	\end{subfigure}~
	\begin{subfigure}{0.34\linewidth}
		\centering
		\subcaption{MMD\thinspace (cardinality\thinspace  vs.\thinspace curvature)}\label{fig:mmd}
	\end{subfigure}
	\caption{Curvatures carry more information than other local features.
		We show a 2-dimensional embedding of graphs from the \stex collection based on kPCA, 
		using an RBF kernel with curvature distributions computed using $\smoothing=0.1$, $\mu^\werw$, and $\aggregation_\mean$ (\ref{fig:kpca-curvy}) or edge neighborhood size distributions (\ref{fig:kpca-noncurvy}) as input features. 
		We see that only curvatures yield a meaningful and discriminative grouping.
		Corroborating this finding, we also depict Bonferroni-adjusted p-values of testing for significant differences in feature distributions---i.e., p-values multiplied by the number $h$ of hypothesis tests, as \citet{bonferroni1936teoria} correction requires $p \leq \nicefrac{\alpha}{h}$ for some desired Type I-error rate $\alpha$---using MMD on distributions of edge curvatures computed with the same parameters as for (\ref{fig:kpca-curvy}) (upper triangle) or edge cardinality (lower triangle), 
		for the subset of the \dblpv collection corresponding to top conferences grouped by areas of research~(\ref{fig:mmd}).}\label{fig:kpca}
\end{figure}

\begin{table}[t]
	\centering
	\setlength{\tabcolsep}{4.6pt}
	\caption{\ourmethod curvatures lead to better clusterings than other local features.
		We show $\operatorname{WCC}_{\curvature(i,j)}$ for collection clusterings computed using RBF or exp.~Wasserstein kernels with edge curvatures, edge neighborhood sizes, edge-averaged node curvatures, or node neighborhood sizes as inputs.}\label{tab:wcc}
\begin{tabular}{rrrrrrrrr}
	\toprule
	{} &
	$\rbf_{\curvature(e)}$ & $\wasserstein_{\curvature(e)}$ &
	$\rbf_{|\neighborhood(e)|}$ & $\wasserstein_{|\neighborhood(e)|}$ & $\rbf_{\curvature^{E}(i)}$ & $\wasserstein_{\curvature^{E}(i)}$ & $\rbf_{|\neighborhood(i)|}$ & $\wasserstein_{|\neighborhood(i)|}$ \\
	\midrule
\dblpv &                   0.2151 &                           0.1908 &                           0.3309 &                                   0.2358 &                            0.2273 &                         \bfseries          0.0445 &                        0.0910 &                                      0.1285 \\
\mus  &                   0.1955 &                           0.1758 &                           0.2609 &                                   0.2723 &                         0.2062 &   \bfseries                               0.1606 &                           0.2774 &                                   0.2458 \\
\stex  &                   0.2651 &                           0.2877 &                           0.3018 &                                    0.2950 &              \bfseries           0.2393 &                                 0.2577 &                           0.3067 &                                   0.2689 \\
\sha  &                   0.5984 &                            0.6390 &                              0.6716 &                                      0.6597 &           \bfseries              0.5021 &                                 0.6526 &                           0.6236 &                                   0.6641 \\
\bottomrule
\end{tabular}
\end{table}

\paragraph{Q3 Hypergraph Learning.}
To explore~\thinspace the utility~\thinspace of curvatures for~\thinspace  
learning~\thinspace  on~\thinspace \emph{individual hyper-graphs},
we perform spectral clustering using either curvatures or other local node features.
To evaluate the resulting node clusterings,
we leverage that \emph{nodes} in the \apsvcout collection correspond to APS papers, 
for which we consistently know the titles. 
Hence, even in the absence of a meaningful ground truth, 
we can still check the sensibility of a clustering by statistically analyzing the titles grouped together using tools from natural language processing.
We find that node clusterings based on curvatures correspond to thematically more coherent groupings (\cref{apx-results}).
For learning on \emph{hypergraph collections}, 
we spectrally cluster the collection using RBF or exponential Wasserstein kernel matrices, $\exp (-\gamma \wasserstein(\mu_x, \mu_y))$, on node and edge curvatures or other local features~\citep{plaen2020wek}.
Lacking ground-truth labels,
we evaluate the clustering quality in an \emph{unsupervised} manner,
using what we call the \emph{Wasserstein Clustering Coefficient} (WCC).
This measure compares averaged \emph{intra}-cluster Wasserstein distances
to averaged \emph{inter}-cluster Wasserstein distances, 
such that a \emph{lower} WCC corresponds to a higher-quality clustering. 
Given $c$ clusters $\mathcal{X} = \{X_1,\dots,X_c\}$ of hypergraphs $\hypergraph$ represented by their feature distributions $\vec{\chi}_{\hypergraph}$, 
we define
\begin{align*}
	\operatorname{WCC}(\mathcal{X}) \defeq \frac{\sum_{X \in \mathcal{X}} \omega(X)}{ 1 + \sum_{X\neq Y \in \mathcal{X}} \omega(X, Y)}\;,~\text{with}\begin{cases}
		\omega(X)\defeq\binom{|X|}{2}^{-1} \sum_{x \neq y \in X} \wasserstein(\vec{\chi}_x, \vec{\chi}_y)\;,\\
		\omega(X, Y)\defeq \left(|X||Y|\right)^{-1} \sum_{x, y \in X\times Y} \wasserstein(\vec{\chi}_x, \vec{\chi}_y)\;.
	\end{cases}
\end{align*} 
As illustrated in \cref{tab:wcc}, when evaluated using WCC with directional curvature distributions as $\vec{\chi}$, 
i.e., $\operatorname{WCC}_{\curvature(i,j)}$,
\ourmethod curvatures consistently yield better clusterings than other local features.

\section{Discussion and Conclusion}\label{sec:Discussion and Conclusion}

We introduced \ourmethod, the  first unified framework for \ollivier on
hypergraphs that integrates and generalizes existing approaches to hypergraph
\orc. \ourmethod disentangles the common building blocks of 
all notions of hypergraph \orc, yielding curvature notions that are provably
aligned with our geometric intuition. We performed a rigorous theoretical and
empirical analysis of \ourmethod curvatures, demonstrating their practical utility and
scalability through extensive experiments, covering both
\emph{hypergraph exploration} and \emph{hypergraph learning}. While our
work paves the way toward future work seeking to leverage the power of
\ollivier for hypergraphs in hypergraph learning algorithms, it still has some
limitations to be addressed.
First, \orc on graphs is defined for \emph{any} probability
measure, but we only consider measures corresponding to a single step of
a random walk.
Future work could thus harness higher-order random walks or alternative
probability measures, and consider analyzing relationships between such
probability measures and other structural hypergraph properties.
Second, hyperedge intersections can vary in cardinality, but this variation is
not currently reflected in our probability measures. 
One could thus integrate \ourmethod with the $s$-walk framework proposed by
\citet{aksoy2020hypernetwork}, or define persistent \ourmethod curvatures
based on hypergraph filtrations, extending work on persistent \orc for graphs
\citep{wee2021ollivier}. 
Third, like the original \orc, \ourmethod curvatures are static, but many
hypergraphs are inherently dynamic, suggesting a need to develop dynamic curvature
notions.
Fourth, despite its comprehensive scope, our study only
scratches the surface regarding the theoretical and empirical analysis of
\ourmethod curvatures, and we believe that there are many more connections
between \ourmethod curvatures and other hypergraph descriptors to be uncovered,
and many additional use cases to be explored. For instance, \ourmethod generalizes
\orc, but not Forman--Ricci curvature~(\frc), and we believe that a
framework for \frc could help uncover new relations between
combinatorial curvature notions and hypergraph structure.
Finally, we imagine that incorporating hypergraph curvature into models as an additional inductive bias could prove useful in hypergraph learning more broadly.
 
  \section*{Ethics Statement}

  Our main contribution is \ourmethod, 
  a unified mathematical framework yielding theoretically sound hypergraph descriptors that are also practically useful for hypergraph exploration and hypergraph learning. 
  As such, \ourmethod comes with the caveats applicable to hypergraph exploration and hypergraph learning methods more generally.
  Most importantly, it should be used with caution on data related to people, and its results should not be decontextualized. 
  We adhered to these principles in our experiments, and selected our datasets accordingly.

  \section*{Reproducibility Statement}

  To facilitate reproducibility, 
  we provide more details on our data, implementation, and results in more detail in \cref{apx-datasets,apx-implementation,apx-results}, 
  and make all our code, data, and results publicly available at \oururl.

\nocite{Gibbs02a,bakry1985diffusions,liu2019distance, munch2020spectrally,kempton2020large,forman2003bochner,topping2022oversquashing,roy2020forman,ollivier2007ricci, ollivier2009ricci,jost2014ollivier,weber2017networks,bauer2017curvature,samal2018comparative,bourne2018ollivier, lin2011ricci,wee2021ollivier, wee2021forman,amburg2020clustering,veldt2020parameterized,zhou2006learning,wachman2007learning,huang2021unignn,gao2019wasserstein,bai2014hypergraph, bloch2013mathematical, martino2020hyper,coupette2022world,cucuringu19sponge,leal2019curvature,leal2020ricci,leal2021forman,saucan2018forman,eidi2020ollivier,asoodeh2018curvature,banerjee2021spectrum,akamatsu2022new,ikeda2021ricci,lin2011ricci,yadav2022poset,murgas2022beyond}
  \bibliography{bibliography,curvature}
  \bibliographystyle{iclr2023_conference}

  \clearpage

  \appendix
  \section{Appendix}

In this Appendix, we include the following materials.

\textbf{\ref{apx-proofs}~~Deferred Proofs.} \newline
All proofs for \cref{theory}, along with supporting definitions, lemmas and corollaries.

\textbf{\ref{apx-related}~~Related Work.}\newline
Discussion of related work treating hypergraph curvatures, graph curvatures, or hypergraph analysis.

\textbf{\ref{apx-datasets}~~Dataset Details.} \newline 
Further information on the provenance, semantics, and statistics of our datasets.

\textbf{\ref{apx-implementation}~~Implementation Details.}\newline
Details on our implementation, including proofs showing the correctness of performance shortcuts.

\textbf{\ref{apx-results}~~Further Results.}\newline
Display and discussion of results not included in the main paper.

\subsection{Deferred Proofs}
\label{apx-proofs}

\measureequality*
\begin{proof}
	For notational simplicity, \Wlog, we assume that $\smoothing = 0$.
	In an $\uniformity$-uniform, $\regularity$-regular, $\cooccurrence$-cooccurrent hypergraph $\hypergraph=(\vertices, \edges)$, each node $i$ has degree $\regularity$ and  $\frac{(\uniformity-1)\regularity}{\cooccurrence}$ neighbors, and each edge has cardinality $\uniformity$. Hence, for nodes $i,j\in\vertices$ with $i\adjacent j$, 
	\begin{align*}
		\mu_i^\enrw(j) = 
		\frac{1}{|\neighborhood(i)|} 
		= \frac{\cooccurrence}{(\uniformity-1)\regularity} 
		= \frac{1}{\regularity}\cdot\cooccurrence\cdot\frac{1}{\uniformity - 1}
		=  \frac{1}{\degree(i)}\underset{e\ni i,j}{\sum}\frac{1}{|e|-1} = \mu_i^\eerw(j)\;\phantom{,}\\
		= \frac{\cooccurrence}{\regularity(\uniformity - 1)}
		=  \frac{|\{e\in E\mid \{i,j\}\subseteq e\}|}{\underset{f\ni i}{\sum}\big(|f|-1\big)}
		= \mu_i^\werw(j)\;.
	\end{align*}
	Graphs are $2$-uniform and $1$-cooccurrent (but not generally regular), and hence, $|\neighborhood(i)| = \degree(i)$. 
	Using this to simplify the probability measure expressions, the claim follows.
\end{proof}

\aggregationequality*
\begin{proof}
	Given probability distributions $\mu_1, \mu_2, \dots, \mu_n$, their
	Wasserstein barycenter is defined as the distribution~$\bar{\mu}$ that minimizes $f(\bar{\mu}) \defeq \frac{1}{n} \sum_{i=1}^{n} \wasserstein_1\mleft(\bar{\mu}, \mu_i\mright)$.
	Since $|e| = 2$, we minimize $\wasserstein_1\mleft(\bar{\mu}, \mu_1\mright) + \wasserstein_1\mleft(\bar{\mu}, \mu_2\mright)$.
  The Wasserstein distance is a metric, so it satisfies the triangle inequality.
  Thus, $\wasserstein_1\mleft(\mu_1, \mu_2\mright) \leq \wasserstein_1\mleft(\bar{\mu}, \mu_1\mright) + \wasserstein_1\mleft(\bar{\mu}, \mu_2\mright)$ for all choices of $\bar{\mu}$.
  Hence, $f$ is minimized by either $\mu_1$ or $\mu_2$.
  Evaluating both cases yields $\aggregation_\mean(e) = \aggregation_\bary(e)$, 
  and observing that $\aggregation_\maxi(e) = \wasserstein_1(\mu_i,\mu_j)$ for $e=\{i,j\}$ by definition, 
  the claim follows.
\end{proof}

\curvatureboundmean*
\begin{proof}
	We bound each of the summands in the curvature calculation. Given
	probability measures $\mu_i, \mu_j$, a result by \citet[Theorem~4]{Gibbs02a}
	states that
\begin{equation}
		\dist_{\min}(\hypergraph) \dist_{\text{TV}}\mleft(\mu_i, \mu_j\mright) \leq \wasserstein_1\mleft(\mu_i, \mu_j\mright) \leq \diam(\hypergraph) \dist_{\text{TV}}\mleft(\mu_i, \mu_j\mright)\;,
    \label{eq:Wasserstein bounds}
	\end{equation}
where $\dist_{\text{TV}}$ refers to the \emph{total variation
  distance}. The intuition behind this bound is that the total variation
  distance represents a specific type of transport plan between the two
  probability measures; the factors arising from the minimum~(maximum)
  distance in a space indicate the minimum~(maximum) distance that realizes
  this transport plan.
Since all our measures are defined over a finite space, we have
	$\dist_{\text{TV}}\mleft(\mu_i, \mu_j\mright) = \nicefrac{1}{2} \mleft\|\mu_i - \mu_j \mright\|_1$.
The claim follows by considering that pairwise distances are being
  \emph{subtracted} to calculate our curvature measure. 
\end{proof}

\curvatureboundmax*
\begin{proof}
  For $\aggregation_\maxi$, \cref{eq:Wasserstein bounds} applies for a single
  pairwise distance only. We thus only obtain a single bound based on the maximum
  total variation distance between two probability measures.
\end{proof}

\bonnet*
\begin{proof}
  Let $\{i, j\} = \argmax\mleft(\dist\mleft(s\mright)\mright)$ as required in
  the theorem. We then have following chain of (in)equalities:
\begin{equation}
    \dist(s) = \dist(i, j) = \wasserstein_1(\delta_i, \delta_j) \leq \wasserstein_1(\delta_i, \mu_i) + \wasserstein_1(\mu_i, \mu_j) + \wasserstein_1(\mu_j, \delta_j)\;.
    \label{eq:Bonnet inequality}
  \end{equation}
Rearranging \cref{eq:orchard:curvature:general}, 
  we have $\mleft(1 - \curvature(s)\mright)\dist(s) = \aggregation(s)$.
According to our assumptions,
  $\wasserstein_1(\mu_i, \mu_j) \leq \aggregation(s) = \mleft(1 - \curvature(s)\mright)\dist(i,j)$.
  Inserting this into \cref{eq:Bonnet inequality} yields
\begin{align}
    && \dist(i, j) & \leq \jump(\mu_i) + \jump(\mu_j) + \mleft(1 - \curvature(s)\mright)\dist(i,j)\\
    \Leftrightarrow && \dist(i, j) - (1 - \curvature(s)) \dist(i, j) & \leq \jump(\mu_i) + \jump(\mu_j)\\
    \Leftrightarrow && \dist(i, j) \leq \frac{\jump(i) + \jump(j)}{\curvature(s)}\;,
  \end{align}
where the last step is only valid since $\curvature(s) \geq \curvature > 0$ by assumption.
\end{proof}

\begin{restatable}[Hypercliques, hypergrids, hypertrees]{defi}{hhh}
	\label{def:hyperthings}
	A simple, connected hypergraph $\hypergraph = (\vertices,\edges)$~is
	\begin{itemize}[label=--]
		\item a \emph{hyperclique} if $\edges = \binom{\vertices}{\uniformity}$ for some  $\uniformity\leq |\vertices|$,
		\item a \emph{hypergrid} if $\hypergraph$ is an $\uniformity$-uniform hypergraph for which there exists a lattice 
		$L=(\vertices,\edges_L)$ \suchthat $\edges = \{e\in\binom{\vertices}{\uniformity}\mid e\text{~corresponds to a path of length~}\uniformity\text{~in~}L\}$, and
		\item a \emph{hypertree}
		if there exists a tree $T=(\vertices,\edges_T)$ \suchthat each edge $e\in\edges_T$ induces a subtree in~$T$.
	\end{itemize}
\end{restatable}
\begin{restatable}{cor}{hhhdeficor}
	Cliques are hypercliques, grids are hypergrids, and trees are hypertrees.
\end{restatable}
\begin{restatable}{cor}{hhhisocor}
\label{cor:isomorphism}
	If $\hypergraph = (\vertices, \edges)$ is a  hyperclique, 
	a hypergrid, 
	or an $\uniformity$-uniform, 
	$\regularity$-regular, $1$-intersecting hypertree,
	for $i,j\in\vertices$, 
	the sets $S_i = \{e\in\edges\mid i \in e\}$ and $S_j = \{e\in\edges\mid j \in e\}$ are isomorphic, i.e., there exists $\varphi: \neighborhood(i)\cup\{i\}\rightarrow\neighborhood(j)\cup\{j\}$ such that $\{\{\varphi(x)\mid x \in e\}\mid e \in S_i\} = S_j$.
\end{restatable}

For hypercliques, hypergrids, and hypertrees with certain regularities, $\aggregation_\mean(e)$ and $\aggregation_\maxi(e)$ are constants. 

\begin{restatable}[Hypercliques, hypergrids, hypertrees]{lem}{agglem}
\label{lem:aggregations}
	If $\hypergraph = (\vertices, \edges)$ is a  hyperclique, a hypergrid, or an $\uniformity$-uniform, 
	$\regularity$-regular, $1$-intersecting hypertree, 
	we have $\aggregation_\mean(e) = \aggregation_\maxi(e) = \wasserstein_1(\mu_i,\mu_j) = w$ 
	for $w\in\reals$,
	$e\in\edges$, 
	and $i,j\in\vertices$ with $i\adjacent j$.
\end{restatable}
\begin{proof}
	By Corollary~\ref{cor:isomorphism},
	we have $w \defeq \wasserstein_1(\mu_i,\mu_j)  = \wasserstein_1(\mu_p,\mu_q)$ for $i,j,p,q\in \vertices$ with $i\adjacent j$ and $p\adjacent q$.
	Hence  $\aggregation_\maxi(e) = w$, 
	and $\aggregation_\mean(e) = \frac{2}{|e|(|e|-1)} \sum_{\{i, j\}\subseteq e}\wasserstein_1(\mu_i,\mu_j) = \frac{2}{|e|(|e|-1)} \frac{|e|(|e|-1)}{2}w = w$, 
	for $e\in\edges$.
\end{proof}
\begin{restatable}{cor}{meanmaxi}
	If $\hypergraph = (\vertices, \edges)$ is a  hyperclique, a hypergrid, or an $\uniformity$-uniform, 
	$\regularity$-regular, $1$-intersecting hypertree, $\aggregation_\mean(e) = \aggregation_\maxi(e)$.
\end{restatable}

Using \cref{lem:aggregations}, we now prove that under $\aggregation_\mean$ and $\aggregation_\maxi$, hypercliques are positively curved, 
hypergrids are flat, 
and hypertrees are negatively curved, as desired.

\hyperclique*
\begin{proof}
	A hyperclique is $\uniformity$-uniform, $(n-1)$-regular, and $(\uniformity-2)$-cooccurrent, 
	so $\mu_i^\enrw = \mu_i^\eerw = \mu_i^\werw$ for each node $i\in\vertices$ by Lemma~\ref{lem:dispersions-general}.
	Thus, considering $\mu_i^\enrw$,
	each node $i\in\vertices$ has $n-1$ neighbors to which it distributes its probability mass equally, 
	and we have $\wasserstein_1(\mu_i,\mu_j) = \frac{1}{n-1}$ for $i,j\in \vertices$ with $i\adjacent j$.
	The claim now follows from Lemma~\ref{lem:aggregations}.
\end{proof}

\hypergrid*

\begin{proof}
	By Corollary~\ref{cor:isomorphism}, 
	the sets $S_i = \{e\in\edges\mid i \in e\}$ and $S_j = \{e\in\edges\mid j \in e\}$ are isomorphic, 
	and due to the symmetries in the hypergrid, 
	the isomorphism $\varphi\colon \neighborhood(i)\cup\{i\} \to \neighborhood(j)\cup\{j\}$ minimizing the cost 
	$\sum_{x\in \neighborhood(i)\cup\{i\}}\dist\left(x,\varphi(x)\right)$
	corresponds to the coupling minimizing $\wasserstein_1(\measure_i,\measure_j)$. 
	The cost of $\varphi$ equals the minimum cost of an isomorphism in $\hypergraph$'s underlying lattice $L$ 
	between the inclusive $(\uniformity-1)$-hop neighborhoods of two nodes adjacent in $L$,
	which is $|\neighborhood(i)\cup\{i\}|$.
	Hence, 
	$\wasserstein_1(\mu_i,\mu_j) = \frac{|\neighborhood(i)\cup\{i\}|}{|\neighborhood(i)\cup\{i\}|} = 1$ 
	for $i,j\in\vertices$ with $i\adjacent j$ and all choices of $\measure$,
	and the claim then follows from Lemma~\ref{lem:aggregations}.
\end{proof}

\hypertree*
\begin{proof}
	An $\uniformity$-uniform, 
	$\regularity$-regular, $1$-intersecting hypertree is $1$-cooccurrent, 
	so we have $\mu_i^\enrw = \mu_i^\eerw = \mu_i^\werw$ for each node $i\in\vertices$ by Lemma~\ref{lem:dispersions-general}.
	Each node $i\in\vertices$ has $(\uniformity-1)\regularity$ neighbors, 
	such that $\mu_i^\enrw$ distributes a fraction $\frac{1}{(\uniformity-1)\regularity}$ of the probability mass to each of $i$'s neighbors.
	Nodes $i,j\in\vertices$ with $i\adjacent j$ share $(\uniformity-2)$ neighbors 
	(those in the unique edge $e$ satisfying $\{i,j\}\subseteq e$), 
	and the probability mass allocated by $\mu_i$ to $j$ can be matched with the probability mass allocated by $\mu_j$ to $i$ at cost $1$.
	Because $\hypergraph$ is a hypertree, 
	the remaining probability mass,
	$(\uniformity-1)(\regularity-1)/\big((\uniformity-1)\regularity\big) = (\regularity-1)/\regularity$,
	needs to be transported from the neighborhood of $i$ to the neighborhood of $j$ at cost $3$.
	Hence, 
	\begin{align*}
		\wasserstein_1(\mu_i,\mu_j) = 1 \cdot \frac{1}{(\uniformity-1)\regularity} + 3 \cdot \frac{\regularity-1}{\regularity}
	\end{align*}
	for $i,j\in\vertices$ with $i\adjacent j$.
	Again, the claim follows from Lemma~\ref{lem:aggregations}.
\end{proof}
 
\clearpage

\subsection{Related Work}
\label{apx-related}

\paragraph{Hypergraph Curvature}
Most closely related to our work is the literature on hypergraph curvatures.
Much of this literature focuses on defining notions of \orc and Forman-Ricci Curvature (\frc) specifically for \emph{directed} hypergraphs 
and studying some of their mathematical and empirical properties \citep[e.g.,][]{leal2019curvature,leal2020ricci,leal2021forman,saucan2018forman}.
Notably, the directed hypergraph \orc introduced by \cite{eidi2020ollivier} is an instantiation of our framework with $\mu^\eerw$ and $\aggregation_\mean$. 
Curvature notions for \emph{undirected} hypergraphs are comparatively less explored, 
and especially the literature generalizing \orc is almost entirely theoretical.
The generalization of \orc proposed by \cite{asoodeh2018curvature} 
and the equivalent measure used by \cite{banerjee2021spectrum} 
are instantiations of our framework using $\mu^\eerw$ and $\aggregation_\bary$.
\cite{akamatsu2022new} propose $(\alpha, h)$-\orc using cost functions based on structured optimal transport, 
and \cite{ikeda2021ricci} define $\lambda$-coarse Ricci curvature using a $\lambda$-nonlinear Kantorovich difference based on a submodular hypergraph Laplacian as a generalization of \orc as introduced by \cite{lin2011ricci}. 
Both of these works define curvature exclusively for pairs of nodes, rather than for hyperedges.
Beyond \orc, \cite{yadav2022poset} study \frc for undirected hypergraphs defined via poset representations, and \cite{murgas2022beyond} explore hypergraphs constructed from protein-protein interactions using a different notion of \frc based on the Hodge Laplacian.
To the best of our knowledge, with \ourmethod, we are the first to introduce a flexible framework generalizing \orc to hypergraphs, 
and to demonstrate the utility of hypergraph \orc in practice.

\paragraph{Graph Curvature.}

Beyond the Ollivier-Ricci concepts, there are also curvature concepts based on the contractivity of operators~\citep{bakry1985diffusions}, which could be considered a ``spiritual precursor'' to Ollivier's work.
This perspective has been used to provide a predominantly \emph{spectral perspective} on curvature~\citep{liu2019distance, munch2020spectrally},
whereas \orc can foremost be seen as a \emph{probabilistic concept}.
Recently, \citet{kempton2020large} defined a hybrid between Ollivier and
Bakry-{\'E}mery curvature on graphs.
A more combinatorial perspective is assumed by \frc, which is motivated by defining equivalent formulations of curvature on structured spaces, such as CW complexes or simplicial complexes.
Originally described by \citet{forman2003bochner},
\frc has since been improved in the context of explaining the learning behavior of graph neural networks~\citep{topping2022oversquashing},
with other recent work focusing on fusing it with topological graph properties~\citep{roy2020forman}.
\orc was first developed for general Markov chains~\citep{ollivier2007ricci, ollivier2009ricci}, but has quickly been adopted to characterize graphs~\citep{jost2014ollivier} and networks~\citep{weber2017networks}.
With numerous follow-up publications elucidating the relationship between structural properties of a graph and \orc~\citep{bauer2017curvature, samal2018comparative},
the initial concept has also been substantially updated~\citep{bourne2018ollivier, lin2011ricci}.
As an emerging research direction, we identified the combination of \orc~(and \frc) with concepts from computational topology, leading to an inherent multi-scale perspective on data.
This has led to promising results for treating biomedical graph data~\citep{wee2021ollivier, wee2021forman}.

\paragraph{Hypergraph Learning.}

Work tackling certain hypergraph learning tasks 
such as hypergraph clustering \citep{amburg2020clustering,veldt2020parameterized} has existed for many years \citep{zhou2006learning,wachman2007learning}.
Some approaches make use of intrinsic structural properties of hypergraphs,
leading to hypergraph neural network architectures~\citep{huang2021unignn} and message passing formulations~\citep{gao2019wasserstein},
whereas others focus on developing similarity measures, i.e., \emph{kernels}~\citep{bai2014hypergraph, bloch2013mathematical, martino2020hyper}.
Methods from the rich literature on \emph{graph} kernels can also be employed to address hypergraph learning tasks,
namely, by transforming the hypergraph into a graph, 
but most popular transformations are lossy and may drastically increase the size of the object under study, 
such that the practicality and utility of this approach is unclear.

\paragraph{Hypergraph Mining and Analysis.}

In recent years, there has been a renewed interest in hypergraph mining and analysis. 
Notably, there is work developing new hypergraph descriptors  \citep{aksoy2020hypernetwork}, 
extending motif discovery to hypergraphs \citep{lee2020hypergraph,lee2021thyme}, 
solving classic graph mining tasks in the hypergraph setting \citep{macgregor2021finding},
or identifying patterns in real-world hypergraphs \citep{do2020structural}.
However, to the best of our knowledge, none of this work draws on curvature concepts to solve the mining and analysis tasks of interest.
 
\clearpage

\subsection{Dataset Details}
\label{apx-datasets}

At a high level, our workflow to produce and work with the datasets used in our experiments (\cref{experiments}) was as follows:
\begin{enumerate}[leftmargin=\widthof{(iii)}+\labelsep]
	\item Obtain raw data in a variety of different formats, e.g., CSV, JSON, or XML.
	\item Transform the raw data into a hypergraph CSV that retains as much of the raw data semantics as possible. 
	This CSV is guaranteed to contain one row per edge, 
	one column with unique edge identifiers, 
	and one column with the nodes contained in each edge. 
	It may also contain additional columns holding further metadata associated with individual edges.
	Column names may differ between datasets to reflect dataset semantics. 
	\item Provide a unified loading interface to the datasets in Python.
	\item Transform semantics-laden hypergraph CSV files into semantics-free one-based integer edge lists and sparse matrices for curvature computations in Julia, 
	compute curvatures in Julia, and store the results in JSON files.
	\item Map results back to original dataset semantics in Python for further examination.
\end{enumerate}

In the following, we give more details on the provenance, semantics, and statistics of our datasets. 
Unless if otherwise noted, 
we make our datasets publicly available with our online materials, along with the raw data and all preprocessing code.\!\footnote{\oururl}

\subsubsection{\apsa, \apsva, \apsvcout: American Physical Society Journal Articles}

The American Physical Society (APS), 
a nonprofit organization working to advance the knowledge of physics, 
publishes several peer-reviewed research journals.
The APS makes two datasets based on its publications available to researchers:
\begin{inparaenum}[(i)]
	\item an edge list containing (citing, cited) pairs of articles contained in its collection, and
	\item a JSON dataset containing the metadata for each article in its collection.
\end{inparaenum}
These datasets are updated on a yearly basis, 
and researchers can request access by filling out a web form located at \url{https://journals.aps.org/datasets}. 
We made a data access request and were granted access to the 2021 versions of the APS datasets within two weeks. 

From the APS datasets, 
we derived the following hypergraphs and hypergraph collections:
\begin{enumerate}[label=(\roman*),leftmargin=\widthof{(iii)}+\labelsep]
	\item \apsa: Each node corresponds to an author who published at least one article in an APS journal.
	Each edge $e$ corresponds to an article in an APS journal, 
	and it contains as nodes all authors of~$e$. 
	This hypergraph is derived from the JSON data.
	\item \apsva: \apsa, split up by journal, for a total of 19 hypergraphs.
	For each journal $j$, the edge set of \apsa is restricted to articles from $j$, and the node set of \apsa is restricted to nodes authoring at least one article from $j$.
	\item \apsvcout: 
	We derive one hypergraph for each of the 19 journals represented in the edge list data.
	For each journal $j$, 
	the edge set comprises articles from $j$ citing at least one article in $j$, 
	and the node set consists of articles in $j$ cited by at least one article in $j$.
\end{enumerate}

\paragraph{Access.}
Due to the terms and conditions associated with data access, 
we cannot make the APS datasets or the hypergraphs derived from them publicly available, 
and researchers seeking to work with this data will have to request data access from APS directly as outlined above.
However, we make our preprocessing code publicly available, 
such that researchers who have obtained access to the APS datasets can easily reproduce our hypergraphs from the raw data.

\paragraph{Caveats.} 
When doing our case studies on the \apsvcout dataset, we observed that some DOIs present in the edge list had no associated metadata in the JSON files provided by APS. 
This does not affect our curvature computations, 
but it might constrain the interpretability of results, 
e.g., when inspecting node clustering results based on article categories present only in the metadata.

\subsubsection{\dblp, \dblpv: dblp Journal Articles and Conference Proceedings}

The DBLP computer science library provides high-quality bibliographic information on computer science publications. 
All DBLP data is released under a CC0 license and freely available in one XML file that is updated regularly. 
We obtained the XML dump dated September 1, 2022 from \url{https://dblp.org/xml/release/} 
and preprocessed it into a CSV file containing only entries corresponding to the XML tags \texttt{article} and \texttt{inproceedings}, with one row per entry and the following columns:
\begin{itemize}[label=--,leftmargin=\widthof{(iii)}+\labelsep,itemsep=0pt]
	\item key: unique identifier of the entry, e.g., \texttt{conf/iclr/XuHLJ19} or\\ \texttt{journals/cacm/Savage16c}.
	\item tag: XML tag associated with the entry, one of \{\texttt{inproceedings}, \texttt{article}\}.
	\item crossref: cross-reference to a venue, e.g., \texttt{conf/iclr/2019}. 
	Sometimes missing although a venue should be present.
	\item author: semicolon-separated list of DBLP author names, e.g., \texttt{Keyulu Xu;Weihua Hu;Jure Leskovec;Stefanie Jegelka}. 
	Sometimes missing (we discard entries without authors when loading the data).
	\item year: entry publication year, e.g., \texttt{2019}.
	\item title: entry title, e.g., \texttt{How Powerful are Graph Neural Networks?}. 
	\item publtype: if present, the type of publication, e.g., \texttt{informal}. Mostly missing.
	\item journal: for article entries, the name of the publishing journal, e.g., \texttt{Commun.~ACM}.
	\item booktitle: for inproceedings entries, the name of the publishing venue, e.g., \texttt{ICLR}.
	\item volume: if present, the publication volume, e.g., \texttt{59}.
	\item number: if present, the publication number, e.g., \texttt{7}.
	\item pages: if present, the entry pages, e.g., \texttt{12--14}.
	\item mdate: modification date, e.g., \texttt{2019-07-25}.
\end{itemize}
This constitutes our individual hypergraph \dblp, 
in which each edge represents a paper, 
and each node represents an author.
From this hypergraph, we additionally derived the \dblpv hypergraph collection, 
which contains different subsets of \dblp by venue or group of venues. 
More precisely, we distinguish 1\,193 hypergraphs as follows:
\begin{enumerate}[label=(\roman*),leftmargin=\widthof{(iii)}+\labelsep]
	\item \texttt{dblp\_journal-all}, \texttt{dblp\_inproceedings-all}: partition of \dblp into entries published in journals and entries published as part of proceedings.
	\item \texttt{dblp\_journal-\{journal\}}: one hypergraph per journal, for all journals with at least 1\,000 articles in the DBLP dataset.
	\item \texttt{dblp\_proceedings-\{venue\}}: one hypergraph per venue (grouped by \texttt{booktitle}), for all venues with at least 1\,000 papers in the DBLP dataset.
	\item \texttt{dblp\_proceedings\_area-\{area\}\_\{venues\}}: one hypergraph per each of the FoR (field of research) areas 4601--4608, 4611--4613 as used in the CORE ranking (4609 and 4610 were not present in the ranking), 
	where each area is represented by all conferences (grouped by \texttt{booktitle}) with CORE rank A$^*$ and A that have at least 1\,000 papers in the DBLP dataset.
	These areas and associated top conferences are as follows:
	\begin{itemize}[label=--, leftmargin=\widthof{(i)}+\labelsep]
		\item 4601: Applied computing -- AIED, ICCS
		\item 4602: Artificial intelligence -- AAAI, AAMAS, ACL, AISTATS, CADE, CIKM, COLING, COLT, CP, CogSci, EACL, EC, ECAI, EMNLP, GECCO, ICAPS, IJCAI, IROS, KR, UAI
		\item 4603: Computer vision and multimedia computation -- AAAI, CVPR, ECAI, ICCV, ICME, IJCAI, IROS, WACV
		\item 4604: Cybersecurity and privacy -- AsiaCCS, CCS, CRYPTO, DSN
		\item 4605: Data management and data science -- CIKM, ECIR, EDBT, ICDAR, ICDE, ICDM, ISWC, KDD, MSR, PODS, RecSys, SDM, SIGIR, VLDB, WSDM, WWW
		\item 4606: Distributed computing and systems software  --  ASPLOS, CCGRID, CLUSTER, CONCUR, DISC, DSN, HPCA, HPDC, ICCAD, ICDCS, ICNP, ICPP, ICS, ICWS, INFOCOM, IPDPS, IPSN, PODC, SC, SIGCOMM, SPAA, WWW 
		\item 4607: Graphics, augmented reality and games -- ISMAR, SIGGRAPH, VR, VRST
		\item 4608: Human-centred computing -- ASSETS, CHI, CSCW, ITiCSE, IUI, SIGCSE, UIST
		\item 4611: Machine learning -- AAAI, AISTATS, COLT, ECAI, ICDM, ICLR, ICML, IJCAI, KDD, NeurIPS, PPSN, WSDM 
		\item 4612: Software engineering -- ASE, ASPLOS, CAV, ICSE, ICST, ISCA, ISSRE, MSR, OOPSLA, PLDI, POPL, RE, SIGMETRICS 
		\item 4613: Theory of computation -- EC, ESA, FOCS, ICALP, ICLP, ISAAC, ISSAC, KR, LICS, MFCS, SODA, STACS, STOC, WG
	\end{itemize}
\end{enumerate}

\paragraph{Caveats.}
For about 0.1\% of all records, our XML parser failed, which originally resulted in ``None'' as one of the authors of all problematic records. 
We then redid the preprocessing (and all subsequent computations) \emph{excluding} those records, but the records were still counted when determining the venues to include in \dblpv.

\subsubsection{\ndcai, \ndcpc: Drugs Approved by the U.S. Food \& Drug Administration}

The U.S. Food and Drug Administration (FDA) collects information on all drugs manufactured, prepared, propagated, compounded, or processed by registered drug establishments for commercial distribution in the United States.
The FDA maintains the National Drug Code (NDC) Directory, 
which is updated daily and contains the listed NDC numbers and all information submitted as part of a drug listing.
We downloaded the NDC data from \url{https://download.open.fda.gov/drug/ndc/drug-ndc-0001-of-0001.json.zip} on August 21, 2022, 
and transformed it into a CSV file, an example record of which is shown in \cref{tab:ndc-example}.
From this CSV file, we derived two hypergraphs. 
In both hypergraphs, edges correspond to FDA-registered drugs. 
In \ndcai, nodes correspond to the active ingredients used in these drugs, 
and in \ndcpc, nodes correspond to the pharmaceutical classes assigned to these drugs.
The edge cardinality distributions resulting from both semantics are shown in \cref{fig:ndc-histograms}.

\begin{table}[t]
	\centering\small
	\caption{Example record from the data underlying the \ndcai and \ndcpc hypergraphs.
	}\label{tab:ndc-example}
\begin{tabular}{lp{0.68\textwidth}}
	\toprule
	Column Name&Record Value\\
	\midrule
	product\_ndc                  &                                                                                                                                                                                                                                                                                     71930-020 \\
	active\_ingredients\_names     &                                                                                                                                                                                                                                                       [ACETAMINOPHEN, HYDROCODONE BITARTRATE] \\
	active\_ingredients\_strengths &                                                                                                                                                                                                                                                                          [325 mg/1, 7.5 mg/1] \\
	pharm\_class                  &                                                                                                                                                                                                                                                 [Opioid Agonist [EPC], Opioid Agonists [MoA]] \\
	marketing\_category           &                                                                                                                                                                                                                                                                                          ANDA \\
	dea\_schedule                 &                                                                                                                                                                                                                                                                                           CII \\
	finished                     &                                                                                                                                                                                                                                                                                          True \\
	packaging                    &  [\{'package\_ndc': '71930-020-12', 'description': '100 TABLET in 1 BOTTLE (71930-020-12)', 'marketing\_start\_date': '20180713', 'sample': False\}, \{'package\_ndc': '71930-020-52', 'description': '500 TABLET in 1 BOTTLE (71930-020-52)', 'marketing\_start\_date': '20180713', 'sample': False\}] \\
	dosage\_form                  &                                                                                                                                                                                                                                                                                        TABLET \\
	product\_type                 &                                                                                                                                                                                                                                                                       HUMAN PRESCRIPTION DRUG \\
	spl\_id                       &                                                                                                                                                                                                                                                          58b53a57-388e-40d0-9985-048e5af09b0d \\
	route                        &                                                                                                                                                                                                                                                                                        [ORAL] \\
	product\_id                   &                                                                                                                                                                                                                                                71930-020\_58b53a57-388e-40d0-9985-048e5af09b0d \\
	application\_number           &                                                                                                                                                                                                                                                                                    ANDA210211 \\
	labeler\_name                 &                                                                                                                                                                                                                                                                               Eywa Pharma Inc \\
	generic\_name                 &                                                                                                                                                                                                                                                      Hydrocodone Bitartrate and Acetaminophen \\
	brand\_name                   &                                                                                                                                                                                                                                                      Hydrocodone Bitartrate and Acetaminophen \\
	brand\_name\_base              &                                                                                                                                                                                                                                                      Hydrocodone Bitartrate and Acetaminophen \\
	brand\_name\_suffix            &                                                                                                                                                                                                                                                                                               \\
	listing\_expiration\_date      &                                                                                                                                                                                                                                                                                    2022-12-31 \\
	marketing\_start\_date         &                                                                                                                                                                                                                                                                                    2018-07-13 \\
	marketing\_end\_date           &                                                                                                                                                                                                                                                                                               \\
	openfda                      &                 \{'manufacturer\_name': ['Eywa Pharma Inc'], 'rxcui': ['856999', '857002', '857005'], 'spl\_set\_id': ['fcd2b59e-8087-475e-9e6b-911bd846ea96'], 'is\_original\_packager': [True], 'upc': ['0371930021121', '0371930020124', '0371930019128'], 'unii': ['NO70W886KK', '362O9ITL9D']\} \\
	\bottomrule
\end{tabular}
\end{table}

\begin{figure}[t]
	\centering
	\begin{subfigure}{0.45\linewidth}
		\includegraphics[width=\linewidth]{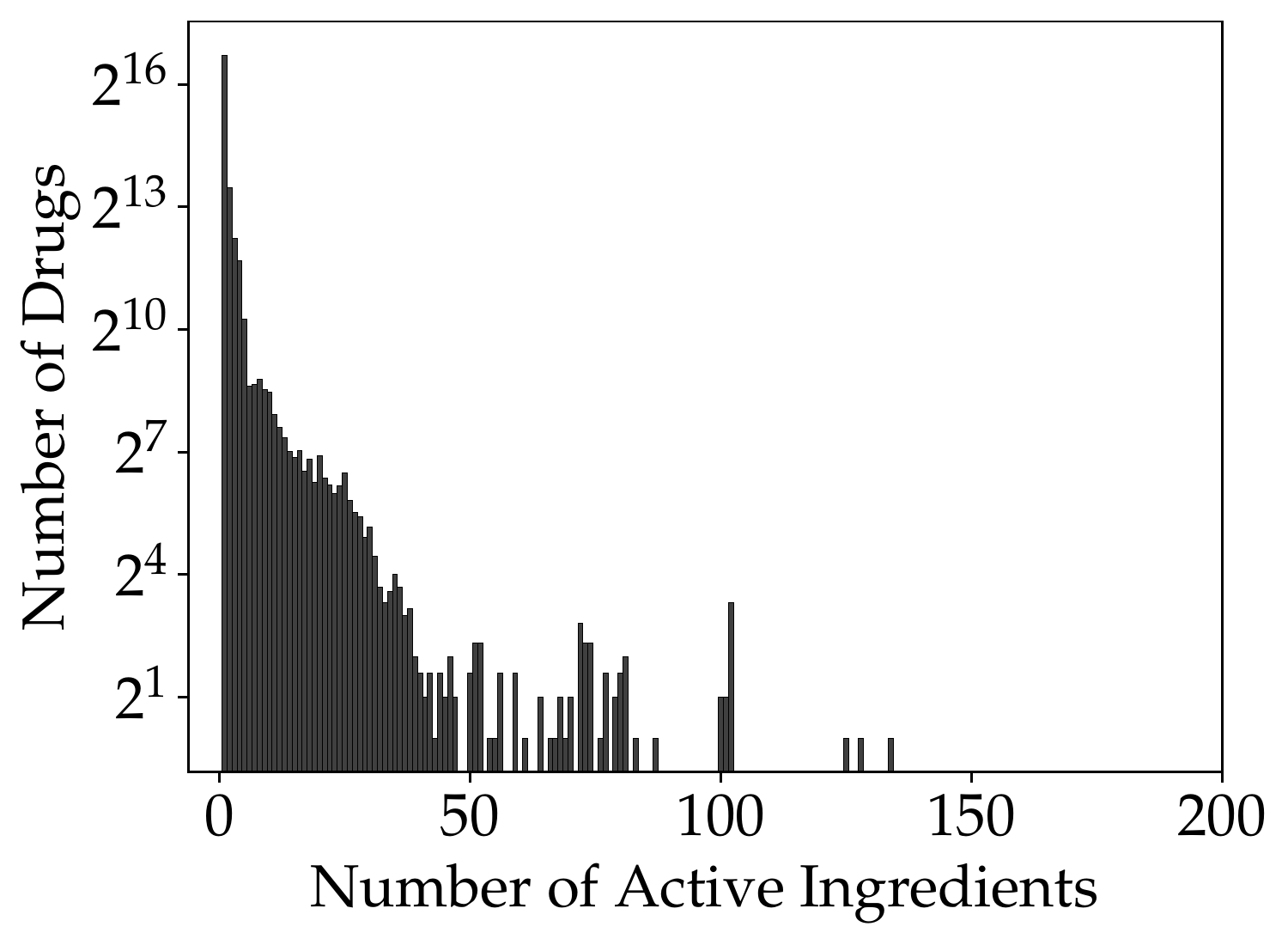}
		\subcaption{\ndcai}
	\end{subfigure}\quad
	\begin{subfigure}{0.45\linewidth}
		\includegraphics[width=\linewidth]{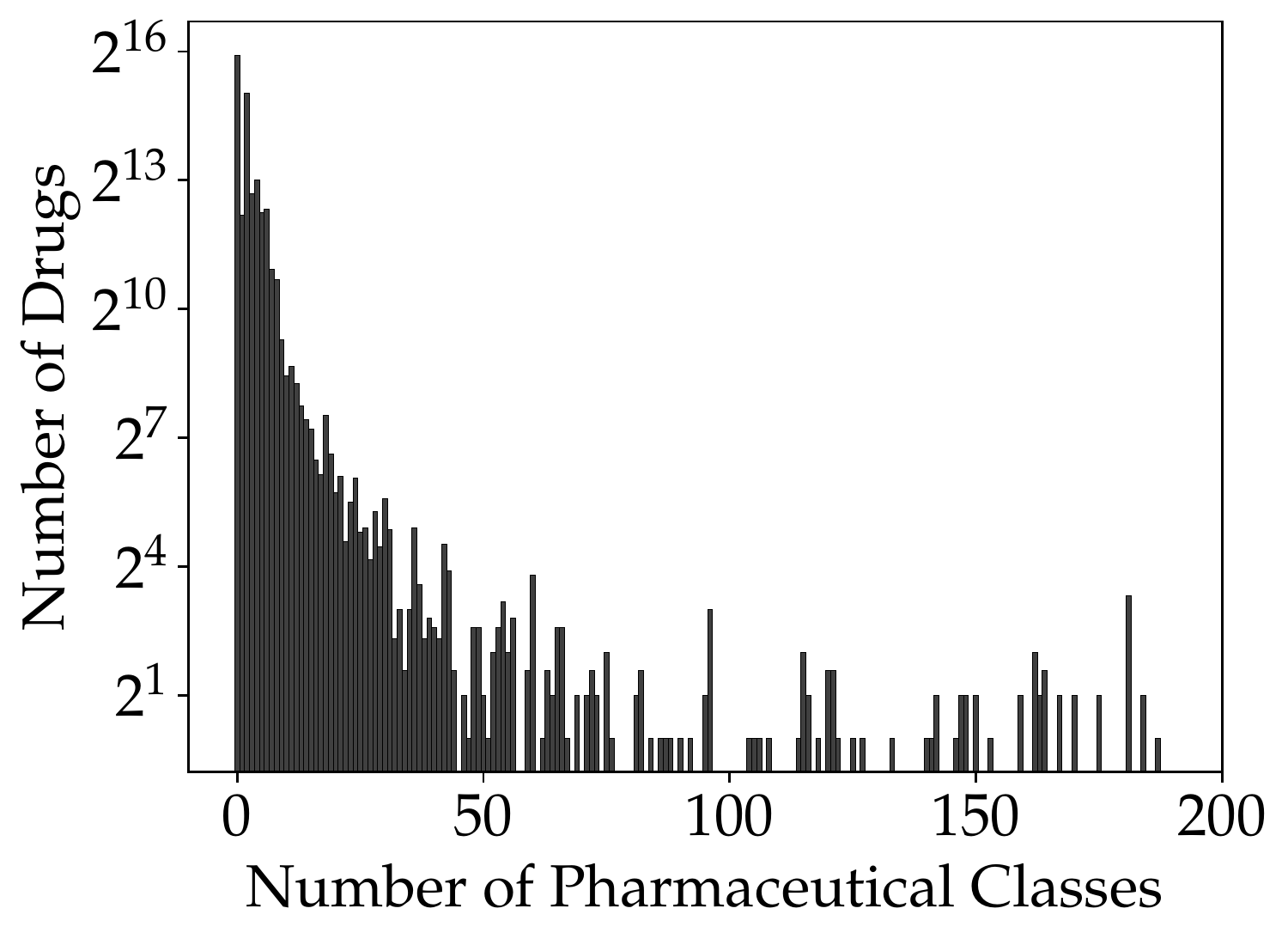}
		\subcaption{\ndcpc}
	\end{subfigure}
	\caption{Edge cardinality distributions for hypergraphs derived from NDC data.}\label{fig:ndc-histograms}
\end{figure}

\subsubsection{\mus: Music Pieces}

\texttt{music21} is an open-source Python library for computer-aided musicology that comes with a corpus of public-domain music in symbolic notation.
Using the \texttt{music21} library, 
we extracted a collection of hypergraphs from the \texttt{music21} corpus.
In this collection, each hypergraph corresponds to a music piece, each edge corresponds to a chord sounding for a specific duration at a particular offset from the start of the piece, 
and each node corresponds to a sound frequency. 
Note that hypergraphs in the \mus collection are node-aligned, which distinguishes them from the hypergraphs in all other collections. 
In \cref{tab:music}, we show the cardinality decomposition of selected music hypergraphs that include the largest edges. 
There, we include edges of cardinality $0$ for completeness (they correspond to pauses in the music), 
but they are discarded in our curvature computations. 

\paragraph{Caveats.}
When constructing our hypergraph collection from the \texttt{music21} corpus, 
we excluded pieces that are primarily monophonic. 
After exploring the corpus manually and evaluating the chord statistics of individual pieces, 
we decided to use only music with the following prefixes (corresponding to names of composers or collections): 
bach, beethoven, chopin, haydn, handel, monteverdi, mozart, palestrina, schumann, schubert, verdi, joplin, trecento, weber.
Some pieces are included in several editions 
(e.g., BWV 190.7, the chorale by Johann Sebastian Bach occupying the first two lines of \cref{tab:music}, 
which is included in both the original and an instrumental version).

\begin{table}[t]
	\centering
	\setlength{\tabcolsep}{2.75pt}
	\caption{Selection of hypergraphs from the \mus collection. 
		$n$ is the number of nodes, $m$ is the number of edges, and the columns labeled $i$ for $i\in\{0,1,\dots, 12\}$ record the number of edges of cardinality $i$ in the hypergraph.
		Identifiers correspond to abbreviated \texttt{music21} identifiers and generally have the shape \{composer\}-\{work identifier\}-\{suffix\}, where \emph{o} stands for \emph{opus}, \emph{m} stands for \emph{movement}, and \emph{inst} stands for \emph{instrumental}.
	}\label{tab:music}
	\begin{tabular}{lrrrrrrrrrrrrrrr}
\toprule
 & $n$ & $m$ & 0 & 1 & 2 & 3 & 4 & 5 & 6 & 7 & 8 & 9 & 10 & 11 & 12 \\
\midrule
bach-bwv190.7-inst & 38 & 233 & 1 & 0 & 0 & 4 & 25 & 60 & 56 & 72 & 9 & 6 & 0 & 0 & 0 \\
bach-bwv190.7 & 38 & 233 & 1 & 0 & 0 & 4 & 25 & 60 & 56 & 72 & 9 & 6 & 0 & 0 & 0 \\
bach-bwv248.23-2 & 35 & 155 & 1 & 0 & 0 & 12 & 45 & 90 & 0 & 3 & 1 & 2 & 1 & 0 & 0 \\
bach-bwv248.42-4 & 38 & 386 & 3 & 1 & 11 & 42 & 147 & 106 & 54 & 14 & 7 & 1 & 0 & 0 & 0 \\
beethoven-o133 & 88 & 5\,140 & 236 & 565 & 828 & 1\,515 & 1\,758 & 168 & 42 & 21 & 5 & 2 & 0 & 0 & 0 \\
beethoven-o18no1-m1 & 70 & 1\,979 & 28 & 295 & 165 & 472 & 761 & 244 & 7 & 6 & 0 & 0 & 1 & 0 & 0 \\
beethoven-o18no1-m4 & 77 & 2\,669 & 13 & 338 & 438 & 678 & 1\,032 & 134 & 33 & 1 & 1 & 1 & 0 & 0 & 0 \\
beethoven-o18no4 & 81 & 4\,730 & 95 & 465 & 674 & 977 & 1\,940 & 521 & 50 & 3 & 3 & 1 & 1 & 0 & 0 \\
beethoven-o59no1-m4 & 75 & 2\,338 & 27 & 80 & 231 & 338 & 1467 & 168 & 18 & 4 & 4 & 0 & 1 & 0 & 0 \\
beethoven-o59no2-m1 & 86 & 2\,338 & 60 & 127 & 398 & 427 & 1\,065 & 203 & 18 & 30 & 4 & 5 & 0 & 0 & 1 \\
beethoven-o59no3-m4 & 81 & 3\,292 & 19 & 381 & 529 & 734 & 1\,219 & 255 & 139 & 14 & 1 & 1 & 0 & 0 & 0 \\
beethoven-o74 & 82 & 6\,492 & 112 & 440 & 922 & 1\,448 & 2\,886 & 538 & 119 & 21 & 5 & 1 & 0 & 0 & 0 \\
monteverdi-madrigal.3.6 & 35 & 480 & 1 & 9 & 40 & 194 & 151 & 76 & 4 & 3 & 1 & 1 & 0 & 0 & 0 \\
schumann-clara-o17-m3 & 63 & 819 & 5 & 12 & 133 & 208 & 151 & 108 & 83 & 74 & 25 & 13 & 5 & 2 & 0 \\
schumann-o41no1-m5 & 72 & 2\,410 & 51 & 130 & 208 & 592 & 919 & 366 & 117 & 18 & 2 & 4 & 0 & 2 & 1 \\
\bottomrule
\end{tabular}
 \end{table}

\subsubsection{\stex: StackExchange Sites}

StackExchange is a platform hosting Q\&A communities also known as sites. 
Each question is assigned at least one and at most five tags.
In the second half of August 2022,
we used the StackExchange API to download all questions asked on all StackExchange sites listed on the StackExchange data explorer (\texttt{https://data.stackexchange.com/}), 
along with their associated tags and other metadata (including question titles and, for smaller sites, also question bodies).
From our downloads, we created the \stex hypergraph collection, 
in which each hypergraph corresponds to a StackExchange site, 
each edge corresponds to a question asked on a site, 
and each node corresponds to a tag used at least once on a site. 
\cref{tab:stex-1,tab:stex-2,tab:stex-3,tab:stex-4,tab:stex-5,tab:stex-6,tab:stex-7} list the basic statistics for each hypergraph from the \stex collection.

\paragraph{Caveats.}
While our curvature computations uniformly include only questions asked no later than August 15, midnight GMT, 
the metadata associated with these questions stems from snapshots at different times in the second half of August 2022.
We also excluded \texttt{stackoverflow.com} and  \texttt{math.stackexchange.com} from our downloads because they could not be downloaded within one day due to API quota limitations, 
and \texttt{ru.stackoverflow.com} because it was large but we would not have been able to interpret our results. 

\subsubsection{\sha: Shakespeare's Plays}

The \sha collection is a subset of the \textsc{Hyperbard} dataset recently introduced by \citet{coupette2022world}, 
based on the TEI-encoded XML files of Shakespeare's plays provided by Folger Digital Texts.
Here, each hypergraph represents one of Shakespeare's plays, which are categorized into three types: 
comedy, history, and tragedy.
In each hypergraph representing a play, 
nodes correspond to named characters in the play, 
and edges correspond to groups of characters simultaneously present on stage. 
These hypergraphs are documented extensively in the paper introducing the \textsc{Hyperbard} dataset \citep{coupette2022world}.

\subsubsection{\sync, \synr, \syns: Synthetic Hypergraphs}

To generate synthetic hypergraphs, 
we wrote hypergraph generators extending three well-known graph models to hypergraphs.
\begin{enumerate}[label=(\roman*),leftmargin=\widthof{(iii)}+\labelsep]
	\item For \sync, we extended the configuration model, which, for undirected graphs, is specified by a degree sequence. 
	Our hypergraph configuration model is specified by a node degree sequence and an edge cardinality sequence.
	\item For \synr, we extended the Erd\H{o}s-R\'enyi random graph model, which, for undirected graphs, is specified by a number of nodes $n$ and an edge existence probability $p$.
	Our Erd\H{o}s-R\'enyi random hypergraph model is specified by a number of nodes $n$, a number of edges $m$, and the probability $p$ of a one in  any cell of the node-to-edge incidence matrix.
	\item For \syns, we extended the stochastic block model which, for undirected graphs, is specified by a vector of $c$ community sizes and a $c\times c$ affinity matrix specifying affiliation probabilities between communities.
	Our hypergraph stochastic block model is specified by a vector of $c_V$ node community sizes, 
	a vector of $c_E$ edge community sizes, 
	and a $c_V\times c_E$ affinity matrix specifying affiliation probabilities between node communities and edge communities. 
\end{enumerate}
We used each of our generators to create 250 hypergraphs with identical node count $n$, edge count $m$, and density $\nicefrac{c}{nm}$, where $c$ is the number of filled cells in the node-to-edge incidence matrix.

\paragraph{Caveats.} 
Our generators work by pairing node and edge indices, 
and duplicated (node, edge) index pairs are discarded to generate simple hypergraphs, 
which can lead to small deviations from the input specification in practice.

\clearpage

\begin{table}[t]
	\centering
	\setlength{\tabcolsep}{2.75pt}
	\caption{Basic statistics of hypergraphs derived from StackExchange sites. 
		$n$ is the number of nodes, $m$ is the number of edges, and columns labeled $i\in[5]$ count edges of cardinality $i$.
	}\label{tab:stex-1}
	\begin{tabular}{lrrrrrrrr}
\toprule
 & $n$ & $m$ & $\nicefrac{n}{m}$ & 1 & 2 & 3 & 4 & 5 \\
\midrule
3dprinting & 416 & 4\,902 & 0.084863 & 1\,003 & 1\,617 & 1\,367 & 649 & 266 \\
3dprinting.meta & 45 & 197 & 0.228426 & 65 & 85 & 38 & 5 & 4 \\
academia & 457 & 39\,270 & 0.011637 & 6\,428 & 11\,831 & 11\,360 & 6\,294 & 3\,357 \\
academia.meta & 91 & 1\,237 & 0.073565 & 396 & 486 & 249 & 95 & 11 \\
ai & 980 & 10\,204 & 0.096041 & 767 & 1\,805 & 2\,696 & 2\,427 & 2\,509 \\
ai.meta & 49 & 315 & 0.155556 & 100 & 132 & 67 & 11 & 5 \\
alcohol & 154 & 1\,138 & 0.135325 & 415 & 406 & 229 & 56 & 32 \\
alcohol.meta & 28 & 94 & 0.297872 & 28 & 42 & 14 & 8 & 2 \\
android & 1\,517 & 56\,403 & 0.026896 & 12\,890 & 18\,313 & 14\,406 & 6\,996 & 3\,798 \\
android.meta & 103 & 996 & 0.103414 & 159 & 447 & 281 & 97 & 12 \\
anime & 1\,528 & 12\,122 & 0.126052 & 9\,510 & 2\,215 & 348 & 43 & 6 \\
anime.meta & 83 & 900 & 0.092222 & 234 & 384 & 215 & 56 & 11 \\
apple & 969 & 121\,999 & 0.007943 & 15\,822 & 34\,777 & 37\,243 & 22\,652 & 11\,505 \\
apple.meta & 108 & 1\,452 & 0.074380 & 354 & 601 & 393 & 90 & 14 \\
arduino & 445 & 23\,616 & 0.018843 & 5\,838 & 7\,357 & 6\,027 & 2\,858 & 1\,536 \\
arduino.meta & 50 & 255 & 0.196078 & 101 & 110 & 34 & 10 & 0 \\
askubuntu & 3\,137 & 393\,266 & 0.007977 & 68\,310 & 104\,529 & 105\,601 & 68\,907 & 45\,919 \\
astronomy & 566 & 12\,773 & 0.044312 & 2\,781 & 3\,812 & 3\,284 & 1\,777 & 1\,119 \\
astronomy.meta & 63 & 339 & 0.185841 & 115 & 93 & 76 & 43 & 12 \\
aviation & 1\,024 & 22\,701 & 0.045108 & 4\,294 & 7\,193 & 6\,384 & 3\,231 & 1\,599 \\
aviation.meta & 73 & 752 & 0.097074 & 247 & 295 & 155 & 46 & 9 \\
bicycles & 548 & 18\,873 & 0.029036 & 4\,884 & 6\,267 & 4\,652 & 2\,097 & 973 \\
bicycles.meta & 74 & 442 & 0.167421 & 150 & 197 & 76 & 15 & 4 \\
bioacoustics & 354 & 287 & 1.233449 & 20 & 50 & 101 & 54 & 62 \\
bioacoustics.meta & 36 & 49 & 0.734694 & 4 & 24 & 16 & 5 & 0 \\
bioinformatics & 490 & 4\,998 & 0.098039 & 922 & 1\,420 & 1\,335 & 782 & 539 \\
bioinformatics.meta & 29 & 112 & 0.258929 & 44 & 53 & 15 & 0 & 0 \\
biology & 745 & 27\,348 & 0.027241 & 5\,487 & 8\,618 & 7\,093 & 3\,742 & 2\,408 \\
biology.meta & 88 & 814 & 0.108108 & 280 & 331 & 145 & 44 & 14 \\
bitcoin & 936 & 28\,882 & 0.032408 & 6\,677 & 8\,927 & 7\,432 & 3\,766 & 2\,080 \\
bitcoin.meta & 58 & 434 & 0.133641 & 142 & 202 & 71 & 16 & 3 \\
blender & 371 & 98\,724 & 0.003758 & 31\,012 & 30\,861 & 22\,200 & 9\,614 & 5\,037 \\
blender.meta & 69 & 716 & 0.096369 & 273 & 291 & 108 & 35 & 9 \\
boardgames & 1\,000 & 13\,166 & 0.075953 & 9\,800 & 2\,779 & 500 & 75 & 12 \\
boardgames.meta & 75 & 659 & 0.113809 & 197 & 289 & 144 & 27 & 2 \\
bricks & 202 & 4\,220 & 0.047867 & 1\,391 & 1\,669 & 805 & 266 & 89 \\
bricks.meta & 52 & 211 & 0.246445 & 45 & 95 & 51 & 17 & 3 \\
buddhism & 487 & 7\,956 & 0.061212 & 2\,381 & 2\,357 & 1\,730 & 896 & 592 \\
buddhism.meta & 59 & 491 & 0.120163 & 104 & 252 & 94 & 30 & 11 \\
cardano & 285 & 2\,248 & 0.126779 & 585 & 664 & 548 & 277 & 174 \\
cardano.meta & 24 & 43 & 0.558140 & 18 & 15 & 10 & 0 & 0 \\
chemistry & 370 & 41\,571 & 0.008900 & 9\,725 & 14\,183 & 10\,803 & 4\,790 & 2\,070 \\
chemistry.meta & 90 & 1\,034 & 0.087041 & 250 & 441 & 243 & 88 & 12 \\
chess & 387 & 7\,864 & 0.049212 & 1\,646 & 2\,682 & 2\,069 & 985 & 482 \\
chess.meta & 62 & 368 & 0.168478 & 102 & 183 & 72 & 9 & 2 \\
chinese & 166 & 10\,298 & 0.016120 & 4\,467 & 3\,438 & 1\,628 & 543 & 222 \\
chinese.meta & 60 & 349 & 0.171920 & 93 & 170 & 67 & 12 & 7 \\
christianity & 1\,129 & 14\,955 & 0.075493 & 1\,739 & 3\,571 & 4\,205 & 2\,967 & 2\,473 \\
christianity.meta & 110 & 1\,579 & 0.069664 & 593 & 589 & 285 & 88 & 24 \\
civicrm & 507 & 14\,324 & 0.035395 & 4\,639 & 5\,150 & 3\,085 & 1\,083 & 367 \\
civicrm.meta & 18 & 69 & 0.260870 & 43 & 18 & 6 & 2 & 0 \\
codegolf & 257 & 13\,228 & 0.019428 & 1\,360 & 4\,586 & 4\,379 & 2\,106 & 797 \\
codegolf.meta & 128 & 2\,276 & 0.056239 & 559 & 848 & 549 & 245 & 75 \\
codereview & 1\,114 & 76\,105 & 0.014638 & 6\,306 & 20\,542 & 23\,777 & 16\,106 & 9\,374 \\
codereview.meta & 133 & 1\,947 & 0.068310 & 190 & 615 & 688 & 345 & 109 \\
\bottomrule
\end{tabular}
 \end{table}

\begin{table}[t]
	\centering
	\setlength{\tabcolsep}{2.75pt}
	\caption{Basic statistics of hypergraphs derived from StackExchange sites (continued). 
		$n$ is the number of nodes, $m$ is the number of edges, and columns labeled $i\in[5]$ count edges of cardinality~$i$.}\label{tab:stex-2}
	\begin{tabular}{lrrrrrrrr}
\toprule
 & $n$ & $m$ & $\nicefrac{n}{m}$ & 1 & 2 & 3 & 4 & 5 \\
\midrule
coffee & 114 & 1\,381 & 0.082549 & 492 & 524 & 260 & 78 & 27 \\
coffee.meta & 27 & 90 & 0.300000 & 45 & 30 & 13 & 2 & 0 \\
communitybuilding & 74 & 559 & 0.132379 & 148 & 219 & 112 & 55 & 25 \\
communitybuilding.meta & 27 & 132 & 0.204545 & 36 & 67 & 24 & 4 & 1 \\
computergraphics & 259 & 3\,600 & 0.071944 & 883 & 1\,024 & 877 & 489 & 327 \\
computergraphics.meta & 34 & 150 & 0.226667 & 55 & 66 & 27 & 2 & 0 \\
conlang & 96 & 448 & 0.214286 & 109 & 204 & 91 & 32 & 12 \\
conlang.meta & 21 & 61 & 0.344262 & 16 & 34 & 7 & 4 & 0 \\
cooking & 834 & 25\,877 & 0.032229 & 6\,568 & 9\,266 & 6\,344 & 2\,682 & 1\,017 \\
cooking.meta & 83 & 866 & 0.095843 & 241 & 410 & 178 & 34 & 3 \\
craftcms & 523 & 13\,756 & 0.038020 & 3\,738 & 4\,912 & 3\,410 & 1\,263 & 433 \\
craftcms.meta & 20 & 50 & 0.400000 & 22 & 11 & 15 & 1 & 1 \\
crafts & 193 & 2\,039 & 0.094654 & 706 & 828 & 397 & 84 & 24 \\
crafts.meta & 49 & 184 & 0.266304 & 40 & 88 & 45 & 11 & 0 \\
crypto & 506 & 27\,447 & 0.018436 & 6\,448 & 9\,056 & 6\,960 & 3\,283 & 1\,700 \\
crypto.meta & 74 & 542 & 0.136531 & 139 & 237 & 127 & 27 & 12 \\
cs & 656 & 44\,794 & 0.014645 & 8\,624 & 14\,332 & 12\,644 & 6\,336 & 2\,858 \\
cs.meta & 86 & 603 & 0.142620 & 90 & 247 & 185 & 68 & 13 \\
cseducators & 210 & 1\,080 & 0.194444 & 297 & 378 & 252 & 116 & 37 \\
cseducators.meta & 29 & 146 & 0.198630 & 52 & 68 & 26 & 0 & 0 \\
cstheory & 498 & 11\,959 & 0.041642 & 1\,653 & 3\,384 & 3\,495 & 2\,052 & 1\,375 \\
cstheory.meta & 80 & 608 & 0.131579 & 157 & 262 & 156 & 30 & 3 \\
datascience & 663 & 33\,997 & 0.019502 & 4\,110 & 8\,028 & 9\,305 & 6\,753 & 5\,801 \\
datascience.meta & 51 & 237 & 0.215190 & 80 & 97 & 38 & 16 & 6 \\
dba & 1\,197 & 96\,887 & 0.012355 & 15\,956 & 29\,750 & 27\,361 & 15\,610 & 7\,682 \\
dba.meta & 76 & 800 & 0.095000 & 280 & 334 & 140 & 38 & 8 \\
devops & 431 & 5\,025 & 0.085771 & 1\,070 & 1\,647 & 1\,340 & 616 & 352 \\
devops.meta & 40 & 144 & 0.277778 & 45 & 63 & 31 & 5 & 0 \\
diy & 919 & 71\,007 & 0.012942 & 19\,347 & 22\,079 & 17\,371 & 8\,399 & 3\,811 \\
diy.meta & 68 & 603 & 0.112769 & 227 & 233 & 118 & 21 & 4 \\
drones & 220 & 731 & 0.300958 & 114 & 240 & 193 & 115 & 69 \\
drones.meta & 28 & 62 & 0.451613 & 11 & 31 & 17 & 3 & 0 \\
drupal & 149 & 86\,283 & 0.001727 & 25\,218 & 37\,599 & 18\,867 & 4\,075 & 524 \\
drupal.meta & 75 & 1\,014 & 0.073964 & 361 & 432 & 186 & 35 & 0 \\
dsp & 509 & 24\,850 & 0.020483 & 4\,460 & 6\,779 & 6\,565 & 4\,081 & 2\,965 \\
dsp.meta & 48 & 307 & 0.156352 & 153 & 108 & 30 & 14 & 2 \\
earthscience & 424 & 6\,329 & 0.066993 & 1\,111 & 1\,778 & 1\,698 & 1\,094 & 648 \\
earthscience.meta & 54 & 321 & 0.168224 & 100 & 145 & 63 & 12 & 1 \\
ebooks & 180 & 1\,466 & 0.122783 & 364 & 489 & 339 & 163 & 111 \\
ebooks.meta & 39 & 99 & 0.393939 & 31 & 37 & 23 & 6 & 2 \\
economics & 494 & 13\,690 & 0.036085 & 3\,488 & 4\,426 & 3\,160 & 1\,678 & 938 \\
economics.meta & 60 & 444 & 0.135135 & 241 & 151 & 40 & 7 & 5 \\
electronics & 2\,318 & 175\,731 & 0.013191 & 31\,201 & 46\,423 & 46\,974 & 29\,107 & 22\,026 \\
electronics.meta & 107 & 1\,685 & 0.063501 & 698 & 628 & 282 & 62 & 15 \\
elementaryos & 314 & 8\,471 & 0.037068 & 3\,043 & 2\,910 & 1\,669 & 619 & 230 \\
elementaryos.meta & 29 & 107 & 0.271028 & 60 & 28 & 17 & 2 & 0 \\
ell & 533 & 99\,970 & 0.005332 & 46\,764 & 31\,310 & 14\,644 & 5\,147 & 2\,105 \\
ell.meta & 93 & 1\,224 & 0.075980 & 448 & 489 & 226 & 52 & 9 \\
emacs & 891 & 23\,939 & 0.037220 & 7\,561 & 9\,371 & 4\,980 & 1\,590 & 437 \\
emacs.meta & 51 & 216 & 0.236111 & 34 & 112 & 59 & 10 & 1 \\
engineering & 468 & 13\,867 & 0.033749 & 3\,582 & 4\,121 & 3\,315 & 1\,770 & 1\,079 \\
engineering.meta & 47 & 217 & 0.216590 & 71 & 87 & 45 & 10 & 4 \\
english & 984 & 125\,848 & 0.007819 & 48\,232 & 38\,850 & 23\,112 & 10\,111 & 5\,543 \\
english.meta & 182 & 3\,589 & 0.050711 & 1\,224 & 1\,305 & 733 & 249 & 78 \\
eosio & 241 & 2\,422 & 0.099505 & 766 & 766 & 533 & 245 & 112 \\
\bottomrule
\end{tabular}
 \end{table}

\begin{table}[t]
	\centering
	\setlength{\tabcolsep}{2.75pt}
	\caption{Basic statistics of hypergraphs derived from StackExchange sites (continued). 
		$n$ is the number of nodes, $m$ is the number of edges, and columns labeled $i\in[5]$ count edges of cardinality~$i$.}\label{tab:stex-3}
	\begin{tabular}{lrrrrrrrr}
\toprule
 & $n$ & $m$ & $\nicefrac{n}{m}$ & 1 & 2 & 3 & 4 & 5 \\
\midrule
eosio.meta & 19 & 27 & 0.703704 & 6 & 14 & 4 & 2 & 1 \\
es.meta.stackoverflow & 168 & 1\,817 & 0.092460 & 310 & 665 & 568 & 230 & 44 \\
es.stackoverflow & 2\,960 & 179\,452 & 0.016495 & 38\,027 & 58\,218 & 47\,343 & 23\,415 & 12\,449 \\
esperanto & 99 & 1\,592 & 0.062186 & 1\,050 & 422 & 96 & 16 & 8 \\
esperanto.meta & 20 & 84 & 0.238095 & 37 & 38 & 9 & 0 & 0 \\
ethereum & 891 & 46\,678 & 0.019088 & 8\,449 & 12\,402 & 12\,327 & 7\,687 & 5\,813 \\
ethereum.meta & 63 & 259 & 0.243243 & 98 & 71 & 59 & 26 & 5 \\
expatriates & 304 & 7\,182 & 0.042328 & 1\,068 & 2\,178 & 2\,163 & 1\,156 & 617 \\
expatriates.meta & 48 & 157 & 0.305732 & 41 & 72 & 41 & 2 & 1 \\
expressionengine & 603 & 12\,447 & 0.048445 & 3\,724 & 4\,239 & 2\,901 & 1\,150 & 433 \\
expressionengine.meta & 35 & 123 & 0.284553 & 59 & 49 & 15 & 0 & 0 \\
fitness & 402 & 9\,667 & 0.041585 & 2\,123 & 2\,864 & 2\,427 & 1\,289 & 964 \\
fitness.meta & 54 & 315 & 0.171429 & 126 & 123 & 57 & 7 & 2 \\
freelancing & 125 & 1\,946 & 0.064234 & 632 & 654 & 394 & 177 & 89 \\
freelancing.meta & 33 & 132 & 0.250000 & 36 & 64 & 25 & 5 & 2 \\
french & 324 & 12\,413 & 0.026102 & 3\,368 & 4\,126 & 2\,923 & 1\,390 & 606 \\
french.meta & 73 & 290 & 0.251724 & 58 & 127 & 80 & 24 & 1 \\
gamedev & 1\,096 & 54\,182 & 0.020228 & 7\,381 & 16\,130 & 15\,996 & 9\,433 & 5\,242 \\
gamedev.meta & 78 & 910 & 0.085714 & 300 & 430 & 148 & 27 & 5 \\
gaming & 5\,883 & 98\,355 & 0.059814 & 72\,655 & 20\,708 & 4\,120 & 758 & 114 \\
gaming.meta & 177 & 4\,062 & 0.043575 & 478 & 1\,853 & 1\,219 & 425 & 87 \\
gardening & 526 & 16\,629 & 0.031631 & 3\,725 & 5\,390 & 4\,122 & 2\,097 & 1\,295 \\
gardening.meta & 60 & 320 & 0.187500 & 95 & 157 & 49 & 17 & 2 \\
genealogy & 465 & 3\,572 & 0.130179 & 421 & 742 & 1\,037 & 902 & 470 \\
genealogy.meta & 56 & 485 & 0.115464 & 133 & 273 & 70 & 8 & 1 \\
german & 265 & 16\,022 & 0.016540 & 6\,003 & 5\,915 & 2\,914 & 927 & 263 \\
german.meta & 69 & 540 & 0.127778 & 177 & 224 & 107 & 30 & 2 \\
gis & 2\,829 & 150\,205 & 0.018834 & 13\,868 & 36\,527 & 45\,339 & 32\,527 & 21\,944 \\
gis.meta & 91 & 1\,016 & 0.089567 & 174 & 361 & 317 & 125 & 39 \\
graphicdesign & 612 & 34\,820 & 0.017576 & 7\,542 & 10\,789 & 9\,364 & 4\,821 & 2\,304 \\
graphicdesign.meta & 83 & 851 & 0.097532 & 253 & 338 & 187 & 58 & 15 \\
ham & 334 & 4\,299 & 0.077692 & 927 & 1\,287 & 1\,199 & 610 & 276 \\
ham.meta & 45 & 156 & 0.288462 & 39 & 65 & 32 & 18 & 2 \\
hardwarerecs & 246 & 3\,945 & 0.062357 & 1\,201 & 1\,366 & 823 & 378 & 177 \\
hardwarerecs.meta & 42 & 255 & 0.164706 & 81 & 100 & 58 & 16 & 0 \\
hermeneutics & 422 & 12\,563 & 0.033591 & 2\,819 & 3\,720 & 3\,074 & 1\,772 & 1\,178 \\
hermeneutics.meta & 63 & 581 & 0.108434 & 256 & 212 & 84 & 22 & 7 \\
hinduism & 825 & 15\,771 & 0.052311 & 2\,597 & 4\,337 & 3\,976 & 2\,876 & 1\,985 \\
hinduism.meta & 89 & 827 & 0.107618 & 196 & 295 & 200 & 98 & 38 \\
history & 843 & 13\,784 & 0.061158 & 2\,071 & 3\,757 & 3\,839 & 2\,436 & 1\,681 \\
history.meta & 68 & 746 & 0.091153 & 340 & 265 & 107 & 31 & 3 \\
homebrew & 415 & 6\,113 & 0.067888 & 1\,393 & 1\,976 & 1\,593 & 803 & 348 \\
homebrew.meta & 50 & 172 & 0.290698 & 67 & 63 & 35 & 4 & 3 \\
hsm & 252 & 3\,898 & 0.064649 & 982 & 1\,272 & 928 & 464 & 252 \\
hsm.meta & 32 & 146 & 0.219178 & 61 & 44 & 37 & 4 & 0 \\
interpersonal & 280 & 3\,890 & 0.071979 & 342 & 1\,030 & 1\,307 & 790 & 421 \\
interpersonal.meta & 76 & 825 & 0.092121 & 214 & 328 & 205 & 62 & 16 \\
iot & 241 & 2\,103 & 0.114598 & 560 & 754 & 504 & 193 & 92 \\
iot.meta & 36 & 136 & 0.264706 & 30 & 74 & 27 & 5 & 0 \\
iota & 148 & 1\,023 & 0.144673 & 300 & 352 & 248 & 84 & 39 \\
iota.meta & 18 & 38 & 0.473684 & 10 & 20 & 8 & 0 & 0 \\
islam & 562 & 13\,792 & 0.040748 & 3\,018 & 4\,990 & 3\,557 & 1\,519 & 708 \\
islam.meta & 103 & 864 & 0.119213 & 240 & 358 & 206 & 47 & 13 \\
italian & 94 & 3\,590 & 0.026184 & 1\,296 & 1\,376 & 636 & 206 & 76 \\
italian.meta & 27 & 151 & 0.178808 & 77 & 57 & 14 & 2 & 1 \\
\bottomrule
\end{tabular}
 \end{table}

\begin{table}[t]
	\centering
	\setlength{\tabcolsep}{2.75pt}
	\caption{Basic statistics of hypergraphs derived from StackExchange sites (continued). 
		$n$ is the number of nodes, $m$ is the number of edges, and columns labeled $i\in[5]$ count edges of cardinality~$i$.}\label{tab:stex-4}
	\begin{tabular}{lrrrrrrrr}
\toprule
 & $n$ & $m$ & $\nicefrac{n}{m}$ & 1 & 2 & 3 & 4 & 5 \\
\midrule
ja.meta.stackoverflow & 74 & 1\,115 & 0.066368 & 193 & 386 & 306 & 204 & 26 \\
ja.stackoverflow & 1\,145 & 28\,785 & 0.039778 & 10\,077 & 10\,518 & 5\,624 & 1\,946 & 620 \\
japanese & 354 & 26\,365 & 0.013427 & 9\,325 & 8\,869 & 5\,191 & 2\,020 & 960 \\
japanese.meta & 75 & 817 & 0.091799 & 270 & 351 & 147 & 43 & 6 \\
joomla & 374 & 7\,190 & 0.052017 & 1\,289 & 2\,221 & 2\,058 & 1\,072 & 550 \\
joomla.meta & 41 & 150 & 0.273333 & 81 & 46 & 19 & 4 & 0 \\
judaism & 1\,264 & 36\,511 & 0.034620 & 3\,753 & 8\,116 & 10\,854 & 8\,042 & 5\,746 \\
judaism.meta & 147 & 1\,455 & 0.101031 & 108 & 576 & 489 & 222 & 60 \\
korean & 118 & 1\,716 & 0.068765 & 767 & 596 & 264 & 69 & 20 \\
korean.meta & 30 & 80 & 0.375000 & 38 & 28 & 8 & 5 & 1 \\
languagelearning & 216 & 1\,287 & 0.167832 & 225 & 466 & 354 & 176 & 66 \\
languagelearning.meta & 52 & 195 & 0.266667 & 31 & 103 & 48 & 12 & 1 \\
latin & 370 & 5\,400 & 0.068519 & 1\,223 & 1\,603 & 1\,371 & 797 & 406 \\
latin.meta & 46 & 192 & 0.239583 & 34 & 80 & 49 & 25 & 4 \\
law & 938 & 23\,649 & 0.039663 & 4\,483 & 7\,573 & 6\,329 & 3\,381 & 1\,883 \\
law.meta & 66 & 499 & 0.132265 & 117 & 216 & 120 & 36 & 10 \\
lifehacks & 140 & 2\,928 & 0.047814 & 1\,024 & 1\,052 & 595 & 190 & 67 \\
lifehacks.meta & 59 & 268 & 0.220149 & 65 & 122 & 72 & 6 & 3 \\
linguistics & 605 & 10\,003 & 0.060482 & 1\,947 & 2\,836 & 2\,556 & 1\,627 & 1\,037 \\
linguistics.meta & 59 & 363 & 0.162534 & 118 & 159 & 58 & 23 & 5 \\
literature & 2\,335 & 5\,614 & 0.415924 & 703 & 1\,621 & 2\,249 & 830 & 211 \\
literature.meta & 63 & 462 & 0.136364 & 56 & 292 & 99 & 15 & 0 \\
magento & 1\,811 & 110\,316 & 0.016416 & 15\,598 & 28\,805 & 32\,671 & 20\,873 & 12\,369 \\
magento.meta & 66 & 575 & 0.114783 & 251 & 227 & 78 & 17 & 2 \\
martialarts & 205 & 2\,199 & 0.093224 & 461 & 696 & 529 & 326 & 187 \\
martialarts.meta & 40 & 218 & 0.183486 & 66 & 97 & 46 & 9 & 0 \\
math.meta & 232 & 9\,169 & 0.025303 & 1\,051 & 3\,485 & 2\,919 & 1\,312 & 402 \\
matheducators & 225 & 3\,360 & 0.066964 & 696 & 1\,118 & 903 & 435 & 208 \\
matheducators.meta & 57 & 255 & 0.223529 & 64 & 119 & 61 & 8 & 3 \\
mathematica & 705 & 85\,069 & 0.008287 & 25\,896 & 31\,653 & 18\,182 & 6\,542 & 2\,796 \\
mathematica.meta & 75 & 914 & 0.082057 & 416 & 341 & 130 & 25 & 2 \\
mathoverflow.net & 1\,530 & 137\,735 & 0.011108 & 20\,381 & 37\,763 & 38\,643 & 24\,597 & 16\,351 \\
mattermodeling & 449 & 2\,422 & 0.185384 & 169 & 547 & 668 & 495 & 543 \\
mattermodeling.meta & 61 & 142 & 0.429577 & 25 & 41 & 29 & 37 & 10 \\
mechanics & 1\,430 & 25\,243 & 0.056649 & 4\,196 & 6\,245 & 7\,592 & 4\,673 & 2\,537 \\
mechanics.meta & 52 & 387 & 0.134367 & 124 & 182 & 66 & 13 & 2 \\
medicalsciences & 1\,435 & 7\,586 & 0.189164 & 1\,423 & 1\,970 & 1\,754 & 1\,261 & 1\,178 \\
medicalsciences.meta & 65 & 501 & 0.129741 & 171 & 191 & 102 & 27 & 10 \\
meta.askubuntu & 196 & 5\,698 & 0.034398 & 1\,625 & 2\,308 & 1\,257 & 397 & 111 \\
meta & 1\,250 & 97\,114 & 0.012871 & 4\,599 & 25\,289 & 34\,007 & 23\,233 & 9\,986 \\
meta.mathoverflow.net & 133 & 1\,687 & 0.078838 & 272 & 601 & 504 & 229 & 81 \\
meta.serverfault & 139 & 2\,173 & 0.063967 & 767 & 799 & 463 & 119 & 25 \\
meta.stackoverflow & 622 & 47\,387 & 0.013126 & 5\,297 & 15\,301 & 15\,792 & 8\,233 & 2\,764 \\
meta.superuser & 207 & 5\,000 & 0.041400 & 1\,010 & 1\,914 & 1\,474 & 510 & 92 \\
monero & 400 & 4\,285 & 0.093349 & 1\,193 & 1\,424 & 969 & 481 & 218 \\
monero.meta & 23 & 85 & 0.270588 & 40 & 26 & 19 & 0 & 0 \\
money & 1\,002 & 36\,187 & 0.027690 & 3\,788 & 8\,036 & 10\,340 & 8\,450 & 5\,573 \\
money.meta & 67 & 672 & 0.099702 & 220 & 260 & 147 & 40 & 5 \\
movies & 4\,537 & 21\,829 & 0.207843 & 4\,857 & 11\,430 & 4\,546 & 877 & 119 \\
movies.meta & 75 & 1\,285 & 0.058366 & 302 & 519 & 391 & 63 & 10 \\
music & 516 & 23\,424 & 0.022029 & 4\,754 & 7\,644 & 6\,370 & 3\,117 & 1\,539 \\
music.meta & 81 & 992 & 0.081653 & 391 & 387 & 166 & 40 & 8 \\
musicfans & 237 & 2\,990 & 0.079264 & 1\,209 & 1\,169 & 465 & 111 & 36 \\
musicfans.meta & 42 & 218 & 0.192661 & 62 & 95 & 38 & 18 & 5 \\
mythology & 303 & 1\,953 & 0.155146 & 484 & 723 & 439 & 215 & 92 \\
\bottomrule
\end{tabular}
 \end{table}

\begin{table}[t]
	\centering
	\setlength{\tabcolsep}{2.75pt}
	\caption{Basic statistics of hypergraphs derived from StackExchange sites (continued). 
		$n$ is the number of nodes, $m$ is the number of edges, and columns labeled $i\in[5]$ count edges of cardinality~$i$.}\label{tab:stex-5}
	\begin{tabular}{lrrrrrrrr}
\toprule
 & $n$ & $m$ & $\nicefrac{n}{m}$ & 1 & 2 & 3 & 4 & 5 \\
\midrule
mythology.meta & 35 & 162 & 0.216049 & 43 & 87 & 31 & 1 & 0 \\
networkengineering & 453 & 15\,624 & 0.028994 & 2\,988 & 4\,240 & 3\,835 & 2\,496 & 2\,065 \\
networkengineering.meta & 53 & 375 & 0.141333 & 192 & 115 & 48 & 17 & 3 \\
opendata & 302 & 5\,990 & 0.050417 & 1\,562 & 2\,002 & 1\,492 & 670 & 264 \\
opendata.meta & 26 & 180 & 0.144444 & 73 & 76 & 30 & 1 & 0 \\
opensource & 203 & 4\,226 & 0.048036 & 845 & 1\,442 & 1\,094 & 528 & 317 \\
opensource.meta & 53 & 225 & 0.235556 & 35 & 109 & 61 & 19 & 1 \\
or & 255 & 2\,865 & 0.089005 & 351 & 809 & 848 & 496 & 361 \\
or.meta & 44 & 114 & 0.385965 & 21 & 61 & 23 & 5 & 4 \\
outdoors & 555 & 5\,908 & 0.093940 & 934 & 2\,017 & 1\,791 & 806 & 360 \\
outdoors.meta & 52 & 512 & 0.101562 & 169 & 276 & 60 & 7 & 0 \\
parenting & 304 & 6\,636 & 0.045811 & 1\,182 & 2\,175 & 1\,873 & 1\,004 & 402 \\
parenting.meta & 61 & 473 & 0.128964 & 96 & 217 & 125 & 31 & 4 \\
patents & 2\,102 & 4\,381 & 0.479799 & 1\,421 & 1\,211 & 879 & 481 & 389 \\
patents.meta & 46 & 167 & 0.275449 & 55 & 69 & 34 & 8 & 1 \\
pets & 289 & 7\,874 & 0.036703 & 781 & 2\,706 & 2\,350 & 1\,305 & 732 \\
pets.meta & 62 & 407 & 0.152334 & 60 & 194 & 112 & 26 & 15 \\
philosophy & 606 & 17\,915 & 0.033826 & 4\,898 & 5\,399 & 4\,079 & 2\,089 & 1\,450 \\
philosophy.meta & 61 & 793 & 0.076923 & 355 & 258 & 127 & 38 & 15 \\
photo & 1\,156 & 25\,961 & 0.044528 & 3\,395 & 6\,960 & 7\,848 & 4\,936 & 2\,822 \\
photo.meta & 107 & 1\,095 & 0.097717 & 289 & 500 & 239 & 60 & 7 \\
physics & 892 & 209\,515 & 0.004257 & 21\,914 & 42\,808 & 53\,150 & 45\,705 & 45\,938 \\
physics.meta & 114 & 3\,228 & 0.035316 & 713 & 1\,085 & 872 & 403 & 155 \\
pm & 283 & 6\,198 & 0.045660 & 1\,379 & 1\,850 & 1\,592 & 870 & 507 \\
pm.meta & 64 & 315 & 0.203175 & 81 & 129 & 73 & 27 & 5 \\
poker & 131 & 2\,051 & 0.063871 & 763 & 659 & 372 & 181 & 76 \\
poker.meta & 29 & 122 & 0.237705 & 74 & 30 & 15 & 3 & 0 \\
politics & 793 & 14\,628 & 0.054211 & 1\,294 & 4\,022 & 4\,663 & 3\,062 & 1\,587 \\
politics.meta & 80 & 1\,067 & 0.074977 & 249 & 436 & 259 & 103 & 20 \\
portuguese & 169 & 2\,349 & 0.071946 & 703 & 898 & 509 & 174 & 65 \\
portuguese.meta & 35 & 137 & 0.255474 & 45 & 61 & 25 & 5 & 1 \\
proofassistants & 223 & 434 & 0.513825 & 80 & 175 & 116 & 42 & 21 \\
proofassistants.meta & 37 & 64 & 0.578125 & 11 & 26 & 18 & 7 & 2 \\
psychology & 401 & 7\,641 & 0.052480 & 1\,632 & 2\,229 & 1\,971 & 1\,115 & 694 \\
psychology.meta & 62 & 557 & 0.111311 & 199 & 237 & 90 & 25 & 6 \\
pt.meta.stackoverflow & 140 & 2\,986 & 0.046885 & 703 & 1\,081 & 775 & 362 & 65 \\
pt.stackoverflow & 2\,936 & 152\,483 & 0.019255 & 28\,143 & 50\,055 & 42\,386 & 21\,287 & 10\,612 \\
puzzling & 209 & 24\,985 & 0.008365 & 6\,912 & 9\,471 & 5\,731 & 2\,020 & 851 \\
puzzling.meta & 98 & 1\,365 & 0.071795 & 351 & 582 & 309 & 97 & 26 \\
quant & 693 & 20\,283 & 0.034167 & 3\,329 & 5\,345 & 5\,392 & 3\,556 & 2\,661 \\
quant.meta & 47 & 252 & 0.186508 & 95 & 115 & 37 & 3 & 2 \\
quantumcomputing & 306 & 7\,823 & 0.039115 & 1\,124 & 2\,585 & 2\,475 & 1\,105 & 534 \\
quantumcomputing.meta & 50 & 187 & 0.267380 & 50 & 73 & 43 & 18 & 3 \\
raspberrypi & 598 & 35\,872 & 0.016670 & 7\,901 & 11\,252 & 9\,351 & 4\,765 & 2\,603 \\
raspberrypi.meta & 61 & 451 & 0.135255 & 213 & 169 & 58 & 8 & 3 \\
retrocomputing & 546 & 4\,976 & 0.109727 & 925 & 1\,694 & 1\,366 & 692 & 299 \\
retrocomputing.meta & 70 & 304 & 0.230263 & 30 & 188 & 56 & 27 & 3 \\
reverseengineering & 347 & 8\,754 & 0.039639 & 1\,878 & 2\,693 & 2\,172 & 1\,249 & 762 \\
reverseengineering.meta & 37 & 150 & 0.246667 & 56 & 62 & 28 & 3 & 1 \\
robotics & 276 & 6\,261 & 0.044082 & 1\,528 & 1\,850 & 1\,519 & 806 & 558 \\
robotics.meta & 39 & 159 & 0.245283 & 52 & 71 & 28 & 8 & 0 \\
rpg & 1\,247 & 46\,635 & 0.026740 & 4\,236 & 12\,463 & 15\,431 & 9\,542 & 4\,963 \\
rpg.meta & 150 & 2\,627 & 0.057099 & 310 & 986 & 844 & 379 & 108 \\
ru.meta.stackoverflow & 242 & 4\,613 & 0.052460 & 445 & 1\,312 & 1\,574 & 979 & 303 \\
rus & 390 & 20\,999 & 0.018572 & 12\,276 & 5\,131 & 2\,341 & 840 & 411 \\
\bottomrule
\end{tabular}
 \end{table}

\begin{table}[t]
	\centering
	\setlength{\tabcolsep}{2.75pt}
	\caption{Basic statistics of hypergraphs derived from StackExchange sites (continued). 
		$n$ is the number of nodes, $m$ is the number of edges, and columns labeled $i\in[5]$ count edges of cardinality~$i$.}\label{tab:stex-6}
	\begin{tabular}{lrrrrrrrr}
\toprule
 & $n$ & $m$ & $\nicefrac{n}{m}$ & 1 & 2 & 3 & 4 & 5 \\
\midrule
rus.meta & 30 & 214 & 0.140187 & 92 & 81 & 37 & 4 & 0 \\
russian & 166 & 4\,516 & 0.036758 & 2\,407 & 1\,337 & 552 & 180 & 40 \\
russian.meta & 37 & 176 & 0.210227 & 80 & 61 & 25 & 7 & 3 \\
salesforce & 2\,085 & 124\,492 & 0.016748 & 22\,537 & 37\,977 & 33\,635 & 19\,220 & 11\,123 \\
salesforce.meta & 79 & 795 & 0.099371 & 412 & 246 & 118 & 18 & 1 \\
scicomp & 346 & 10\,381 & 0.033330 & 1\,905 & 3\,156 & 2\,883 & 1\,566 & 871 \\
scicomp.meta & 48 & 215 & 0.223256 & 75 & 90 & 42 & 8 & 0 \\
scifi & 3\,693 & 69\,344 & 0.053256 & 17\,338 & 26\,498 & 17\,146 & 6\,584 & 1\,778 \\
scifi.meta & 149 & 3\,265 & 0.045636 & 506 & 1\,560 & 889 & 266 & 44 \\
security & 1\,253 & 65\,817 & 0.019038 & 11\,950 & 19\,799 & 18\,266 & 9\,809 & 5\,993 \\
security.meta & 101 & 1\,124 & 0.089858 & 311 & 507 & 242 & 52 & 12 \\
serverfault & 3\,864 & 314\,342 & 0.012292 & 40\,967 & 83\,417 & 92\,763 & 60\,560 & 36\,635 \\
sharepoint & 1\,722 & 99\,911 & 0.017235 & 16\,092 & 27\,312 & 28\,073 & 17\,305 & 11\,129 \\
sharepoint.meta & 78 & 581 & 0.134251 & 206 & 233 & 127 & 14 & 1 \\
sitecore & 362 & 11\,395 & 0.031768 & 5\,106 & 4\,265 & 1\,611 & 342 & 71 \\
sitecore.meta & 24 & 202 & 0.118812 & 40 & 60 & 99 & 3 & 0 \\
skeptics & 682 & 10\,700 & 0.063738 & 2\,227 & 4\,165 & 2\,952 & 1\,042 & 314 \\
skeptics.meta & 100 & 1\,529 & 0.065402 & 528 & 605 & 310 & 77 & 9 \\
softwareengineering & 1\,674 & 61\,392 & 0.027267 & 8\,950 & 17\,773 & 17\,580 & 10\,572 & 6\,517 \\
softwareengineering.meta & 165 & 2\,611 & 0.063194 & 421 & 1\,023 & 776 & 310 & 81 \\
softwarerecs & 962 & 21\,792 & 0.044145 & 3\,090 & 6\,533 & 6\,199 & 3\,723 & 2\,247 \\
softwarerecs.meta & 85 & 654 & 0.129969 & 86 & 297 & 189 & 66 & 16 \\
sound & 1\,224 & 9\,786 & 0.125077 & 2\,122 & 2\,717 & 2\,330 & 1\,624 & 993 \\
sound.meta & 42 & 160 & 0.262500 & 65 & 66 & 25 & 1 & 3 \\
space & 1\,203 & 17\,392 & 0.069170 & 1\,672 & 4\,012 & 4\,924 & 3\,712 & 3\,072 \\
space.meta & 74 & 682 & 0.108504 & 205 & 237 & 150 & 63 & 27 \\
spanish & 274 & 8\,592 & 0.031890 & 2\,276 & 2\,722 & 2\,140 & 1\,010 & 444 \\
spanish.meta & 84 & 498 & 0.168675 & 94 & 216 & 135 & 42 & 11 \\
sports & 261 & 5\,730 & 0.045550 & 926 & 2\,371 & 1\,637 & 609 & 187 \\
sports.meta & 57 & 350 & 0.162857 & 76 & 170 & 82 & 21 & 1 \\
sqa & 462 & 11\,242 & 0.041096 & 2\,263 & 3\,250 & 2\,881 & 1\,705 & 1\,143 \\
sqa.meta & 41 & 211 & 0.194313 & 115 & 71 & 17 & 7 & 1 \\
stackapps & 210 & 2\,756 & 0.076197 & 277 & 858 & 883 & 514 & 224 \\
stats & 1\,572 & 196\,835 & 0.007986 & 19\,622 & 47\,967 & 57\,502 & 41\,443 & 30\,301 \\
stats.meta & 132 & 1\,685 & 0.078338 & 327 & 576 & 491 & 198 & 93 \\
stellar & 115 & 1\,493 & 0.077026 & 585 & 438 & 298 & 109 & 63 \\
stellar.meta & 19 & 31 & 0.612903 & 9 & 14 & 8 & 0 & 0 \\
substrate & 512 & 1\,814 & 0.282249 & 366 & 563 & 491 & 260 & 134 \\
substrate.meta & 40 & 44 & 0.909091 & 6 & 21 & 13 & 2 & 2 \\
superuser & 5\,676 & 480\,854 & 0.011804 & 64\,273 & 127\,561 & 135\,549 & 91\,137 & 62\,334 \\
sustainability & 234 & 2\,012 & 0.116302 & 431 & 713 & 536 & 235 & 97 \\
sustainability.meta & 37 & 151 & 0.245033 & 38 & 75 & 32 & 6 & 0 \\
tex & 2\,035 & 237\,763 & 0.008559 & 60\,247 & 84\,998 & 59\,476 & 23\,747 & 9\,295 \\
tex.meta & 163 & 2\,277 & 0.071585 & 389 & 921 & 671 & 235 & 61 \\
tezos & 210 & 1\,828 & 0.114880 & 567 & 605 & 380 & 180 & 96 \\
tezos.meta & 18 & 32 & 0.562500 & 7 & 15 & 8 & 1 & 1 \\
tor & 218 & 5\,636 & 0.038680 & 1\,888 & 1\,817 & 1\,147 & 464 & 320 \\
tor.meta & 43 & 163 & 0.263804 & 57 & 76 & 25 & 4 & 1 \\
travel & 1\,916 & 45\,040 & 0.042540 & 2\,985 & 8\,914 & 13\,809 & 11\,528 & 7\,804 \\
travel.meta & 99 & 1\,379 & 0.071791 & 293 & 567 & 406 & 98 & 15 \\
tridion & 274 & 7\,234 & 0.037877 & 1\,471 & 2\,758 & 1\,915 & 818 & 272 \\
tridion.meta & 14 & 138 & 0.101449 & 93 & 39 & 6 & 0 & 0 \\
ukrainian & 124 & 2\,094 & 0.059217 & 664 & 873 & 404 & 127 & 26 \\
ukrainian.meta & 33 & 104 & 0.317308 & 21 & 45 & 31 & 6 & 1 \\
unix & 2\,777 & 220\,644 & 0.012586 & 29\,059 & 61\,964 & 66\,657 & 40\,340 & 22\,624 \\
\bottomrule
\end{tabular}
 \end{table}

\clearpage

\begin{table}[t]
	\centering
	\setlength{\tabcolsep}{2.75pt}
	\caption{Basic statistics of hypergraphs derived from StackExchange sites (continued). 
		$n$ is the number of nodes, $m$ is the number of edges, and columns labeled $i\in[5]$ count edges of cardinality~$i$.}\label{tab:stex-7}
	\begin{tabular}{lrrrrrrrr}
\toprule
 & $n$ & $m$ & $\nicefrac{n}{m}$ & 1 & 2 & 3 & 4 & 5 \\
\midrule
unix.meta & 118 & 1\,668 & 0.070743 & 367 & 727 & 407 & 144 & 23 \\
ux & 1\,032 & 31\,459 & 0.032805 & 4\,660 & 8\,934 & 8\,823 & 5\,530 & 3\,512 \\
ux.meta & 94 & 899 & 0.104561 & 273 & 358 & 199 & 54 & 15 \\
vegetarianism & 115 & 677 & 0.169867 & 85 & 233 & 205 & 106 & 48 \\
vegetarianism.meta & 41 & 133 & 0.308271 & 26 & 62 & 32 & 13 & 0 \\
vi & 421 & 12\,558 & 0.033524 & 4\,494 & 4\,802 & 2\,358 & 694 & 210 \\
vi.meta & 35 & 201 & 0.174129 & 63 & 105 & 30 & 3 & 0 \\
video & 327 & 8\,661 & 0.037755 & 2\,705 & 2\,693 & 1\,831 & 882 & 550 \\
video.meta & 41 & 200 & 0.205000 & 63 & 96 & 32 & 8 & 1 \\
webapps & 951 & 33\,202 & 0.028643 & 14\,343 & 11\,667 & 5\,160 & 1\,435 & 597 \\
webapps.meta & 106 & 937 & 0.113127 & 97 & 447 & 311 & 76 & 6 \\
webmasters & 1\,078 & 36\,840 & 0.029262 & 5\,772 & 10\,197 & 10\,531 & 6\,286 & 4\,054 \\
webmasters.meta & 70 & 649 & 0.107858 & 202 & 258 & 135 & 45 & 9 \\
windowsphone & 287 & 3\,440 & 0.083430 & 975 & 1\,257 & 801 & 306 & 101 \\
windowsphone.meta & 44 & 148 & 0.297297 & 47 & 64 & 27 & 8 & 2 \\
woodworking & 244 & 3\,739 & 0.065258 & 1\,129 & 1\,270 & 880 & 347 & 113 \\
woodworking.meta & 34 & 142 & 0.239437 & 69 & 46 & 25 & 2 & 0 \\
wordpress & 702 & 112\,778 & 0.006225 & 27\,669 & 37\,039 & 28\,491 & 13\,228 & 6\,351 \\
wordpress.meta & 82 & 866 & 0.094688 & 381 & 330 & 118 & 30 & 7 \\
workplace & 498 & 30\,369 & 0.016398 & 6\,371 & 9\,325 & 8\,103 & 4\,221 & 2\,349 \\
workplace.meta & 113 & 1\,829 & 0.061782 & 506 & 699 & 447 & 150 & 27 \\
worldbuilding & 675 & 34\,358 & 0.019646 & 2\,958 & 8\,284 & 10\,839 & 7\,267 & 5\,010 \\
worldbuilding.meta & 120 & 2\,032 & 0.059055 & 445 & 901 & 511 & 147 & 28 \\
writing & 391 & 11\,699 & 0.033422 & 2\,456 & 3\,869 & 3\,055 & 1\,557 & 762 \\
writing.meta & 88 & 789 & 0.111534 & 145 & 415 & 173 & 49 & 7 \\
\bottomrule
\end{tabular}
 	\vspace*{12cm}
\end{table}
 
\clearpage

\subsection{Implementation Details}
\label{apx-implementation}

To simplify the computation of Wasserstein distances between adjacent nodes, 
we leverage the following fact about the relevant distances (i.e., transportation costs) between nodes.

\begin{lemma}\label{lem:truncated-distances}
	Given a hypergraph $\hypergraph = (\vertices, \edges)$ and nodes $i,j,k,\ell\in\vertices$ with $i\adjacent j$ as well as $\mu_i(k) > 0$ and $\mu_j(\ell) > 0$, 
	$\dist(k,\ell) \leq 3$.
\end{lemma}
\begin{proof}
	By the triangle inequality and the definition of our probability measures, we have
	\begin{align*}
		\dist(k,\ell) \leq \dist(k,i) + \dist(i,j) + \dist(j,\ell) = 3\;.
	\end{align*}
\end{proof}

Furthermore, we speed up the computation of Wasserstein distances by exploiting the following observation to reduce each instance to its smallest equivalent instance.

\begin{lemma}
	Given a hypergraph $\hypergraph = (\vertices, \edges)$ and nodes $i,j\in\vertices$ with $i\adjacent j$, 
	if $\mu_i(k) = \mu_j(k)$ for some node $k\in\vertices$, 
	then $\wasserstein_1(\mu_i,\mu_j) = \wasserstein_1(\mu^{-k}_i,\mu^{-k}_j)$, 
	where $\mu^{-k}_i$ is defined as 
	\begin{align*}
		\mu^{-k}_i(j) \defeq \begin{cases}
			0&j = k\\
			\mu_i(j)&j\neq k\;.
		\end{cases}
	\end{align*}
\end{lemma}
\begin{proof}
	If $\mu_i(k) = \mu_j(k) = 0$, the claim holds trivially.
	Otherwise, $\mu_i(k) = \mu_j(k) = \beta > 0$.
	In this case, let $C^*$ be an optimal coupling between $\mu_i$ and $\mu_j$.
	If the probability mass allocated to $k$ by $\mu_i$ does not get moved at all in $C^*$, 
	it contributes $0$ to $\wasserstein_1(\mu_i,\mu_j)$, 
	and we are done.
	Therefore, assume otherwise.
	Then there exist nodes $p,q\in\vertices$ such that probability mass gets moved from $p$ to $k$ and from $k$ to $q$ in $C^*$.
	By the triangle inequality, 
	$\dist(p,q) \leq \dist(p,k) + \dist(k,q)$,
	and as $\dist(k,k) = 0$, 
	the cost of moving that mass directly from $p$ to $q$ and keeping all mass at $k$ 
	cannot be larger than the cost of moving the  mass from $p$ to $k$ and from $k$ to $q$. 
	Hence, we can modify $C^*$ such that the mass allocated to $k$ by $\mu_i$ does not get moved at all without increasing the coupling cost.
	Thus, there always exists an optimal coupling in which all mass at $k$ remains at $k$,
	and the claim follows.
\end{proof} 
\clearpage

\subsection{Further Results}
\label{apx-results}

Here, we showcase further results to support and supplement the exposition in the main paper.

\paragraph{Q1 Parametrization.}
Expanding the discussion on \ourmethod parametrizations, 
\cref{fig:parameter-violins} shows the distributions of edge curvatures and edge-averaged node curvatures for two hypergraphs from the \dblpv collection, 
representing top conferences in machine learning and theoretical computer science, respectively. 
The figure highlights once more the consistently concentrating effect of increasing $\smoothing$, 
and it elucidates the differential effects of moving from maximum aggregation (left parts of the split violins) to mean aggregation (right parts of the split violins), 
from almost no shifts to large shifts in probability mass (compare, e.g., \cref{fig:parameter-violins:tcs}, top right panel, with \cref{fig:parameter-violins:tcs}, bottom left panel). 
\cref{fig:parameter-violins} might convey the impression that, 
other parameters being equal, 
the distributions of curvatures based on $\mu^\enrw$ and $\mu^\werw$ are more similar to each other than to $\mu^\eerw$. 
This does not hold in general, however, as demonstrated for \ndcpc in \cref{fig:ndcpc-violins}, 
where node curvature distributions based on $\mu^\werw$ are more similar to those based on $\mu^\eerw$ than to the node curvature distributions based on $\mu^\enrw$.
Comparing \cref{fig:ndcpc-violins} to \cref{fig:ndcai-violins} (\ndcai), 
we further observe that rather similar distributions of edge curvature and directional curvature can be accompanied by rather different distributions of edge-averaged and direction-averaged node curvatures, 
even for hypergraphs originating from the same domain.
Finally, when visualizing curvatures for hypergraphs in the same collection or across collections with related semantics (\cref{fig:aps-vcout-violins}), 
we can identify several distinct prototypical shapes of curvature distributions and relationships between curvatures based on different probability measures.

\begin{figure}[tbp]
	\centering
	\begin{subfigure}[t]{0.9\linewidth}
		\centering
		\includegraphics[width=\linewidth]{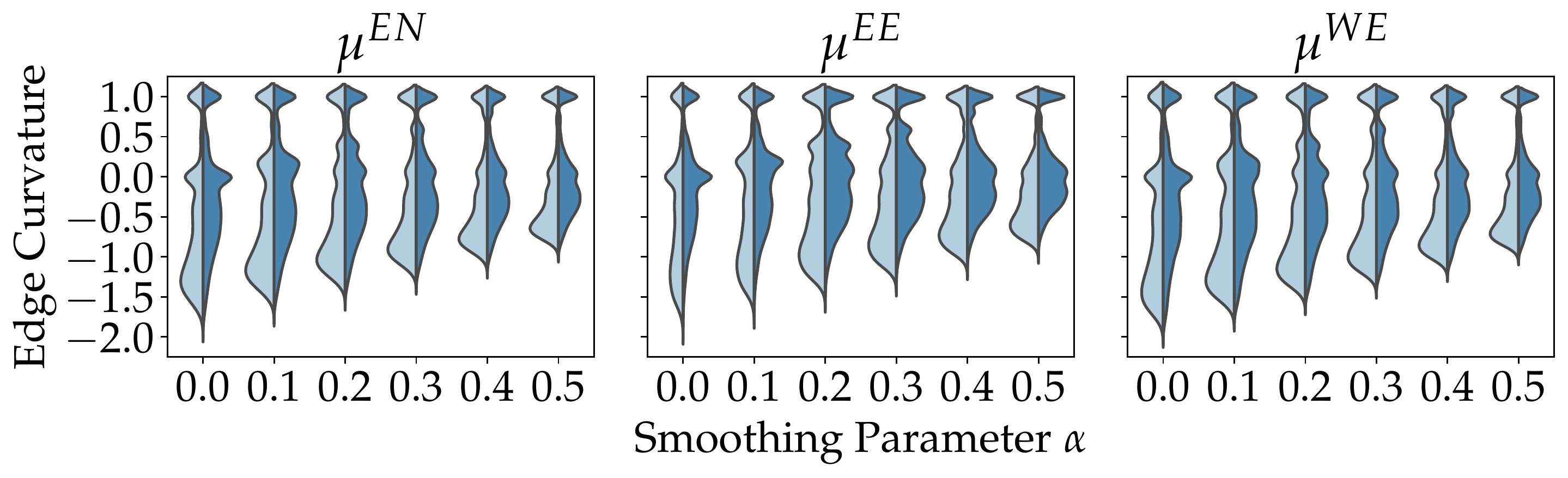}
		\includegraphics[width=\linewidth]{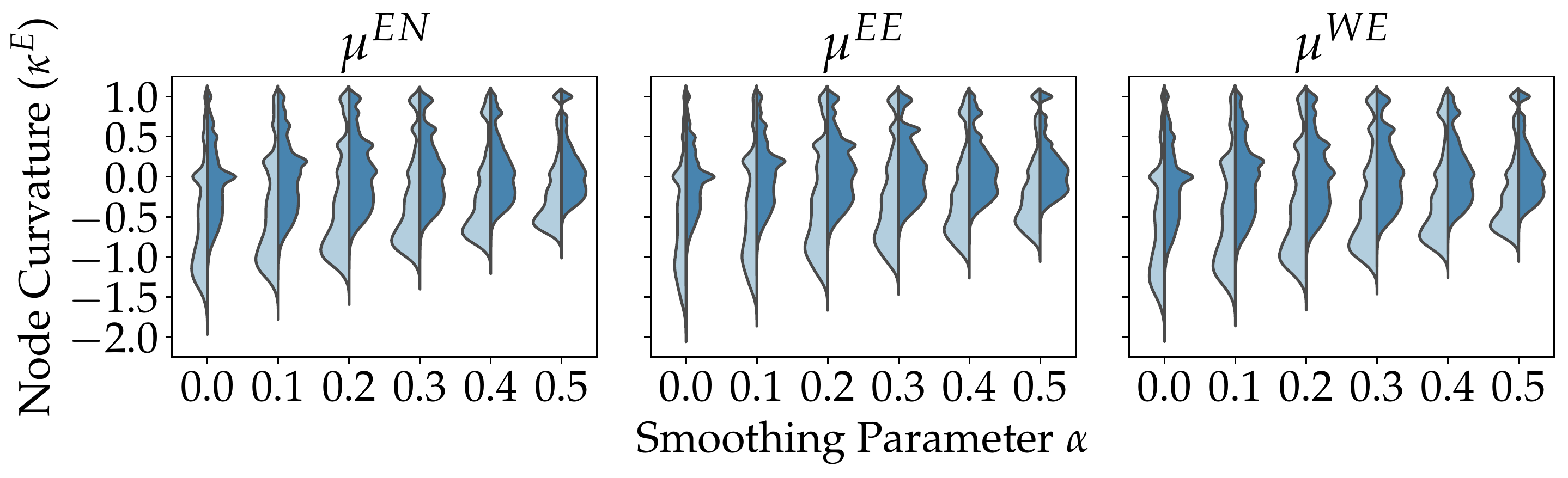}
		\caption{Top Conferences in Machine Learning}\label{fig:parameter-violins:ml}
	\end{subfigure}
	\begin{subfigure}[t]{0.9\linewidth}
		\centering
		\includegraphics[width=\linewidth]{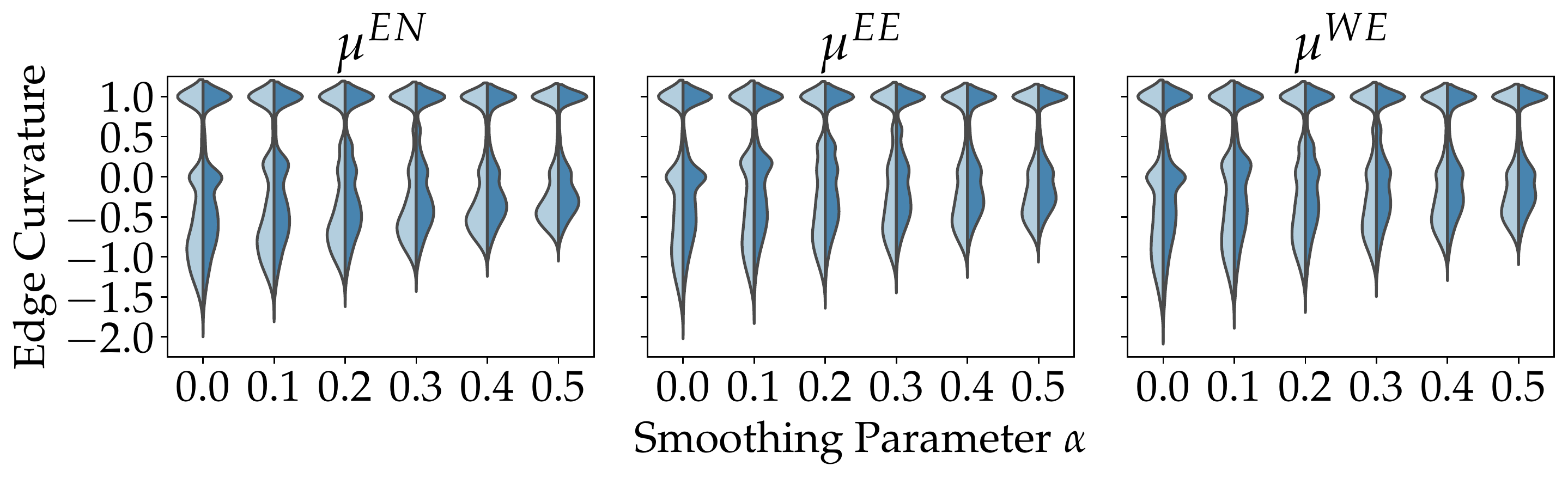}
		\includegraphics[width=\linewidth]{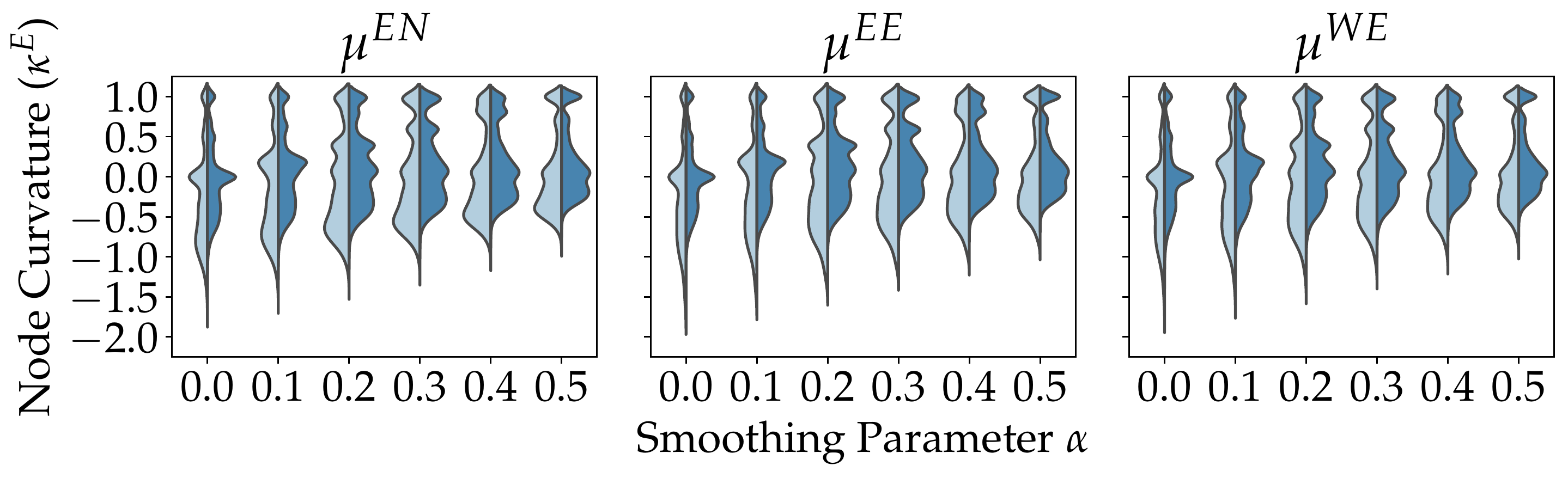}
		\caption{Top Conferences in Theoretical  Computer Science}\label{fig:parameter-violins:tcs}
	\end{subfigure}
	\caption{\ourmethod curvatures are non-redundant.
		We show distributions of \ourmethod edge curvatures (top) and edge-averaged node curvatures (bottom) 
		using probability measures $\mu^\enrw$, $\mu^\eerw$, and $\mu^\werw$ 
		with smoothing $\alpha$,
		for the aggregation functions $\aggregation_\maxi$ (light blue) and $\aggregation_\mean$ (dark blue) 
		on \dblpv hypergraphs representing top conferences in machine learning and in theoretical computer science.
	}\label{fig:parameter-violins}
\end{figure}

\begin{figure}[tbp]
	\centering
	\begin{subfigure}[t]{0.9\linewidth}
		\centering
		\includegraphics[width=\linewidth]{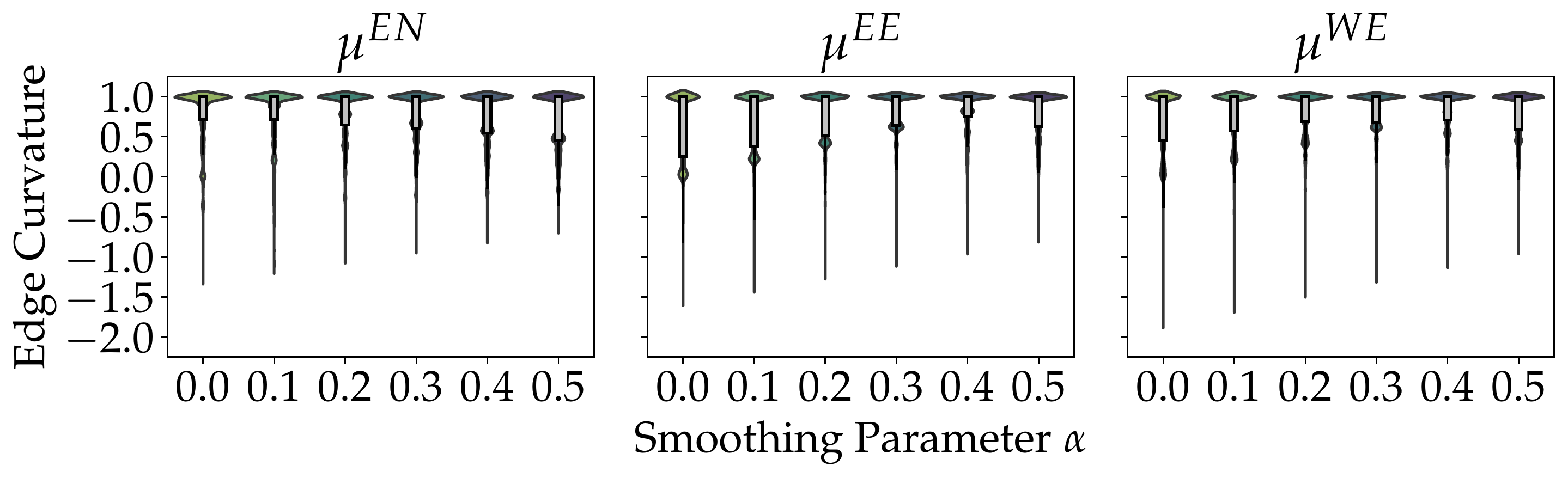}
		\includegraphics[width=\linewidth]{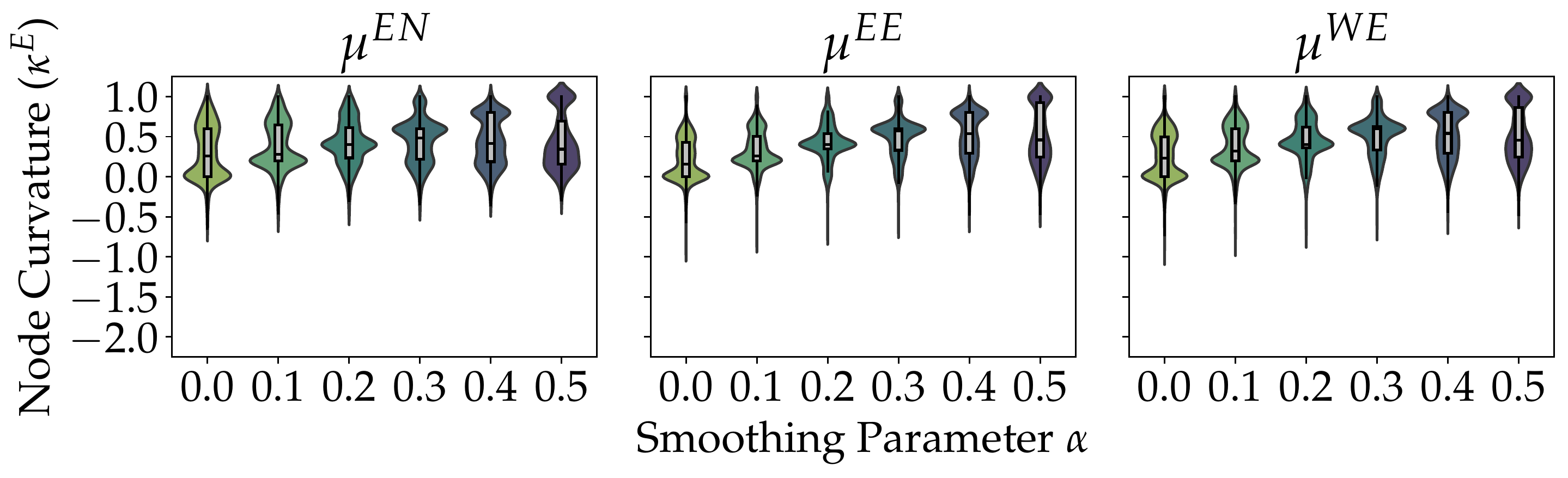}
		\includegraphics[width=\linewidth]{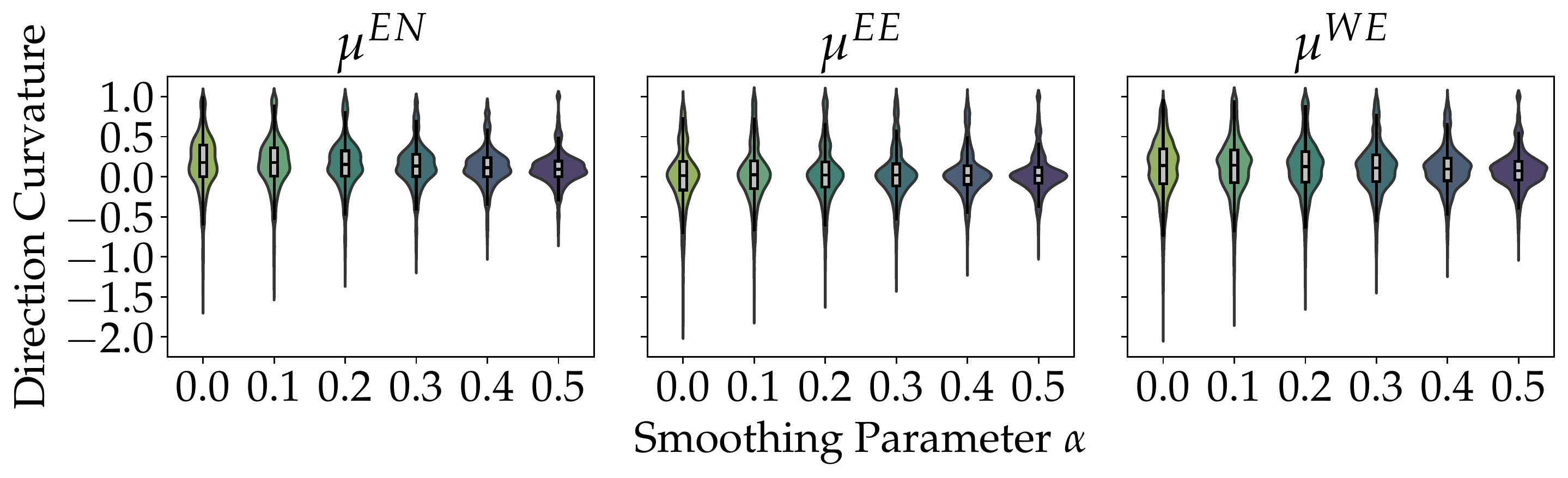}

		\includegraphics[width=\linewidth]{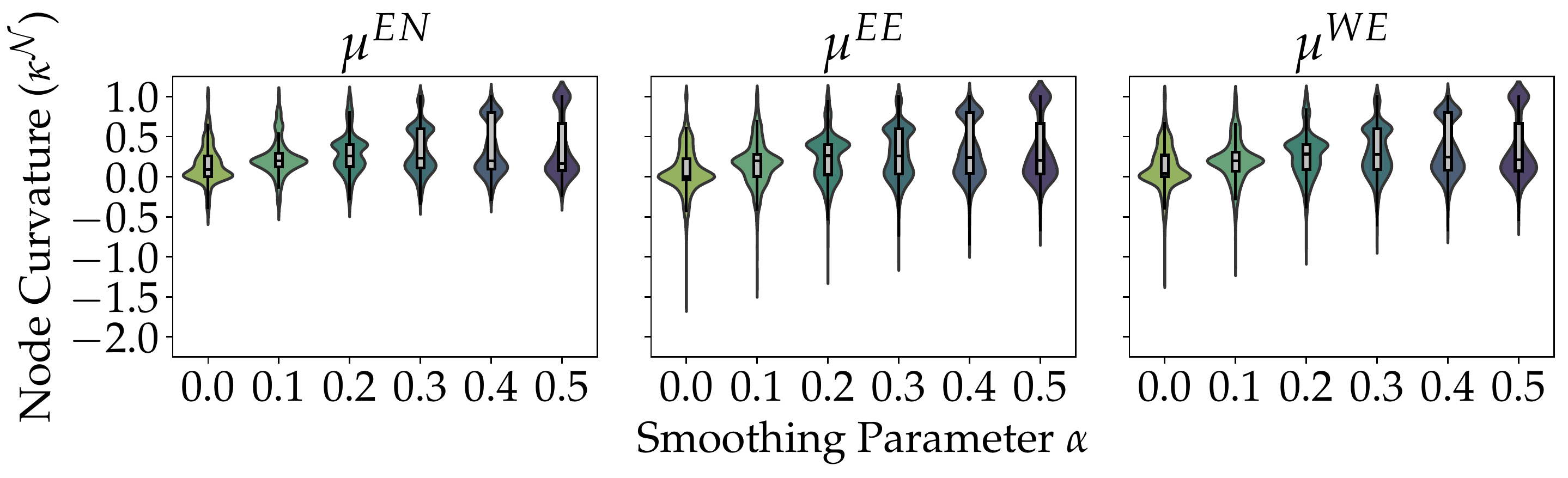}
		\subcaption{\ndcpc
		}\label{fig:ndcpc-violins}
	\end{subfigure}\caption{Hypergraphs with similar distributions of one curvature type may differ in their distributions of other curvature types.
	We show \ourmethod curvatures computed using $\aggregation_\mean$, for all curvature types, probability measures, and $\smoothing \in \{0.0, 0.1, 0.2, 0.3, 0.4, 0.5\}$. (Figure continues on next page.)
}\label{fig:ndc-violins}
\end{figure}

\begin{figure}
  \ContinuedFloat
  \centering
	\begin{subfigure}[t]{0.9\linewidth}
		\centering
		\includegraphics[width=\linewidth]{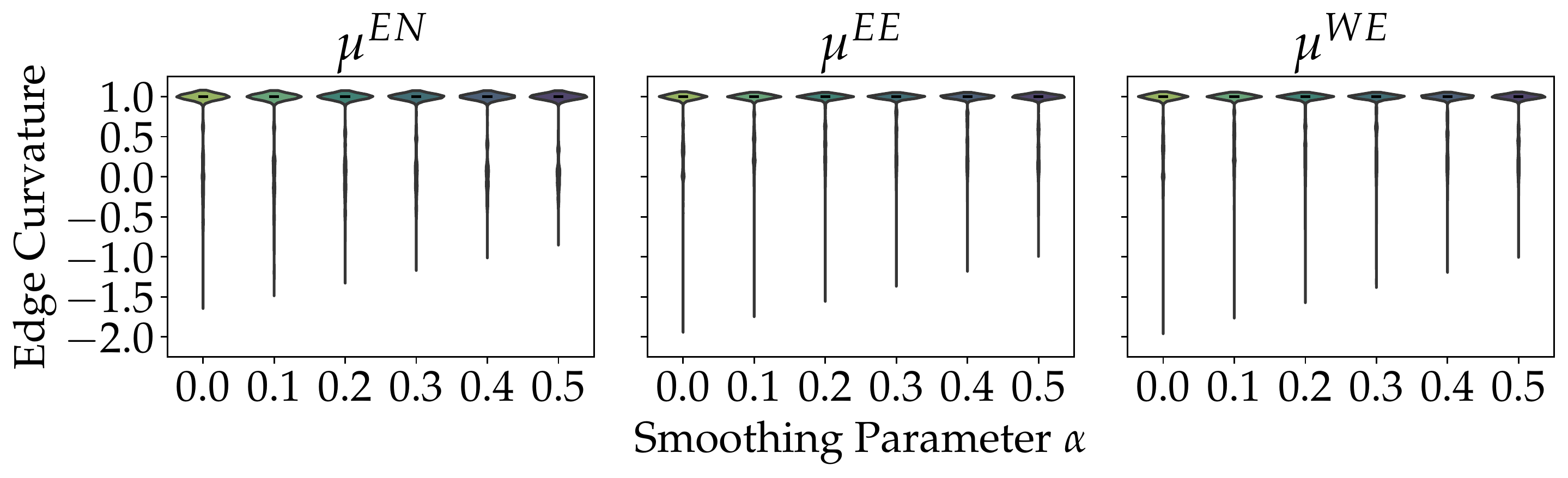}
		\includegraphics[width=\linewidth]{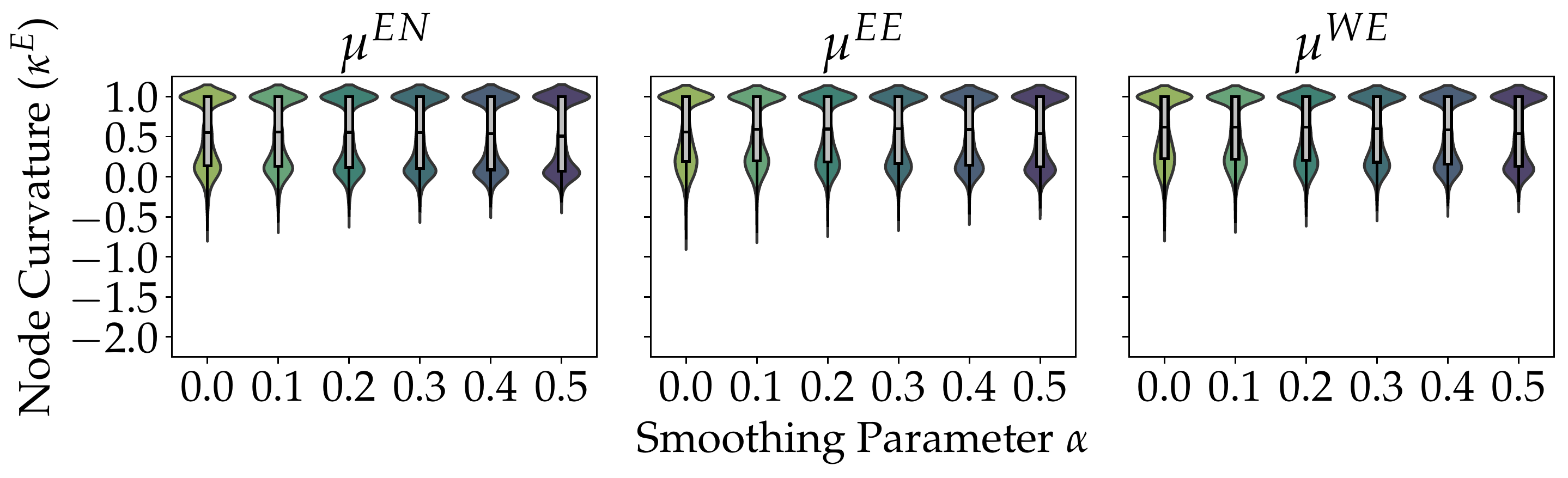}
		\includegraphics[width=\linewidth]{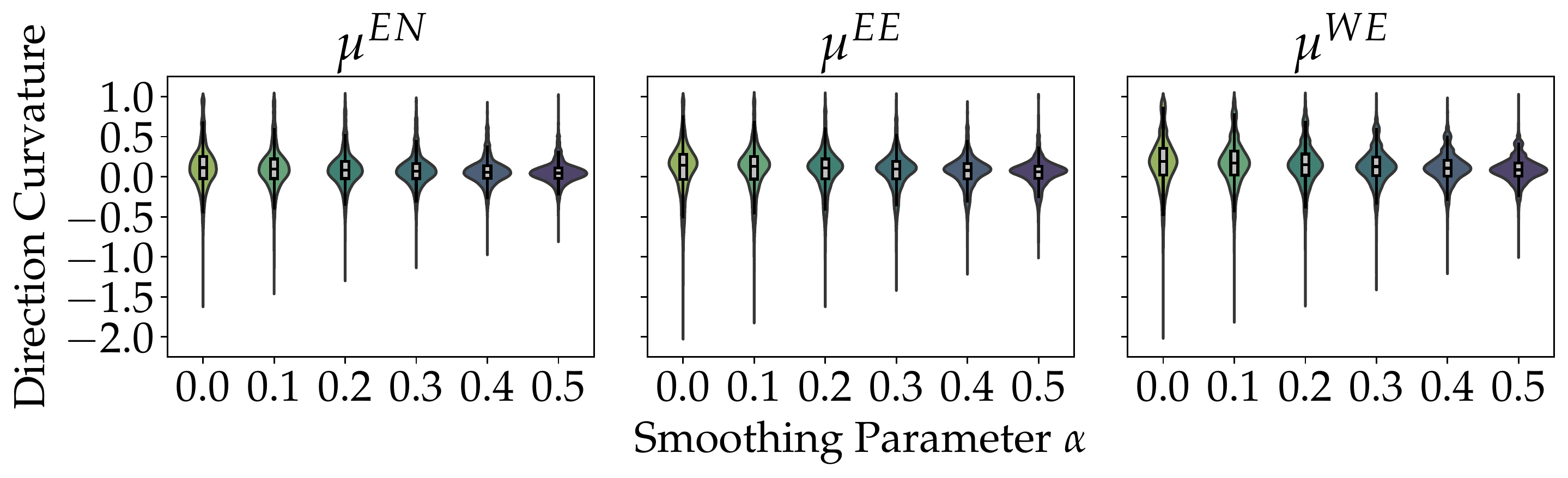}
		\includegraphics[width=\linewidth]{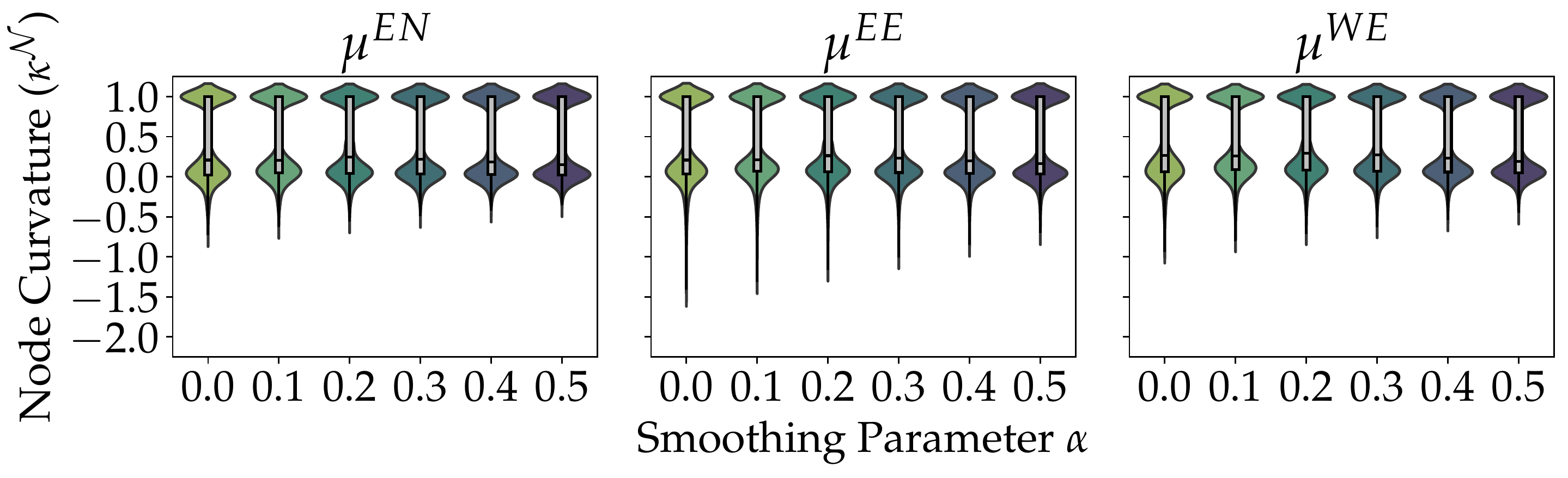}
		\subcaption{\ndcai
		}\label{fig:ndcai-violins}
	\end{subfigure}
  \caption{Hypergraphs with similar distributions of one curvature type may differ in their distributions of other curvature types.
    We show \ourmethod curvatures computed using $\aggregation_\mean$, for all curvature types, probability measures, and $\smoothing \in \{0.0, 0.1, 0.2, 0.3, 0.4, 0.5\}$.
  (Figure continued from previous page.)
  }
\end{figure}

\begin{figure}[tbp]
	\centering
	\begin{subfigure}[t]{0.8\linewidth}
		\includegraphics[width=\linewidth]{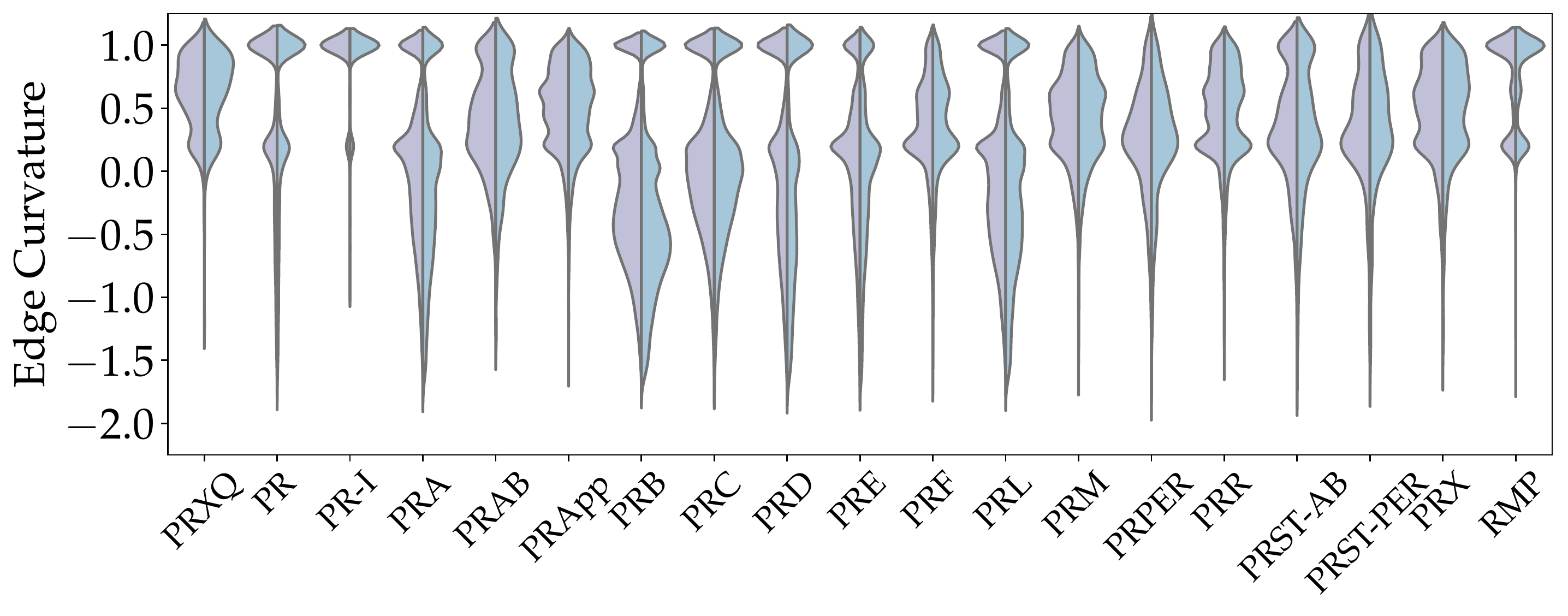}
	\end{subfigure}
	\begin{subfigure}[t]{0.8\linewidth}
		\includegraphics[width=\linewidth]{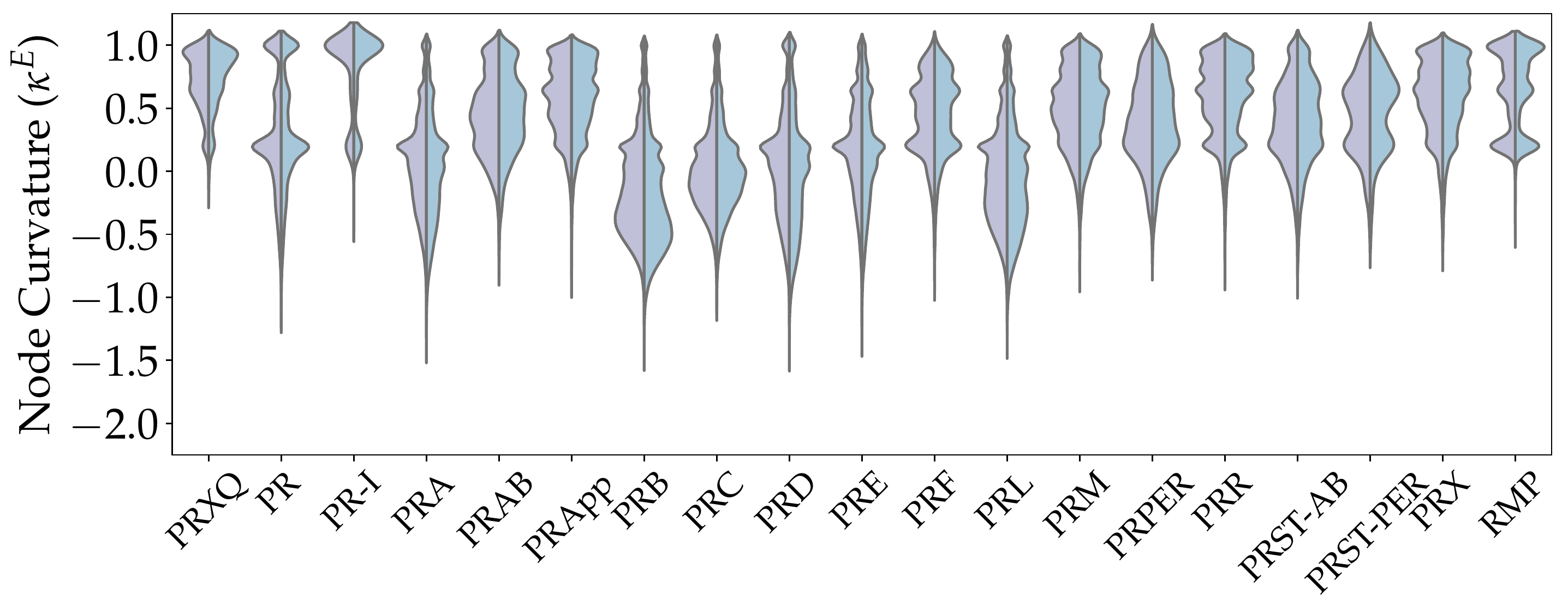}\vspace*{-6pt}
		\subcaption{\apsva}
	\end{subfigure}
  \vspace*{18pt}

	\begin{subfigure}[t]{0.8\linewidth}
		\includegraphics[width=\linewidth]{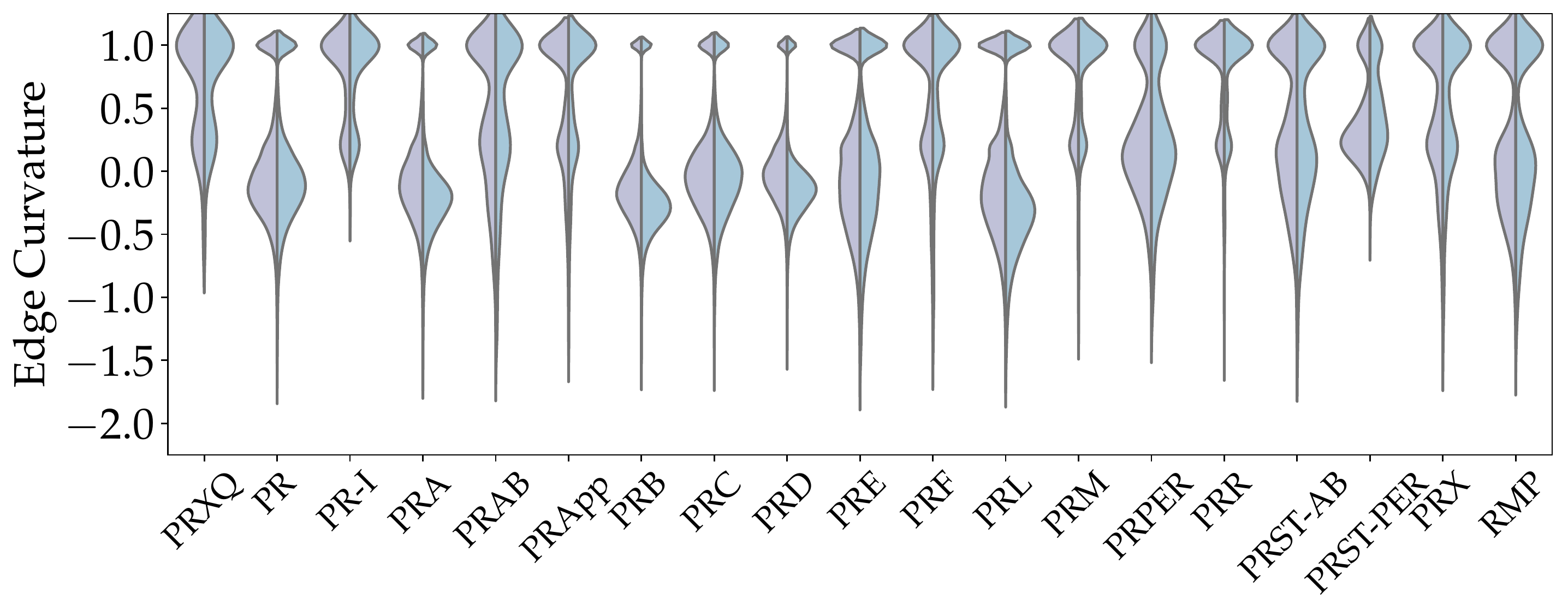}
	\end{subfigure}
	\begin{subfigure}[t]{0.8\linewidth}
		\includegraphics[width=\linewidth]{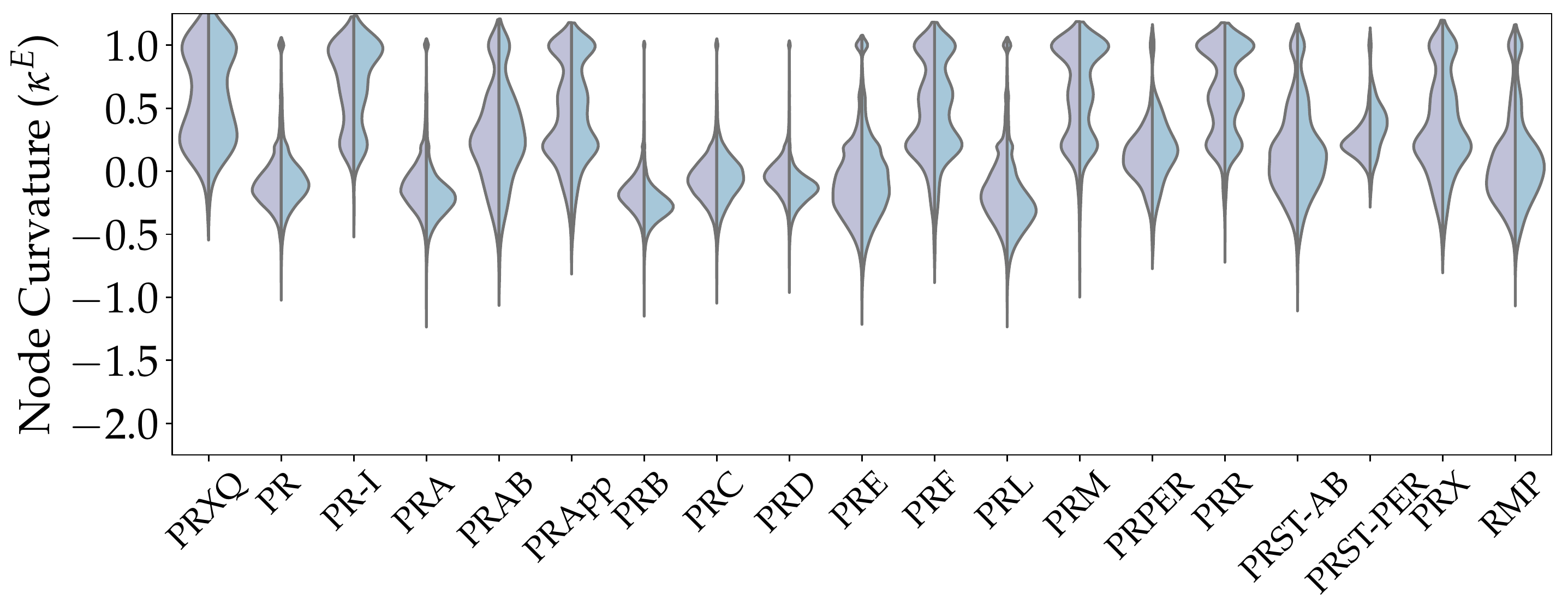}\vspace*{-6pt}
		\subcaption{\apsvcout}
	\end{subfigure}
	\caption{\ourmethod curvature distributions within the same collection and across semantically related collections exhibit prototypical shapes, 
		accompanied by varying types of relationships between probability measures.
		We show distributions of \ourmethod edge curvatures (top) and edge-averaged node curvatures (bottom) 
		computed using $\smoothing = 0.1$ and $\aggregation_\mean$, for $\mu^\eerw$ (violet) and $\mu^\werw$ (blue), 
		for all hypergraphs in \apsva and \apsvcout. 
		Recall that the edges in \apsva and \apsvcout as well as the nodes in \apsvcout represent essentially the same set of APS papers, 
		but in \apsva, they connect co-authors, 
		and in \apsvcout, they connect co-cited papers (edges) or are connected by citing papers (nodes).
	}\label{fig:aps-vcout-violins}
\end{figure}

\paragraph{Q2 Hypergraph Exploration.}

Extending the discussion of individual hypergraph exploration in the main paper, 
we focus on a case study of the citation hypergraph of the journal Physical Review E (PRE),
which regularly publishes, inter alia, interdisciplinary work on graphs and networks. 
In this hypergraph, which has 45\,504 nodes and 52\,574 edges,
nodes represent PRE articles \emph{cited} by at least one other PRE article, edges represent PRE articles \emph{citing} at least one other PRE article, and each edge $i$ comprises the nodes $j$ cited by the paper corresponding to $i$. 
Therefore, the \emph{edge} curvature of a (citing) paper $i$ can be interpreted as an indicator of its \emph{breadth of content}:
The more \emph{positive} the edge curvature,
the stronger the general tendency of the papers jointly cited by paper $i$ to be cited together, suggesting that these papers are topically related.
Similarly, the \emph{node} curvature of a (cited) paper $j$ can be interpreted as an indicator of its \emph{breadth of impact}: 
The more \emph{negative} the node curvature, 
the more diversely the paper has been cited in the literature.

With these interpretations in mind, 
we compute all curvatures for the PRE citation hypergraph, using $\smoothing=0.1$, $\mu^\werw$, and $\aggregation_\mean$.
We find that for all 54 articles with at least 100 citations (top articles), 
the edge-averaged node curvature is larger than the direction-averaged node curvature, which is always negative,
although only 36\% of all PRE articles exhibit this feature combination. 
This matches the intuition that from highly cited articles, the literature should diverge in many different directions.
At the same time, we observe that curvatures span a considerable range, even among top articles. 
In \cref{tab:aps-pre-exploration}, we record the top articles with extreme curvature values, 
and in \cref{fig:pairplot-pre}, 
we display the pairwise relationships between curvature features and other local features for \emph{all} PRE articles. 
In line with the interpretations sketched above, the top article with the largest node curvatures is a classic reference for community detection in the highly integrated field of network science, 
whereas the articles with the smallest node curvatures address topics relevant to a broader range of approaches to collective phenomena in many-body systems (which are the focus of PRE).

\begin{table}[t]
	\caption{Top articles display varying relationships between different curvature values.
		We list the PRE articles that, out of all PRE articles cited at least 100 times, 
		exhibit the most extreme curvature-related values.
	}\label{tab:aps-pre-exploration}
	\tiny
	\setlength{\tabcolsep}{4pt}
\begin{tabular}{p{0.085\linewidth}lrrrrp{0.34\linewidth}}
	\toprule
	{} &                   DOI & $\curvature^{\edges}(i)$ &  $\curvature^{\neighborhood}(i)$ &  $\Delta(\curvature(i))$ &   $\curvature(e)$ &                                                                                                                    Title \\
\midrule
$\max\curvature^{\edges}(i)$, $\max\curvature^{\neighborhood}(i)$ &  10.1103/PhysRevE.70.066111 &              0.220092 &                    -0.006001 &               0.226093 &        0.425336 &                                                                       Finding community structure in very large networks \\
$\min\curvature^{\edges}(i)$               &     10.1103/PhysRevE.47.851 &             -0.319638 &                    -0.555431 &               0.235793 &             0 &                                                                   Scale-invariant motion in intermittent chaotic systems \\
$\min\curvature^{\neighborhood}(i)$          &     10.1103/PhysRevE.48.R29 &             -0.241216 &                    -0.704752 &               0.463536 &             0 &                                                                              Extended self-similarity in turbulent flows \\
$\max\Delta(\curvature(i))$               &  10.1103/PhysRevE.64.056101 &             -0.131542 &                    -0.668266 &               0.536724 &        0.038477 &  Determining the density of states for classical statistical models: A random walk algorithm to produce a flat histogram \\
$\min\Delta(\curvature(i))$               &  10.1103/PhysRevE.74.016118 &             -0.015495 &                    -0.191193 &               0.175697 &       -0.156824 &                                                                         Amorphous systems in athermal, quasistatic shear \\
$\max\curvature(e)$                              &     10.1103/PhysRevE.57.610 &              0.129557 &                    -0.251635 &               0.381192 &        0.610123 &                                                                Topological defects and interactions in nematic emulsions \\
$\min\curvature(e)$                          &  10.1103/PhysRevE.64.016706 &             -0.191094 &                    -0.552908 &               0.361815 &       -0.644446 &                                                                  Fast Monte Carlo algorithm for site or bond percolation \\
\bottomrule
\end{tabular}
\end{table}

\begin{figure}[t]
	\centering
	\includegraphics[width=\linewidth]{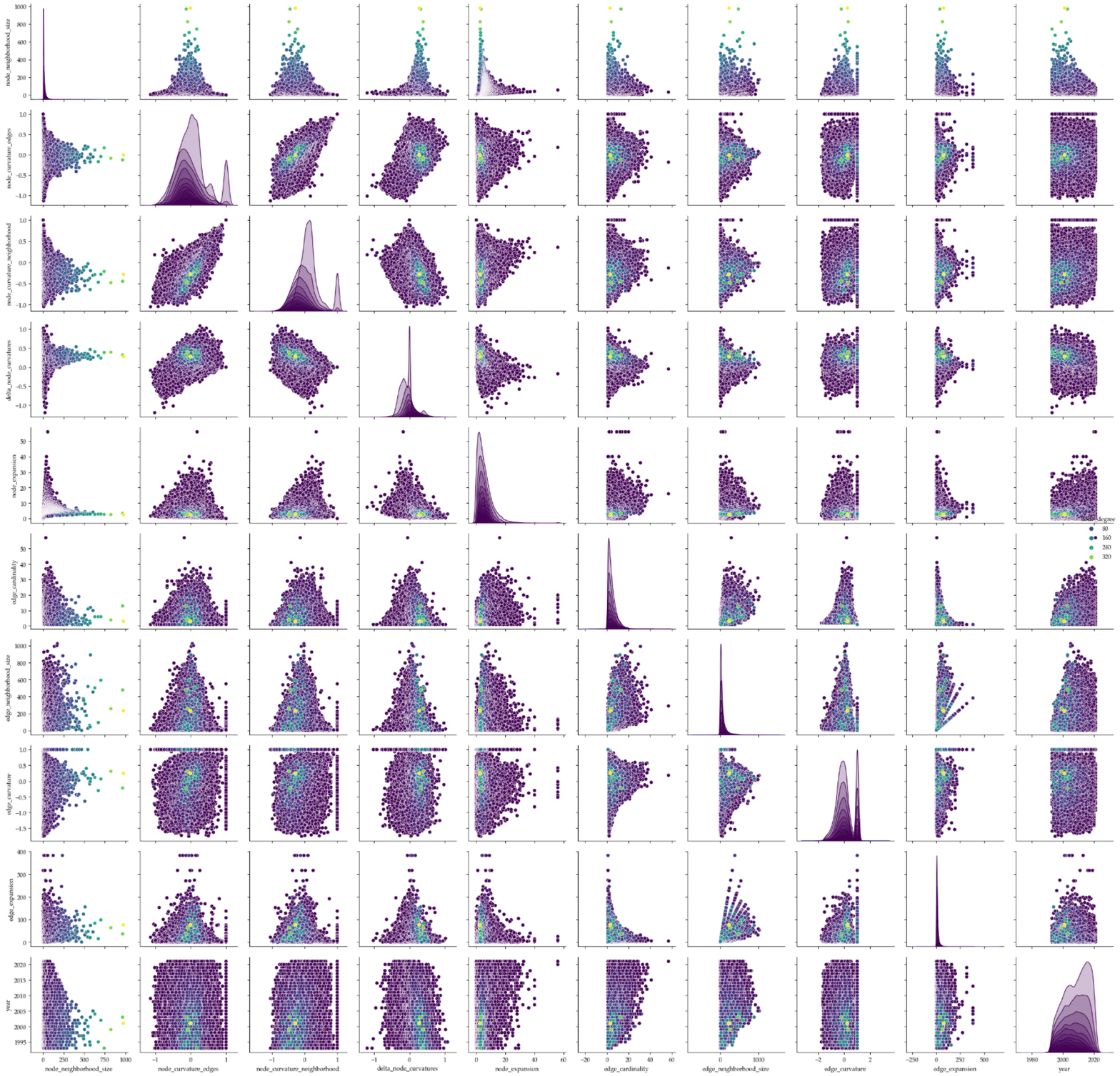}
	\caption{Highly cited articles have distinct curvature distributions.
		Pairwise relationships between (left-to-right, top-to-bottom) node neighborhood size, edge-averaged node curvature, direction-averaged node curvature, curvature delta, node expansion~$\defeq \nicefrac{\deg(i)}{|\neighborhood(i)|}$,
		edge cardinality, edge neighborhood size, edge curvature, edge expansion~$\defeq \nicefrac{\deg(e)}{|\neighborhood(e)|}$, and (as an additional metadata feature) publication year, 
		for all PRE articles cited at least once by another PRE article, 
		colored by node degree (number of citations within PRE), where brighter colors signal larger node degrees.
	 }\label{fig:pairplot-pre}
\end{figure}

\clearpage

\paragraph{Q3 Hypergraph Learning.}
Continuing the discussion of node clustering in hypergraphs abridged in the main paper, 
we again focus on the citation hypergraph corresponding to articles from Physical Review E (PRE). 
We experiment with a variety of features, clustering methods, and combinations thereof, 
including both classic and recent clustering methods, such as SPONGE \citep{cucuringu19sponge}.
We aim for 17 clusters, which is the number of ``disciplines'' present in the APS metadata (unfortunately, disciplines are only assigned to more recent articles, and hence, cannot serve as ground truth). 
As depicted in \cref{fig:aps-cout-nodes},
we find that clusterings generated using curvatures as features differ radically from clusterings generated using other local features.
To evaluate the semantic sensibility of our clusterings in the absence of a suitable ground truth, 
we leverage the metadata associated with PRE articles. 
In particular, we concatenate the titles of the articles grouped in each of our clusters into ``documents'', and consider the set of all clusters as our ``document collection'', 
to then identify characteristic terms for each cluster using TF-IDF feature extraction. 
We observe that clusterings based on \ourmethod features tend to be more thematically coherent than clusterings based on other local features. 
As illustrated in \cref{tab:aps-tfidf}, \ourmethod features tend to separate paper titles well by topic (many frequently occurring terms are associated with only very few clusters, and the terms grouped together characterize specific subfields of the physics of collective phenomena covered by PRE), 
whereas clusters based on non-\ourmethod features are much less topically focused.

\begin{figure}[t]
	\centering
	\includegraphics[width=\linewidth]{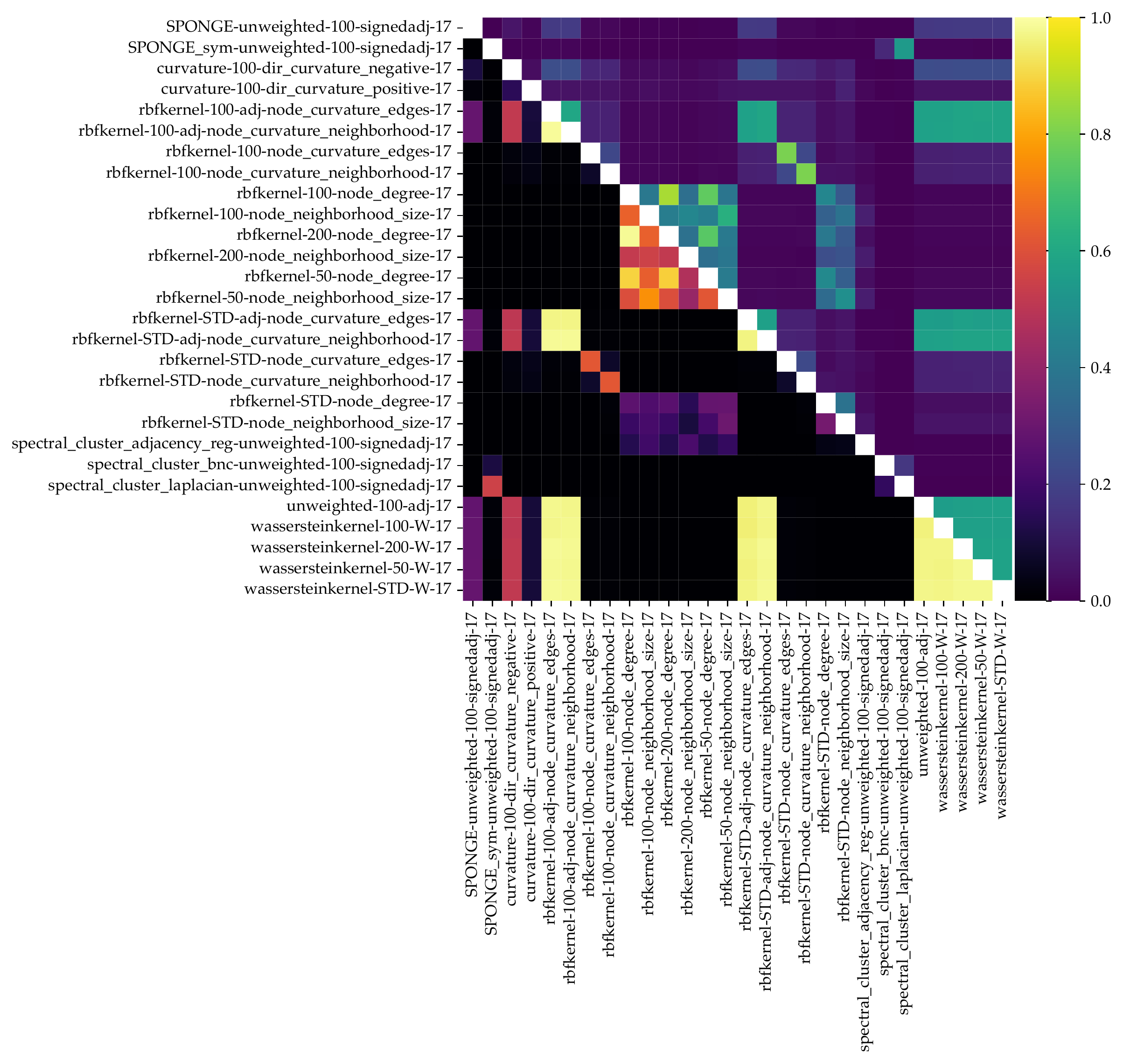}
	\caption{Node clusterings based on curvature features differ radically from clusterings based on other local features. 
		We show the normalized mutual information (upper triangle) and the adjusted rand score (lower triangle) of node clusterings based on different method/feature combinations, 
		computed on the citation hypergraph of PRB from the \apsvcout collection, 
		with curvatures computed using $\smoothing = 0.1$, $\mu^\werw$, and $\aggregation_\mean$.
	}\label{fig:aps-cout-nodes}
\end{figure}

\begin{table}[t]
	\centering
	\tiny
	\caption{\ourmethod features lead to node clusterings that are semantically more coherent than node clusterings derived from other local features.
	For two clusterings of the PRE citation hypergraph from the \apsvcout collection---one a spectral clustering using the sign of directional curvatures as a feature (\cref{tab:aps-tfidf-orchid}), 
	the other a clustering using an RBF kernel with node neighborhood size as a feature (\cref{tab:aps-tfidf-nonorchid})---we show the top terms, i.e., the terms associated with each cluster that have a TF-IDF score of at least 0.1, 
	along with their TF-IDF scores and their occurrence frequency across all clusters, in tuples of shape (term, TF-IDF score, global occurrence frequency).
}\label{tab:aps-tfidf}
\subcaption{Feature: sign of directional \ourmethod curvatures}
\label{tab:aps-tfidf-orchid}
\begin{tabular}{p{\linewidth}}
		\toprule
		(smectic, 0.51, 1), (liquid, 0.39, 4), (crystals, 0.22, 4), (antiferroelectric, 0.21, 1), (crystal, 0.19, 2), (phase, 0.17, 4), (chiral, 0.17, 1), (c$\alpha$, 0.15, 1), (paper, 0.15, 1), (rock, 0.15, 1), (scissors, 0.15, 1), (electric, 0.14, 1), (phases, 0.14, 2), (ray, 0.13, 1), (cyclic, 0.13, 1), (species, 0.12, 1), (field, 0.11, 3), (games, 0.1, 2)\\\midrule
		(resetting, 0.76, 1), (stochastic, 0.32, 1), (random, 0.24, 2), (walks, 0.18, 1), (diffusion, 0.17, 2), (brownian, 0.15, 1), (processes, 0.11, 1)\\\midrule
		(nematic, 0.66, 2), (liquid, 0.41, 4), (crystal, 0.3, 2), (colloidal, 0.26, 1), (colloids, 0.18, 1), (crystals, 0.16, 4), (particles, 0.15, 1), (interaction, 0.14, 1)\\\midrule
		(boltzmann, 0.75, 1), (lattice, 0.51, 1), (method, 0.2, 1), (flows, 0.15, 1), (model, 0.11, 5)\\\midrule
		(quantum, 0.58, 3), (heat, 0.38, 1), (engine, 0.34, 1), (engines, 0.27, 1), (efficiency, 0.24, 1), (performance, 0.21, 1), (power, 0.17, 1), (maximum, 0.17, 1), (otto, 0.12, 1), (carnot, 0.12, 1), (refrigerators, 0.1, 1)\\\midrule
		(granular, 0.85, 2), (gas, 0.17, 1), (gases, 0.16, 1), (inelastic, 0.13, 1), (driven, 0.13, 1)\\\midrule
		(chimera, 0.7, 1), (states, 0.35, 1), (oscillators, 0.33, 1), (coupled, 0.31, 2), (networks, 0.2, 3), (nonlocally, 0.13, 1), (chimeras, 0.12, 1), (coupling, 0.1, 1)]\\\midrule
		(dynamics, 0.19, 1), (model, 0.18, 5), (networks, 0.17, 3), (liquid, 0.16, 4), (diffusion, 0.13, 2), (phase, 0.13, 4), (quantum, 0.13, 3), (dimensional, 0.12, 1), (random, 0.12, 2), (flow, 0.11, 2), (systems, 0.11, 1), (plasma, 0.11, 1), (coupled, 0.1, 2), (time, 0.1, 1)\\\midrule
		(dynamic, 0.41, 1), (ising, 0.35, 2), (phase, 0.34, 4), (oscillating, 0.34, 1), (field, 0.32, 3), (transition, 0.24, 1), (kinetic, 0.2, 1), (model, 0.2, 5), (magnetic, 0.15, 1), (nonequilibrium, 0.13, 1), (blume, 0.12, 1), (capel, 0.12, 1), (transitions, 0.11, 1)\\\midrule
		(biaxial, 0.53, 1), (nematic, 0.5, 2), (liquid, 0.29, 4), (crystals, 0.19, 4), (phases, 0.19, 2), (bent, 0.17, 1), (phase, 0.16, 4), (molecules, 0.15, 1), (core, 0.14, 1), (simulation, 0.12, 1), (molecular, 0.1, 1), (antinematic, 0.1, 1), (mesogenic, 0.1, 1)
		\\\midrule
		(passive, 0.47, 1), (scalar, 0.41, 1), (anomalous, 0.39, 1), (scaling, 0.29, 1), (advected, 0.24, 1), (turbulence, 0.22, 1), (turbulent, 0.18, 1), (advection, 0.15, 1), (loop, 0.12, 1), (anisotropy, 0.11, 1), (anisotropic, 0.11, 1), (renormalization, 0.11, 1), (vector, 0.11, 1), (field, 0.1, 3)\\\midrule
		(quantum, 0.51, 3), (decay, 0.45, 1), (loschmidt, 0.33, 1), (echo, 0.33, 1), (fidelity, 0.25, 1), (chaotic, 0.23, 1), (semiclassical, 0.18, 1), (lyapunov, 0.13, 1), (perturbations, 0.11, 1)\\\midrule
		(casimir, 0.69, 1), (critical, 0.37, 1), (forces, 0.27, 1), (films, 0.13, 1), (size, 0.13, 1), (force, 0.13, 1), (finite, 0.12, 1), (free, 0.11, 1), (ising, 0.11, 2), (thermodynamic, 0.1, 1), (model, 0.1, 5)\\\midrule
		(traffic, 0.88, 1), (flow, 0.3, 2), (model, 0.13, 5), (car, 0.13, 1), (following, 0.11, 1)\\\midrule
		(rogue, 0.62, 1), (schr\"odinger, 0.34, 1), (waves, 0.31, 2), (wave, 0.29, 2), (equation, 0.25, 1), (nonlinear, 0.21, 2), (solutions, 0.17, 1), (soliton, 0.12, 1), (solitons, 0.11, 1)\\\midrule
		(cooperation, 0.6, 1), (dilemma, 0.38, 1), (prisoner, 0.34, 1), (game, 0.25, 1), (games, 0.24, 2), (evolutionary, 0.19, 1), (networks, 0.18, 3), (spatial, 0.17, 1), (social, 0.14, 1), (public, 0.12, 1), (goods, 0.1, 1)\\\midrule
		(granular, 0.59, 2), (chains, 0.36, 1), (chain, 0.32, 1), (propagation, 0.22, 1), (waves, 0.21, 2), (nonlinear, 0.2, 2), (solitary, 0.2, 1), (wave, 0.17, 2), (pulse, 0.15, 1), (crystals, 0.14, 4), (strongly, 0.12, 1)\\\bottomrule
	\end{tabular}
\vspace*{10cm}
\end{table}
\begin{table}[t]
	\centering
	\tiny
\ContinuedFloat
\subcaption{Feature: node neighborhood size}
\label{tab:aps-tfidf-nonorchid}
\vspace*{6pt}
\begin{tabular}{p{\linewidth}}
	\toprule
	(relation, 0.37, 1), (entropy, 0.34, 1), (differences, 0.34, 2), (production, 0.34, 1), (theorem, 0.34, 1), (work, 0.31, 2), (fluctuation, 0.29, 1), (nonequilibrium, 0.27, 2), (free, 0.27, 5), (energy, 0.27, 3)
	\\\midrule
	(model, 0.25, 5), (phase, 0.23, 5), (dimensional, 0.2, 3), (dynamics, 0.18, 6), (time, 0.17, 4), (networks, 0.16, 10), (lattice, 0.15, 7), (systems, 0.15, 8), (granular, 0.13, 6), (stochastic, 0.13, 2), (random, 0.12, 6), (noise, 0.12, 1), (liquid, 0.12, 4), (nonlinear, 0.12, 2), (field, 0.12, 2), (diffusion, 0.11, 4), (quantum, 0.11, 4), (coupled, 0.11, 2), (transition, 0.11, 5), (boltzmann, 0.1, 6)
	\\\midrule
	(model, 0.28, 5), (networks, 0.25, 10), (lattice, 0.22, 7), (boltzmann, 0.21, 6), (equations, 0.19, 2), (stochastic, 0.15, 2), (dynamics, 0.14, 6), (transition, 0.13, 5), (synchronization, 0.13, 3), (granular, 0.13, 6), (time, 0.13, 4), (scale, 0.13, 3), (glass, 0.13, 2), (systems, 0.12, 8), (random, 0.12, 6), (dimensional, 0.12, 3), (phase, 0.12, 5), (diffusion, 0.11, 4), (complex, 0.1, 3), (reaction, 0.1, 1), (free, 0.1, 5)
	\\\midrule
	(model, 0.32, 5), (microstates, 0.27, 1), (auxiliary, 0.27, 1), (violating, 0.27, 1), (connections, 0.24, 1), (generate, 0.24, 1), (steady, 0.22, 1), (collisions, 0.22, 1), (ising, 0.2, 1), (approach, 0.2, 2), (distribution, 0.2, 1), (second, 0.2, 1), (law, 0.2, 1), (generalized, 0.2, 1), (arbitrary, 0.2, 2), (equilibrium, 0.18, 2), (synchronization, 0.18, 3), (chaos, 0.17, 1), (gases, 0.17, 2), (states, 0.17, 2), (granular, 0.15, 6), (networks, 0.13, 10)
	\\\midrule
	(equation, 0.22, 4), (fokker, 0.19, 1), (planck, 0.19, 1), (hard, 0.19, 2), (fractional, 0.17, 1), (dynamics, 0.16, 6), (observable, 0.13, 1), (evolution, 0.12, 1), (characteristics, 0.12, 1), (quasistatic, 0.12, 1), (correction, 0.12, 1), (cohesion, 0.12, 1), (pair, 0.12, 1), (nearly, 0.12, 1), (ordered, 0.12, 1), (characterization, 0.12, 1), (preasymptotic, 0.12, 1), (formulas, 0.12, 1), (thermalization, 0.12, 1), (depinning, 0.12, 1), (theorems, 0.11, 1), (low, 0.11, 1), (amorphous, 0.11, 2), (intermittency, 0.11, 1), (hydrodynamics, 0.11, 1), (avalanche, 0.11, 1), (athermal, 0.11, 1), (correlation, 0.11, 2), (transport, 0.11, 1), (solution, 0.11, 1), (jammed, 0.11, 2), (propelled, 0.11, 1), (collective, 0.11, 1), (interacting, 0.11, 1), (asymptotic, 0.11, 1), (heterogeneity, 0.11, 1), (singularities, 0.11, 1), (dense, 0.1, 2), (highly, 0.1, 1), (near, 0.1, 1), (inelastic, 0.1, 1), (quantum, 0.1, 4), (fluid, 0.1, 2), (self, 0.1, 2), (shear, 0.1, 2), (rheology, 0.1, 2), (flow, 0.1, 3), (work, 0.1, 2), (liquids, 0.1, 1), (growth, 0.1, 1), (laws, 0.1, 1), (application, 0.1, 1), (disordered, 0.1, 1), (walks, 0.1, 1)
	\\\midrule
	(networks, 0.38, 10), (scientific, 0.22, 1), (collaboration, 0.19, 1), (path, 0.19, 1), (ii, 0.16, 1), (density, 0.14, 2), (diffusion, 0.14, 4), (herds, 0.12, 1), (theory, 0.12, 4), (systems, 0.12, 8), (granular, 0.12, 6), (random, 0.12, 6), (schools, 0.11, 1)
	\\\midrule
	(lattice, 0.27, 7), (networks, 0.26, 10), (boltzmann, 0.21, 6), (phase, 0.18, 5), (models, 0.16, 1), (structure, 0.16, 2), (method, 0.16, 4), (interactions, 0.14, 1), (self, 0.14, 2), (network, 0.13, 2), (community, 0.13, 2), (social, 0.12, 1), (free, 0.12, 5), (dimensions, 0.12, 1), (scale, 0.12, 3), (granular, 0.11, 6), (random, 0.11, 6), (systems, 0.11, 8), (motion, 0.1, 1), (graphs, 0.1, 2), (transition, 0.1, 5), (emulsions, 0.1, 1)
	\\\midrule
	(networks, 0.24, 10), (glass, 0.17, 2), (transition, 0.17, 5), (lattice, 0.15, 7), (solutions, 0.14, 1), (systems, 0.14, 8), (equations, 0.13, 2), (lévy, 0.13, 1), (large, 0.13, 1), (external, 0.13, 1), (jammed, 0.13, 2), (flights, 0.13, 1), (correlated, 0.13, 1), (quantum, 0.12, 4), (coupled, 0.12, 2), (analysis, 0.12, 1), (force, 0.12, 1), (colloidal, 0.12, 1), (order, 0.12, 1), (packings, 0.11, 1), (synchronization, 0.11, 3), (hard, 0.11, 2), (time, 0.11, 4), (langevin, 0.11, 1), (fluctuations, 0.11, 1), (density, 0.1, 2)
	\\\midrule
	(model, 0.24, 5), (dynamics, 0.22, 6), (dimensional, 0.2, 3), (phase, 0.19, 5), (liquid, 0.15, 4), (nonlinear, 0.15, 2), (systems, 0.14, 8), (field, 0.13, 2), (time, 0.12, 4), (quantum, 0.11, 4), (transition, 0.11, 5), (diffusion, 0.11, 4), (flow, 0.1, 3), (induced, 0.1, 1)
	\\\midrule
	(networks, 0.41, 10), (evaluating, 0.36, 1), (uncorrelated, 0.36, 1), (generation, 0.36, 1), (finding, 0.3, 1), (structure, 0.28, 2), (community, 0.28, 2), (free, 0.26, 5), (scale, 0.26, 3), (random, 0.23, 6)
	\\\midrule
	(zero, 0.55, 1), (epitome, 0.3, 1), (applied, 0.27, 1), (applications, 0.27, 1), (distributions, 0.25, 1), (stress, 0.25, 1), (arbitrary, 0.23, 2), (disorder, 0.23, 1), (degree, 0.23, 1), (temperature, 0.23, 1), (jamming, 0.23, 1), (graphs, 0.21, 2), (random, 0.17, 6)
	\\\midrule
	(measurements, 0.35, 1), (hierarchical, 0.32, 1), (differences, 0.32, 2), (approach, 0.3, 2), (organization, 0.3, 2), (master, 0.3, 1), (equilibrium, 0.27, 2), (nonequilibrium, 0.25, 2), (free, 0.25, 5), (energy, 0.25, 3), (complex, 0.24, 3), (equation, 0.21, 4), (networks, 0.2, 10)
	\\\midrule
	(nucleotides, 0.29, 1), (hove, 0.29, 1), (mixing, 0.26, 1), (van, 0.26, 1), (mosaic, 0.26, 1), (correlation, 0.24, 2), (dna, 0.24, 1), (patterns, 0.24, 1), (organization, 0.22, 2), (testing, 0.22, 1), (mixture, 0.22, 1), (lennard, 0.2, 1), (jones, 0.2, 1), (supercooled, 0.2, 1), (function, 0.19, 1), (coupling, 0.19, 1), (mode, 0.19, 1), (binary, 0.19, 2), (theory, 0.17, 4), (networks, 0.15, 10)
	\\\midrule
	(lattice, 0.42, 7), (percolation, 0.32, 1), (boltzmann, 0.26, 6), (monte, 0.21, 1), (term, 0.21, 1), (carlo, 0.21, 1), (site, 0.19, 1), (forcing, 0.19, 1), (fast, 0.17, 1), (bond, 0.17, 1), (transitions, 0.17, 1), (algorithm, 0.17, 1), (effects, 0.16, 1), (liquid, 0.16, 4), (phase, 0.16, 5), (discrete, 0.16, 2), (network, 0.16, 2), (gas, 0.16, 2), (simulation, 0.16, 2), (nonideal, 0.16, 1), (small, 0.15, 1), (world, 0.15, 1), (scaling, 0.15, 2), (gases, 0.15, 2), (model, 0.14, 5), (method, 0.14, 4), (equation, 0.12, 4)
	\\\midrule
	(viscoplastic, 0.33, 1), (dissipation, 0.29, 1), (deformation, 0.29, 1), (isotropy, 0.29, 1), (dispersion, 0.29, 1), (amorphous, 0.27, 2), (stability, 0.27, 1), (solids, 0.27, 1), (galilean, 0.27, 1), (invariance, 0.27, 1), (dynamics, 0.2, 6), (method, 0.2, 4), (lattice, 0.2, 7), (boltzmann, 0.19, 6), (theory, 0.19, 4)
	\\\midrule
	(boltzmann, 0.63, 6), (lattice, 0.5, 7), (equation, 0.29, 4), (simulations, 0.21, 1), (fluid, 0.21, 2), (liquid, 0.19, 4), (gas, 0.19, 2), (binary, 0.18, 2), (method, 0.17, 4), (theory, 0.16, 4), (systems, 0.16, 8)
	\\\midrule
	(plane, 0.38, 1), (flow, 0.32, 3), (dynamics, 0.26, 6), (granular, 0.24, 6), (rheophysics, 0.24, 1), (endemic, 0.21, 1), (temperatures, 0.19, 1), (equilibration, 0.19, 1), (bagnold, 0.19, 1), (inclined, 0.19, 1), (partial, 0.17, 1), (effective, 0.17, 1), (flows, 0.16, 1), (dense, 0.16, 2), (epidemic, 0.16, 1), (shear, 0.16, 2), (slow, 0.16, 1), (rheology, 0.16, 2), (discrete, 0.15, 2), (simulation, 0.15, 2), (states, 0.14, 2), (scaling, 0.14, 2), (materials, 0.14, 1), (energy, 0.14, 3), (complex, 0.13, 3), (systems, 0.12, 8), (networks, 0.11, 10)
	\\\bottomrule
\end{tabular}
\vspace*{10cm}
\end{table}

\clearpage
  
\end{document}